\documentclass[letterpaper]{article} % DO NOT CHANGE THIS
\usepackage{aaai25}  % DO NOT CHANGE THIS
\usepackage{times}  % DO NOT CHANGE THIS
\usepackage{helvet}  % DO NOT CHANGE THIS
\usepackage{courier}  % DO NOT CHANGE THIS
\usepackage[hyphens]{url}  % DO NOT CHANGE THIS
\usepackage{graphicx} % DO NOT CHANGE THIS
\usepackage{natbib}  % DO NOT CHANGE THIS AND DO NOT ADD ANY OPTIONS TO IT
\usepackage{caption} % DO NOT CHANGE THIS AND DO NOT ADD ANY OPTIONS TO IT
\usepackage{algorithm}
\usepackage{algorithmic}

\usepackage[utf8]{inputenc} % allow utf-8 input
\usepackage[T1]{fontenc}    % use 8-bit T1 fonts
\usepackage{url}            % simple URL typesetting
\usepackage{booktabs}       % professional-quality tables
\usepackage{amsfonts}       % blackboard math symbols
\usepackage{nicefrac}       % compact symbols for 1/2, etc.
\usepackage{microtype}      % microtypography
\usepackage{amsmath}
\usepackage{amssymb}
\usepackage{mathtools}
\usepackage{amsthm}

\newcommand{\Ab}{\mathbf{A}}
\newcommand{\EE}{\mathbb{E}}
\newcommand{\RR}{\mathbb{R}}
\newcommand{\norm}[1]{\|#1\|}
\newcommand{\bbb}{\mathbf{b}}

\usepackage[capitalize,noabbrev]{cleveref}

\usepackage[table]{xcolor}

\theoremstyle{plain}
\newtheorem{theorem}{Theorem}[section]
\newtheorem{proposition}[theorem]{Proposition}
\newtheorem{lemma}{Lemma}
\newtheorem{sublemma}{lemma}[lemma]

\theoremstyle{definition}

\newtheorem{assumption}[theorem]{Assumption}
\theoremstyle{remark}
\newtheorem{remark}[theorem]{Remark}

\usepackage{newfloat}
\usepackage{listings}
\DeclareCaptionStyle{ruled}{labelfont=normalfont,labelsep=colon,strut=off} % DO NOT CHANGE THIS
\floatstyle{ruled}
\newfloat{listing}{tb}{lst}{}
\floatname{listing}{Listing}
\title{Two-Timescale Critic-Actor for Average Reward MDPs with Function Approximation}
\author{Prashansa Panda and Shalabh Bhatnagar
}
\affiliations{
    Indian Institute of Science, Bengaluru, India.\\
    E-mail: prashansap@iisc.ac.in; shalabh@iisc.ac.in
}

\usepackage{bibentry}
\begin{document}

\maketitle

\begin{abstract}
Several recent works have focused on carrying out non-asymptotic convergence analyses for AC algorithms. Recently, a two-timescale critic-actor algorithm has been presented for the discounted cost setting in the look-up table case where the timescales of the actor and the critic are reversed and only asymptotic convergence shown. 
In our work, we present the first two-timescale critic-actor algorithm with function approximation in the long-run average reward setting and present the first finite-time non-asymptotic as well as asymptotic convergence analysis for such a scheme. 
We obtain optimal learning rates and prove that our algorithm achieves a sample complexity of {$\mathcal{\tilde{O}}(\epsilon^{-(2+\delta)})$ with $\delta >0$  arbitrarily close to zero,} for the mean squared error of the critic to be upper bounded by $\epsilon$ which is better than the one obtained for two-timescale AC in a similar setting. A notable feature of our analysis is that we present the asymptotic convergence analysis of our scheme in addition to the finite-time bounds that we obtain and show the almost sure asymptotic convergence of the (slower) critic recursion to the attractor of an associated differential inclusion with actor parameters corresponding to local maxima of a perturbed average reward objective. We also show the results of numerical experiments on three benchmark settings and observe that our critic-actor algorithm performs the best amongst all algorithms.
\end{abstract}

% Uncomment the following to link to your code, datasets, an extended version or similar.
%
% \begin{links}
%     \link{Code}{https://aaai.org/example/code}
%     \link{Datasets}{https://aaai.org/example/datasets}
%     \link{Extended version}{https://aaai.org/example/extended-version}
% \end{links}

\section{Introduction}

Actor-Critic (AC) methods have proved to be efficient in solving many reinforcement learning (RL) tasks. Actor-only methods such as REINFORCE suffer from high variance during the estimation of the policy gradient whereas critic-only methods like Q-learning are efficient in the tabular setting but can diverge when function approximation is used. AC methods try to circumvent these problems by combining both policy- and value-based methods to solve RL problems. In these approaches, the goal of the actor is to learn the optimal policy using the value updates provided by the critic, while the goal of the critic is to learn the value function for a policy prescribed by the actor. One obtains stable behavior of such algorithms through a difference in timescales that we explain in more detail below.

%Thus, the critic solves the problem of prediction whereas actor solves the problem of control. 

The AC framework is designed to mimic the policy iteration (PI) procedure for Markov decision processes \citep{puterman}. 
% PI proceeds via two nested loops where the outer loop updates the policy while the inner loop updates the value function corresponding to a given policy update. The outer loop procedure thus needs to wait until completion of the inner loop update.  
The AC algorithms incorporate two-timescale coupled stochastic recursions with the learning rate of the actor typically converging to zero at a rate faster than that of the critic. 
The timescale separation in two-timescale stochastic approximation algorithms such as AC is critical in ensuring stability of the recursions and their almost sure convergence. This is because
from the viewpoint of the faster timescale, the slower recursion appears to be quasi-static while from the viewpoint of the slower timescale, the faster recursion appears to have converged. This helps the AC scheme to emulate policy iteration and thereby converge to the optimal policy. Asymptotic convergence analyses of two-timescale AC schemes are largely available via the ordinary differential equation (ODE) based approach.
%uses the concept of two - timescale ordinary differential equation (ODE) where the slower parameter is considered fixed when viewed from the timescale of the faster parameter and the faster parameter converges to a point which is a function of the slower component.Hence in two-timescale actor critic algorithm , the convergence point of the critic is a function of the actor which is assumed to be constant(when viewed from the timescale of the critic).
In \citep{bhatnagar2023actorcritic}, the critic-actor (CA) algorithm was proposed, in the lookup table setting for the infinite-horizon discounted cost criterion, where the roles of the actor and the critic were reversed by swapping their timescales. %In other words, contrary to popular wisdom, the critic performs updates in this case on the slower timescale while the actor performs updates on the faster one. 
The resulting procedure is seen to track value iteration instead of policy iteration.  

In this paper, we carry this idea forward and present, for the first time, a critic-actor algorithm with function approximation and for the long-run average (and not discounted) reward setting. We then carry out detailed asymptotic and non-asymptotic convergence analyses of the same. Our algorithm runs temporal difference learning on the slower timescale to estimate the critic updates and stochastic policy gradient on the faster timescale for the actor. 
We prove that this algorithm emulates an approximate value iteration scheme. 
% \textcolor{blue}{While AC algorithm has been studied and analyzed significantly, such studies have so far been lacking for the CA algorithm. 
Our paper plugs in an important gap that existed previously by studying a new class of algorithms obtained by merely reversing the timescales of the actor and the critic. 
Our finite-time analysis shows that the two-timescale CA algorithm has a better sample complexity when compared with the two-timescale AC algorithm.

In our algorithm, even though there are three recursions, the average reward and actor recursions together proceed on the same timescale that is faster than the timescale of the critic update. 
Notice the difference of our scheme with AC algorithms for average-reward MDPs such as those in \citep{wu2022finite, BHATNAGAR20092471}, where the average-reward recursion proceeds on the same (faster) timescale as the critic while the actor recursion proceeds slower.
We use linear function approximation for the critic recursion and a policy gradient actor. %and we consider a completely model-free setting where we do not know the system model but have access to data-tuples of state, action, reward and next state either from a real data source or from a simulation device.
We perform the non-asymptotic analysis of this algorithm and obtain its sample complexity. In addition, we prove that the scheme remains asymptotically stable and is almost surely convergent to the attractors of an underlying differential inclusion. %Such an analysis has not been previously carried out and is a generalization of the ODE based analysis in the presence of multiple attractors \citep{aubin, benaim-inclusions}. 
Our analysis helps us in getting the  optimised learning rates for the actor and the critic recursions. %For example, in Theorem \ref{critic_convergence}, after getting the expression for the critic error, we find the optimal values of $\nu , \beta$ and $\sigma$ that provide the tightest bound.

Finally, we show numerical performance comparisons of our algorithm with the AC and a few other algorithms over three different OpenAI Gym environments and observe that the CA algorithm shows the best performance amongst all algorithms considered, though by small margins. In terms of the training time performance, CA is better than all algorithms except DQN on all three environments, and in fact, it takes about half the training time on two of the environments. 

\noindent \textit{\textbf{Main Contributions:}}\\
%We summarise our main contributions below.\\
 (\textit{a}) We present the first critic-actor (CA) algorithm with linear function approximation for the long-run average-reward criterion where the critic runs on a slower timescale than the actor.\\
(\textit{b}) We carry out the first finite-time analysis of the two-timescale CA algorithm wherein we present finite-time bounds for the critic error, actor error and the average reward estimation error, respectively. In particular, we obtain a sample complexity of {$\mathcal{\tilde{O}}(\epsilon^{-(2+\delta)})$ with $\delta >0$  arbitrarily close to zero}, for the mean squared error of the critic to be upper bounded by $\epsilon$. This is better than  
the sample complexity of $\tilde{\mathcal{O}}(\epsilon^{-2.5})$ obtained by the two-timescale AC algorithm of \citep{wu2022finite} and can be brought as close as possible to the sample complexity of the recently studied single-timescale AC schemes \citep{olshevsky, chen2023finitetime} where the same is $\tilde{\mathcal{O}}(\epsilon^{-2})$. Note that for the latter schemes, there are no formal proofs available for the asymptotic stability and almost sure convergence (see Section ~\ref{ft} and Appendix for details). %The actor parameter of \cite{BHATNAGAR20092471} operates on a slower timescale indicating that the control policy parameter is treated as a `higher-level parameter' than the value update as it responds to the value update. In our algorithm, the critic parameter is on a slower timescale and thus is the `higher-level' parameter responding to the policy update. Note again that our algorithm mimics value iteration unlike \cite{BHATNAGAR20092471} and other papers where the two-timescale AC algorithm is devised to mimic policy iteration.
\\
%We obtain optimal learning rates for the actor updates, critic updates as well as the average reward estimates.\\
(\textit{c}) We perform a novel asymptotic analysis of convergence of this scheme by showing that the slower timescale critic recursion remains stable and tracks a limiting differential inclusion that depends on the set of local maxima of the actor recursion corresponding to any critic update. Such an analysis under Markov noise has not been previously carried out in the context of any AC algorithm and is a generalization of the ODE based analysis of such algorithms in the presence of multiple attractors of the actor \citep{aubin, benaim-inclusions}. We mention here that unlike us, most papers on finite-time analysis of RL algorithms do not prove stability and almost sure convergence of such algorithms, see Table~\ref{sample-table}. As a result, we provide stronger guarantees than such algorithms. See
Appendix~\ref{asympanalysis} for details of this analysis. \\
%This is a generalization of the standard ODE based analysis that is necessitated by the presence of multiple attractors. \\
(\textit{d}) We show the results of experiments comparing our CA algorithm with some other well-studied algorithms, on three different OpenAI Gym environments and observe that CA performs better than the other algorithms in average reward performance. In terms of training time, the CA algorithm performs uniformly better than AC requiring half the training time on two of the environments (see Section~\ref{exp_results} and Appendix~\ref{simdetails}). 

\vspace*{10pt}
\noindent\textbf{Notation: }
For two sequences $\{c_n\}$ and $\{d_n\}$, we write $c_n = \mathcal{O}(d_n)$ if there exists a constant $P>0$ such that ${\displaystyle \frac{|c_n|}{|d_n|} \le P}$. To further hide logarithmic factors, we use the notation $\tilde{\mathcal{O}}(\cdot)$. Without any other specification, $\|\cdot\|$ denotes the $\ell_2$-norm of Euclidean vectors.
$d_{TV}(M,N)$ is the total variation norm distance between two probability measures $M$ and $N$, and is defined as $d_{TV}(M,N) = \frac{1}{2} \int_{\mathcal{X}}|M(dx)-N(dx)|$.

\section{Related Work}
\label{related_work}

\begin{table*}[t]
\caption{Comparison with related works: \citep{olshevsky} uses Discounted Reward Setting while Others are for Average Reward.
}
\label{sample-table}
\vskip 0.15in
\begin{center}
\begin{small}
%\begin{sc}
\begin{tabular}{|p{2 cm}|c|c|c|p{2 cm}|p{1.6 cm}|}
\toprule
Reference & Algorithm & Sampling & Asymptotic Analysis & Sample Complexity & Critic\\     
\midrule
\citep{wu2022finite}    & \multicolumn{1}{c|}{Two-timescale AC} & \multicolumn{1}{c|}{Markovian} &\multicolumn{1}{c|}{Shown in \citep{BHATNAGAR20092471}} & $\tilde{\mathcal{O}}(\epsilon^{-2.5})$ & TD(0) \\ \hline
\citep{olshevsky} &  \multicolumn{1}{c|}{Single-timescale AC} 
& \multicolumn{1}{c|}{i.i.d}  &\multicolumn{1}{c|}{Not shown} & $\tilde{\mathcal{O}}(\epsilon^{-2})$ & TD(0)\\ \hline
\citep{chen2023finitetime} & \multicolumn{1}{c|}{Single-timescale AC}  & \multicolumn{1}{c|}{Markovian}  &\multicolumn{1}{c|}{Not shown} & $\tilde{\mathcal{O}}(\epsilon^{-2})$ & TD(0)\\ \hline
\citep{suttle2023exponentially} &    \multicolumn{1}{c|}{Two-timescale MLAC}  & \multicolumn{1}{c|}{Markovian} &  \multicolumn{1}{c|}{Not Shown} & $\widetilde{\mathcal{O}}(\tau^{2}_{mix}\epsilon^{-2})$ & MLMC\\ \hline
\rowcolor{blue!10} Our work    & \multicolumn{1}{c|}{Two-timescale CA}  & \multicolumn{1}{c|}{Markovian} & Shown & $\tilde{\mathcal{O}}(\epsilon^{-(2+\delta)})$ & TD(0) \\
\bottomrule
\end{tabular}
%\end{sc}
\end{small}
\end{center}
\vskip -0.1in
\end{table*}

We briefly review here some of the related work.
%\subsection*{Actor Critic Algorithms}
  In \citep{konda_actor_critic_type}, AC algorithms were presented for the look-up table representations and the first asymptotic analysis of these algorithms was carried out. Subsequently, \citep{konda_onactorcritic} presented AC algorithms with function approximation using the Q-value function and an asymptotic analysis of convergence was presented. In \citep{kakade_2001}, a natural gradient based algorithm was presented. Subsequently, works such as  \citep{dicastro2009convergent} and \citep{ zhang2020provably} have also carried out the  asymptotic analysis of AC algorithms. In \citep{BHATNAGAR20092471}, natural AC algorithms were presented that perform bootstrapping in both the actor and the critic recursions, and an asymptotic analysis of convergence including stability was provided. A new method for solving two-timescale optimization that achieves faster convergence was recently proposed in \citep{pmlr-v247-zeng24a}. %We also present below some of the recent works related to non-asymptotic analysis of AC  algorithms.  

%\subsection*{The Critic Actor Algorithm}
The CA algorithm was introduced in 
\citep{bhatnagar2023actorcritic} for the look-up table case. In this, the actor recursion is on the faster timescale compared to critic and the infinite horizon discounted cost criterion is considered. Asymptotic stability and almost sure convergence of the algorithm is shown. %For experimental comparison, they have shown results comparing the performance of Critic-Actor with an AC Algorithm in \citep{konda_actor_critic_type} for the discounted cost setting. 
In our work, we present the first CA algorithm for the case of (a) function approximation and (b) the long-run average reward setting. Further, we present both -- asymptotic as well as non-asymptotic convergence analyses of the proposed scheme where we observe that our algorithm gives a better upper bound on the sample complexity as opposed to AC. We also observe that our algorithm performs on par and is in fact slightly better than the two-timescale AC algorithm.

%%%%%
%\subsection*{Finite-Time Analysis}
During the past few years there has been significant research activity on finite-time analysis of various algorithms in RL. A finite-time analysis of a two-timescale AC algorithm under Markovian sampling has been conducted in \citep{wu2022finite} and a sample complexity of $\tilde{\mathcal{O}}(\epsilon^{-2.5})$ for convergence to an  $\epsilon$-approximate stationary point of the performance function has been obtained.

Finite-time analyses of a single-timescale AC algorithm have been presented in \citep{olshevsky, chen2023finitetime}.
In these algorithms, the actor and the critic recursions proceed on the same timescale but there are no proofs of  stability and almost sure convergence of the recursions. A prime reason here is that AC algorithms are based on the policy iteration procedure whereby one ideally requires convergence of the critic in between two updates of the actor. Such guarantees can usually be obtained when there is a timescale difference between the two updates.
%Non-asymptotic (finite time) analysis of AC is shown in the single-timescale case 
A sample complexity of $\tilde{\mathcal{O}}(\epsilon^{-2})$ is obtained in \citep{olshevsky, chen2023finitetime} for single-timescale AC. While \citep{olshevsky} considers i.i.d sampling from the stationary distribution of the Markov chain in a discounted reward setting, \citep{chen2023finitetime} makes use of Markovian sampling and works with the average reward formulation. On the other hand, we obtain a sample complexity of ${\cal O}(\epsilon^{-(2+\delta)})$ with Markovian sampling and in the average reward setting, where $\delta>0$ can be made arbitrarily small. In the limit when $\delta=0$, one obtains a single-timescale AC algorithm for which asymptotic guarantees are not available. Thus, a major contribution of our work is to provide a sample complexity of our two-timescale CA scheme that is arbitrarily close to that of single-timescale AC but while providing theoretical assurances of asymptotic stability and almost sure convergence that single-timescale AC does not provide.   

Non-asymptotic convergence properties of two-timescale natural AC algorithm have been studied in \citep{9827586} in the look-up table case where a sample complexity of $\tilde{\mathcal{O}}(\epsilon^{-6})$ has been obtained. Finite-time analysis is helpful in finding out the optimal learning rates for different updates used in the various algorithms. Amongst other recent works, {\cite{han2024finitetimedecoupledconvergencenonlinear}, \cite{NEURIPS2022_6f6dd92b} and \cite{NEURIPS2021_096ffc29} have also provided finite-time analyses}.

Table \ref{sample-table} shows the comparison of our work with  some of these related works. In \citep{suttle2023exponentially}, a multi-level Monte-Carlo AC algorithm is analyzed with sample complexity of $\widetilde{\mathcal{O}}(\tau^{2}_{mix}\epsilon^{-2})$. However, unlike us, asymptotic stability and almost sure convergence is not shown.
% (for definition of $\tau_{mix}$ please refer \citep{suttle2023exponentially}).   \citep{suttle2023exponentially} mention that they obtain a sample complexity of $\tilde{\mathcal{O}}(\epsilon^{-2})$
% despite using a two-timescale AC algorithm as they 
% incorporate Monte-Carlo updates and makes use of a modified Adagrad step-size, unlike say \citep{wu2022finite} that does not use such step-sizes and gets a sample complexity of
% $\tilde{\mathcal{O}}(\epsilon^{-2.5})$. 
%Both \citep{wu2022finite} and \citep{chen2023finitetime} considered long-run average reward setting with linear function approximation for the critic. %\citep{wu2022finite} and \citep{chen2023finitetime} found the sample complexity to reach $\epsilon - $ approximate stationary point of performance function whereas in this paper we have found out the sample complexity for the mean squared error of the critic to be upper bounded by $\epsilon$. 
%To the best of our knowledge, both non-asymptotic and asymptotic convergence analyses of the CA algorithm in the long-run average reward setting with function approximation have not been carried out in the past. 
%Our analysis shows that we obtain a sample complexity of \textcolor{blue}{$\tilde{\mathcal{O}}(\epsilon^{-2-\delta})$} for the two-timescale CA algorithm for which we also provide asymptotic stability and almost sure convergence guarantees. The sample complexity we obtain is better than the sample complexity of $\tilde{\mathcal{O}}(\epsilon^{-2.5})$ obtained by the two-timescale AC algorithm \citep{wu2022finite}. 
It is also important to note that unlike many other variants (including the single-timescale AC algorithms), the two-timescale AC algorithm, as with our two-timescale CA algorithm, possesses asymptotic stability and almost sure convergence guarantees. 
For our algorithm, the latter properties are shown using a differential inclusions based analysis, see  Appendix~\ref{asympanalysis} for details. 

\section{The Framework and Algorithm}
\label{framework_and_algo}

In this section, we first discuss the Markov decision process (MDP) framework. We then present our two-timescale CA algorithm where we use linear function approximation for the value function estimates.

\subsection{Markov Decision Process}

We consider an MDP with finite state and action spaces that is characterised by the tuple $(S,A,P,r)$, where 
$S$ denotes the state space, 
$A$ is the action space, $P(s^{'} \vert s,a)$ is the probability of transition from state $s$ to $s^{'}$ under action $a$. Further,
$r$ denotes the single-stage reward that depends on the state $s$ and action $a$ at a given instant. Moreover, we let $\vert r(s,a) \vert \leq U_r$, $\forall s \in S,\forall a \in A$ where $U_r > 0$ is a constant.
We consider stationary randomized policies $\pi_\theta(a|s)$, $a\in A, s\in S$ parameterised by $\theta$. Our aim is to maximise the long-run average reward (with $\mu_\theta$ being the stationary distribution):
\begin{align*}
    L(\theta) :&= \lim\limits_{T \rightarrow \infty}\frac{1}{T}\sum\limits_{t=0}^{T-1}r(s_t,a_t) = E_{s \sim \mu_{\theta},a \sim \pi_{\theta}}[r(s,a)].
\end{align*}
%where $\mu_{\theta}$ is the stationary state distribution induced by $\pi_{\theta}$.
The differential value function $V^{\theta}(s), s\in S$ is defined as (with $s_0$ being the starting state, $a_t \sim \pi_{\theta}(\cdot|s_t)$ and $s_{t+1} \sim P(\cdot| s_t,a_t)$): ${\displaystyle 
    V^{\theta}(s) = E\bigg[\sum\limits_{t=0}^{\infty}(r(s_t,a_t) - L(\theta)) |s_0 = s \bigg]}$.
The differential action-value (Q-value) function %to evaluate the overall rewards starting from $s$, taking action $a$, and following policy $\pi_\theta$ thereafter 
is defined as
% \begin{equation}
%     \begin{aligned}%\label{eq2.1.2}
%     Q_{\theta}(s,a)=&\ \mathbb{E}_{\theta}[\sum\limits_{t=0}^{\infty}(r(s_t,a_t)-L(\theta))|s_0=s,a_0=a]\\
%     \overset{\text{(i)}}{=}&\ r(s,a)-L(\theta)+\mathbb{E}[V^{\theta}(s')], \nonumber
% \end{aligned}
% \end{equation}
\begin{align*}
Q_{\theta}(s,a)&= \mathbb{E}_{\theta}[\sum\limits_{t=0}^{\infty}(r(s_t,a_t)-L(\theta))|s_0=s,a_0=a]\\
    &\overset{\text{(i)}}
    {=}r(s,a)-L(\theta)+\mathbb{E}[V^{\theta}(s')],
\end{align*}
where the expectation in (i) is taken over $s'\sim P(\cdot|s,a)$.

% Next we obtain the gradient of the performance function with respect to parameter $\theta$ as follows:
% \begin{align*}
%     \nabla L(\theta)=\mathbb{E}_{s\sim\mu_{\theta},a\sim\pi_{\theta}}[(Q_{\theta}(s,a)-V^{\theta}(s))\nabla_{\theta}\log\pi_{\theta}(s|a)].
% \end{align*}
The policy gradient theorem \citep{sutton1999, suttonbarto} gives the following expression for $\nabla_\theta L(\theta)$:
\begin{align*}
\nabla_{\theta}L(\theta)=\mathbb{E}_{s\sim\mu_{\theta},a\sim\pi_{\theta}}[A_{\theta}(s,a)\nabla_{\theta}\log\pi_{\theta}(a|s)],
\end{align*}
where $A_{\theta}(s,a) =Q_\theta(s,a)-V^{\theta}(s)$ denotes the advantage function.

\subsection{Function Approximation}

In order to save on the computational effort needed to find exact solutions, one often uses value function approximation techniques based on linear or nonlinear function approximation architectures. We use linear function approximators here for our theoretical results. Such approximators have been found to be viable for asymptotic analyses. For instance, see \citep{tsitsiklisroy2} for an asymptotic analysis of temporal difference learning algorithms and \citep{BHATNAGAR20092471} for an analysis of AC algorithms when linear function approximators are used in the average cost setting.
We approximate the state-value function here using a linear approximation architecture as
${\displaystyle
\widehat{V}^{\theta}(s;v)=\phi(s)^\top v}$, 
where $\phi: \mathcal{S}\rightarrow \mathbb{R}^{d_1}$ is a known feature mapping and $\theta$ is the policy parameter for the considered policy. %For our experiments, we incorporate neural network based function approximators and observe that our algorithm works well experimentally with nonlinear approximators.

\subsection{Two-Timescale Critic-Actor Algorithm}

\begin{algorithm}[tb]
   \caption{Two Timescale Critic-Actor Algorithm}
   \label{algo}
\begin{algorithmic}
   \STATE {\bfseries Input:} initial average reward parameter $L_0$, initial actor parameter $\theta_0$, initial critic parameter $v_0$, step-size $\alpha_{t}$ for actor, $\beta_{t}$ for critic and $\gamma_{t}$ for the average reward estimator.
   \STATE Draw $s_{0}$ from some initial distribution.
   \FOR {$t=0,1,2,\dots$}
    \STATE Take the action $a_t \sim \pi_{\theta_t}(\cdot|s_t)$
    \STATE Observe next state $s_{t+1} \sim P(\cdot|s_t,a_t)$ and the reward $r_t = r(s_t,a_t)$
    \STATE $L_{t+1} = L_t + \gamma_t(r_t - L_t)$
    \STATE $\delta_t = r_t - L_t + \phi(s_{t+1})^{\top} v_{t} - \phi(s_t)^{\top} v_{t}$
    \STATE $v_{t+1} = \Gamma(v_{t} + \beta_{t} \delta_t \phi(s_t))$\label{algline:critic_update}
    \STATE $\theta_{t+1} = \theta_{t} + \alpha_t \delta_t \nabla_{\theta} \log \pi_{\theta_{t}}(a_t|s_t)$\label{algline:actor_update}
\ENDFOR 
\end{algorithmic}
\end{algorithm}

Algorithm \ref{algo} presents the two-timescale CA algorithm involving linear function approximation for the critic recursion. All step-sizes satisfy the standard Robbins-Monro conditions. In addition, $\beta_t = o(\alpha_t)$ for $t \geq 0$  and $\gamma_t = K\alpha_t$ for some $K > 0$, $t \geq 0$. As a result of this, the average reward and actor updates are performed on the faster timescale compared to the critic updates. The projection operator $\Gamma(\cdot)$ has been used for the estimates of the critic. Here, for any $x\in \mathbb{R}^{d_1}$, $\Gamma(x)$ denotes the projection of $x$ to a compact and convex set $C\subset \mathbb{R}^{d_1}$. For any vector $y \in C$, we have $\Vert y \Vert \leq U_{v}$,  where $U_{v} > 0$ is a constant. As mentioned earlier, the single-stage reward is a function of the current state and action taken.

\section{Finite-Time Analysis}\label{ft}

We provide, in this section, the assumptions required and the main theoretical results for carrying out a non-asymptotic convergence analysis. We also state below the main results providing the optimal learning rate and sample complexity for the two-timescale CA algorithm. The detailed proofs are given in the appendix. Specifically, the details of the non-asymptotic analysis are provided in Appendix~\ref{finitetimeanalysis}. We further show the asymptotic convergence analysis in Appendix~\ref{asympanalysis}. 
% The same could not be accommodated here for lack of space. Note that asymptotic convergence guarantees for many algorithms in the literature such as the single-timescale AC algorithms of \citep{olshevsky, chen2023finitetime} are not available as it is not easy to provide such guarantees.

%\subsection{Assumptions}\label{assumptions}

\begin{assumption} \label{assum:bounded_feature_norm}
    The norm of each state feature is bounded by 1,  i.e., $\Vert \phi(i) \Vert \le 1$.
\end{assumption}
The above is not a restrictive assumption since the number of states $|S|$ is finite. Thus, the requirement on features can be accomplished by replacing any features $\phi(i)\in \mathbb{R}^{d_1},i\in S$ by ${\displaystyle \frac{\phi(i)}{\max_{j\in S} \phi(j)}}$. %This will ensure that Assumption~\ref{assum:bounded_feature_norm}
 %holds. 
 This assumption is helpful in carrying out the finite time analysis of the actor and critic recursions as it helps provide suitable upper bounds for some of the terms.

\begin{assumption}
\label{assum:negative-definite}
    For all potential policy parameters $\theta$, the matrix $\Ab$ defined as under is negative definite:  % and has the maximum eigenvalue $- \lambda <0$.
${\displaystyle
    \Ab := \EE_{s,a,s^{'}} \big[ \phi(s) \big( \phi(s^{'}) - \phi(s)\big)^{\top} \big],
}$
where $s \sim \mu_{\theta}(\cdot)$ (the stationary distribution under policy parameter $\theta$) and $a \sim \pi_{\theta}(\cdot | s), s^{'} \sim P(\cdot | s, a)$. %(the probability of transition to a next state following state $s$, when action $a$ is chosen).
Further, let $\lambda_\theta$ denote the largest eigenvalue of $\Ab$. Then $-\lambda \stackrel{\triangle}{=} \sup_\theta \lambda_\theta <0$. 
\end{assumption}

Under a given policy $\pi$, Assumption \ref{assum:negative-definite} has been shown to hold in \citep{tsitsiklisroy2} in the setting of temporal difference learning under the requirements that (a) the feature vectors are linearly independent and (b) $\Phi r \not= e$, where $e$ is the vector of all $1$'s. %This has also been derived in the literature on AC algorithms, see for instance, \citep{BHATNAGAR20092471}. 
This assumption   
helps give the existence and uniqueness of $v^{*}(\theta)$ because the following equations hold: For $s \sim \mu_{\theta}(\cdot), a \sim \pi_{\theta}(\cdot | s)$,
\begin{align}
    &\Ab v^{*}(\theta) + \bbb = 0,\label{critic_conv_point}\\
    &\bbb:= \EE_{s,a,s^{'}} [(r(s,a)- L(\theta))\phi(s) ]\notag.
\end{align}
Assumption~\ref{assum:negative-definite} helps in carrying out a finite time analysis of the critic error.

\begin{assumption}[Uniform ergodicity] \label{assum:ergodicity}
     %For a fixed $\theta$, denote $\mu_{\theta}(\cdot)$ as the stationary distribution induced by the policy $\pi_{\theta}(\cdot|s)$ and the transition probability measure $P(\cdot|s,a)$. 
     Consider a Markov chain generated as per the following: $a_t \sim \pi_{\theta}(\cdot | s_t), s_{t+1} \sim P(\cdot | s_t, a_t)$. Then there exist $b > 0$ and $k \in (0,1)$ such that:
    \begin{align*}
        d_{TV}\big(P(s_{\tau} \in \cdot | s_0 = s), \mu_{\theta}(\cdot)\big) \le b k^{\tau}, \forall \tau \ge 0, \forall s \in S.
    \end{align*}
\end{assumption}

Assumption \ref{assum:ergodicity} states that the $\tau$-step state distribution of the Markov chain under policy $\pi_\theta$ converges at a geometric rate to the stationary distribution $\mu_\theta$. 

\begin{assumption} \label{assum:policy-lipschitz-bounded}
%Let $\pi_{\theta}(a|s)$ be a policy parameterized by $\theta$. 
There exist $L,B, K>0$ such that for all $s,s^{'}\in S$ and $a,a^{'}\in A$,
\begin{enumerate}
\item[(a)] $\big\|\nabla \log \pi_{\theta}(a|s) \big\| \le B$, $\forall \theta \in \RR^d$,
\item[(b)] $\big\|\nabla \log \pi_{\theta_1}(a|s) - \nabla \log \pi_{\theta_2}(a^{'}|s^{'}) \big\| \le  K \norm{\theta_1 - \theta_2}$, $\forall \theta_1,\theta_2 \in \RR^d$,
\item[(c)] $\big|\pi_{\theta_1}(a|s) - \pi_{\theta_2}(a|s) \big| \le L \norm{\theta_1 - \theta_2}$, $\forall \theta_1,\theta_2 \in \RR^d$.
\end{enumerate}
\end{assumption}

Assumptions \ref{assum:policy-lipschitz-bounded}(a) and (c) are standard in the literature on policy gradient methods, see \citep{wu2022finite}. Assumption \ref{assum:policy-lipschitz-bounded}(b) implies that the randomized policy is also $K$-smooth in the parameter $\theta$, in addition to being Lipschitz continuous (see Assumption \ref{assum:policy-lipschitz-bounded} (c)).

\begin{assumption}\label{smoothness_mu}
$\forall \theta_{1},\theta_{2} \in \RR^{d}$, $\forall s\in S$, $\exists L_{\mu}>0$ such that \\ $\Vert \nabla \mu_{\theta_1}(s) - \nabla \mu_{\theta_2}(s) \Vert \leq L_{\mu}\Vert \theta_1 - \theta_2 \Vert$.
\end{assumption}

Assumption \ref{smoothness_mu} implies that the stationary distribution $\mu_\theta$ is $L_\mu$-smooth as a function of $\theta$.
%We show in Theorem~\ref{thm-stationary} that $\nabla\mu_\theta$ exists and is continuous.
%Assumption~\ref{assum:policy-lipschitz-bounded} and Theorem 2 of \citep{schweitzer} that Assumption \ref{smoothness_mu} holds. 
This assumption 
is required for proving smoothness of $v^{*}(\theta)$ and has been adopted in \citep{chen2023finitetime}. We provide sufficient conditions in Theorem~\ref{thm-stationary} for the verification of %Asthat under the requirement that the Markov chain is ergodic, the gradient of $\mu_\theta$ exists and is continuous. Further, if $\nabla^2\mu_\theta(s)$ exists and
%is uniformly bounded, then $\mu_\theta$ is also $L_\mu$-smooth, implying 
Assumption~\ref{smoothness_mu}. %Thus, we provide sufficient conditions in Theorem~\ref{thm-stationary} for Assumption~\ref{smoothness_mu}.

\begin{assumption}\label{V_lipschitz}
$\exists L_{v} > 0$ such that for any $s \in S$,
    \begin{align*}
        \Vert V^{\theta_1}(s) - V^{\theta_2}(s) \Vert \leq L_{v}\Vert \theta_1 - \theta_2 \Vert , \forall \theta_1,\theta_2 \in \RR^{d}.
    \end{align*}
\end{assumption}
Assumption \ref{V_lipschitz} is needed for deriving finite time bounds while proving convergence of actor.

Let $\tau_{t}$ denote the mixing time of an ergodic Markov chain. So, 
\begin{align}\label{eq:def_mixing_time}
    \tau_t & := 
    \min 
    \big \{
   m \ge 0 \mid 
    bk^{m-1} \le
    \min \{ \alpha_t, \beta_t,\gamma_t \}
    \big \},
\end{align}
where $b,k$ are defined as in Assumption \ref{assum:ergodicity}.

\subsection{Sample Complexity Results}

We provide here the sample complexity bounds that we obtain. The proofs of these results require several detailed steps that cannot be accommodated in the limited space. Hence,  we provide the complete detailed analysis in Appendix~\ref{finitetimeanalysis} while brief proof sketches of the main results are given here. 

We consider the following step-sizes: $\alpha_t = c_\alpha/(1+t)^\nu$, $\beta_t = c_\beta/(1+t)^\sigma , \gamma_t = c_\gamma/(1+t)^{\nu}$ with $0 < \nu < \sigma < 1$, $2\sigma < 3\nu$, $2\sigma - \nu < 1$ and $c_\alpha,c_\beta,c_\gamma > 0$. Thus, the actor and the average reward recursions proceed here on the same timescale but which is faster than the critic recursion. Let
%We choose $c_{\alpha} > 0$ and $c_{\gamma} > 0$ such  that,
\[
    \frac{c_\alpha}{c_\gamma} <  \frac{1}{2B(G + U_w) + U_{w}B},
\]
where, $G = 2(U_r + U_v)B$, $U_{w} = 2B(U_{v} +  \bar{U}_{v})$ and $\vert V^{\theta}(s) \vert \leq \bar{U}_{v} ,\forall \theta \in \RR^{d},\forall s \in S$, respectively.

\begin{theorem}[Convergence of Average reward estimate]
\label{average_reward_convergencee}
Under assumptions \ref{assum:bounded_feature_norm}, \ref{assum:ergodicity}, \ref{assum:policy-lipschitz-bounded}, \ref{V_lipschitz},
\begin{align*}
    &\sum\limits_{k=\tau_t}^{t} \mathbb{E}[(L_k - L(\theta_k))^2]
   \leq \mathcal{O}(\log^2 t \cdot t^{1-\nu}) + \mathcal{O}(t^\nu)\\
    & + 2\frac{(G + U_w)^2}{(1 - \frac{c_\alpha}{c_\gamma}U_w B)^2}\frac{c_\alpha^2}{c_\gamma^2}\sum\limits_{k=\tau_t}^{t}\mathbb{E}\Vert M(\theta_k,v_k)\Vert^2,  
\end{align*}
where, ${\displaystyle L(\theta_k) = \EE_{s \sim \mu_{\theta_k},a \sim \pi_{\theta_k},s^{'} \sim P(.|s,a) }[r(s,a)]}$ and \\
$ M(\theta_t,v_t) = E_{s_t \sim \mu_{\theta_t},a_t \sim \pi_{\theta_t},s_{t+1} \sim P(.|s_t,a_t)}[( r(s_t,a_t)- L(\theta_t)
    + \phi(s_{t+1})^{\top} v_{t} - \phi(s_t)^{\top} v_{t})\nabla \log\pi_{\theta_t}(a_t|s_t)]$.\\
% \begin{align*}
%     \\
%     M(\theta_t,v_t) & = E_{s_t \sim \mu_{\theta_t},a_t \sim \pi_{\theta_t},s_{t+1} \sim p}[( r(s_t,a_t)- L(\theta_t)\\
%     &\quad + \phi(s_{t+1})^{\top} v_{t} - \phi(s_t)^{\top} v_{t})\nabla \log\pi_{\theta_t}(a_t|s_t)].
% \end{align*}
\end{theorem}
\begin{proof}
See Appendix~\ref{appendix}.
\end{proof}
% \textit{Proof sketch.}

% We denote by $ y_t := L_t - L(\theta_t)$ and expand $y_{t+1}^2$ to find an upper bound on it as follows:
% {\small
% \begin{align*}
%     y_{t+1}^2 &\leq (1-2\gamma_t)y_t^2+2\gamma_ty_t(r_t-L(\theta_t)) 
%     +2y_t(L(\theta_t)-L(\theta_{t+1}))\\
%     &\qquad +2(L(\theta_t)-L(\theta_{t+1}))^2
%      +2\gamma_t^2(r_t-L_t)^2.
% \end{align*}}
% After taking expectation, rearranging and summing from $\tau_t$ to $t$, we obtain 
% \begin{align*}
%     \sum\limits_{k=\tau_t}^{t} \mathbb{E}[y_k^2]
%     \leq &\  \underbrace{\sum\limits_{t=\tau_t}^t\frac{1}{2\gamma_k}\mathbb{E}(y_k^2-y^2_{k+1})}_{I_1}+\underbrace{\sum\limits_{k=\tau_t}^{t}\mathbb{E}[y_k(r_k-L(\theta_k))]}_{I_2}\\
%  &+\underbrace{\sum\limits_{k=\tau_t}^{t}\frac{1}{\gamma_k}\mathbb{E}[y_k(L(\theta_k)-L(\theta_{k+1})]}_{I_3}\\
%       & +\underbrace{\sum\limits_{k=\tau_t}^{t}\frac{1}{\gamma_k}\mathbb{E}[(L(\theta_k)-L(\theta_{k+1}))^2]}_{I_4}\\
%     &+\underbrace{\sum\limits_{k=\tau_t}^{t}\gamma_k\mathbb{E}[(r_k-L_k)^2]}_{I_5}.
% \end{align*}
% After analysing terms $I_1, \ldots, I_5$, we get the desired result. Please refer section \ref{average_reward_convergence} in the Appendix for the detailed proof.

\begin{theorem}[Convergence of actor]\label{actor_convergence}Under assumptions \ref{assum:bounded_feature_norm}, \ref{assum:ergodicity}, \ref{assum:policy-lipschitz-bounded}, \ref{V_lipschitz},
{\small
    \begin{align*}
     \frac{1}{(1 + t - \tau_t)}\sum\limits_{k=\tau_t}^{t}E\Vert  M(\theta_k,v_k)\Vert^2
    = \mathcal{O}(t^{\nu - 1}) + \mathcal{O}(\log^2 t \cdot t^{-\nu}).
 \end{align*}
}
\end{theorem}
\begin{proof}
See Appendix~\ref{appendix}.
\end{proof}
{From Lemma 4 of \cite{BHATNAGAR20092471}, $M(\theta_k,v_k)$ equals the sum of the gradient of average reward and an error term that depends on the function approximator of the critic. From Theorem \ref{actor_convergence}, convergence is to the stationary points of a function whose gradient is this sum.} %above.}

% \textit{Proof sketch.}

% By applying Lemma \ref{L_smoothness} to the update rule of the actor, we obtain
% \begin{align*}
%     L(\theta_{t+1}) &\geq L(\theta_t) + \alpha_t \langle \nabla L(\theta_t) ,\delta_{t}\nabla \log\pi_{\theta_t}(a_t|s_t) \rangle \\
% &- M_{L}\alpha_t^2\Vert\delta_{t}\nabla \log\pi_{\theta_t}(a_t|s_t)\Vert^2.
% \end{align*}

% Next, the term $\langle \nabla L(\theta_t) ,\delta_{t}\nabla \log\pi_{\theta_t}(a_t|s_t) \rangle $ is split as 
% \begin{align*}
%     &\langle \nabla L(\theta_t) ,\delta_{t}\nabla \log\pi_{\theta_t}(a_t|s_t) \rangle\\
%     &= I(O_t,\theta_t,L_t,v_t) + \langle \nabla L(\theta_t) , M(\theta_t,v_t) \rangle\\
%  &\qquad + \langle \nabla L(\theta_t) , E_{\theta_t}[(L(\theta_t) - L_t)\nabla \log\pi_{\theta_t}(a_t|s_t)] \rangle,
% \end{align*}
% where $I(O_t,\theta_t,L_t,v_t)$ has been defined in Section \ref{actor_convergence_proof} of Appendix and $E_{\theta_t}[\cdot]$ is expectation w.r.t $s_t \sim \mu_{\theta_t} , a_t \sim \pi_{\theta_t}$. Finally, after taking expectation, rearranging and summing from $\tau_t$ to $t$ we get an upper bound for $\sum\limits_{k=\tau_t}^{t}E\Vert  M(\theta_k,v_k)\Vert^2$. We end up with the desired result after analysing the bound. Please refer Section \ref{actor_convergence_proof} of the Appendix for a detailed analysis.

\begin{figure}
   \vskip 0.2in
\begin{center}
\centerline{\includegraphics[width=0.5\columnwidth]{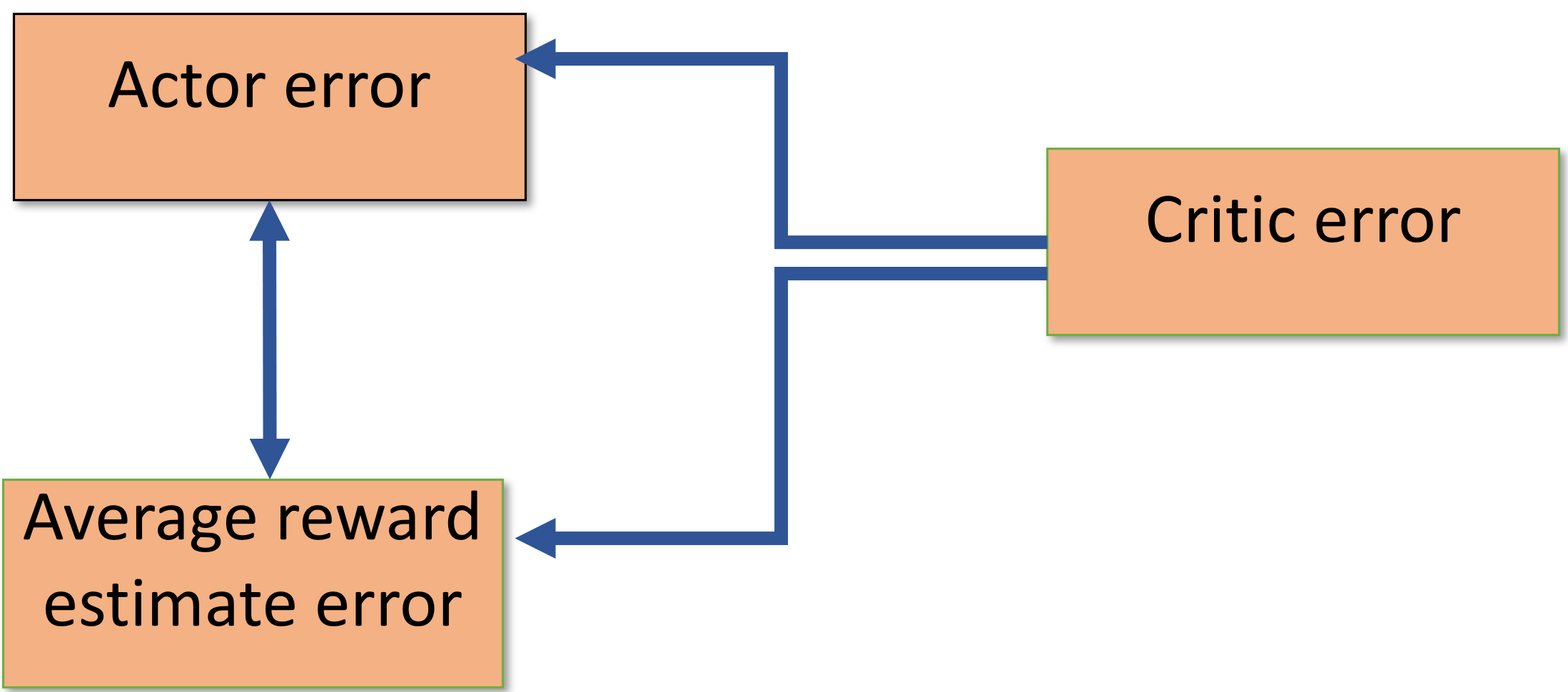}}
    \caption{Dependency of errors among the actor, critic and average reward estimate.}
    \label{fig:enter-label}
\end{center}
\vskip -0.2in
\end{figure}

\begin{theorem}[Convergence of critic] \label{critic_convergence}Under assumptions \ref{assum:bounded_feature_norm}, \ref{assum:negative-definite}, \ref{assum:ergodicity}, \ref{assum:policy-lipschitz-bounded}, \ref{smoothness_mu}, \ref{V_lipschitz},
    \begin{align*}
    &\frac{1}{1+t-\tau_t}\sum_{k=\tau_t}^{t}E\Vert v_k - v^{*}(\theta_k) \Vert^2\\
    &= \mathcal{O}(\log^2 t \cdot t^{ \sigma - 2\nu})
    + \mathcal{O}(t^{2\sigma - \nu - 1}) +\mathcal{O}(\log^2 t \cdot t^{-3\nu + 2\sigma}),
\end{align*}
where $v^{*}(\theta_k)$ is as defined in \cref{critic_conv_point}.
\end{theorem}

\textit{Proof sketch.}

We denote $z_t := v_t - v^{*}(\theta_t)$. After expanding $\Vert z_t \Vert^2$ and using Assumption \ref{assum:negative-definite}, we get an upper bound for $\Vert z_t \Vert^2$ as:
\begin{align*}
    &\Vert z_{t+1} \Vert^2 \\
    &\leq \Vert z_t \Vert^2 + 2\beta_t\langle z_t ,  \delta_t\phi(s_t) - E_{\theta_t}[\delta_t\phi(s_t)]\rangle
     -
    2\beta_t\lambda\Vert z_t\Vert^2\\
    &\qquad +
    2\langle z_t , v^{*}(\theta_t) - v^{*}(\theta_{t+1}) \rangle\\
    &\qquad+ 2\beta_t^2\delta_t^2\Vert \phi(s_t)\Vert^2 + 2\Vert v^{*}(\theta_t) - v^{*}(\theta_{t+1})  \Vert^2.
\end{align*}
We then rearrange the terms and take expectation of the summation from $\tau_t$ to $t$, to get 
\begin{align*}
    & \lambda\sum_{k=\tau_t}^{t}E\Vert z_k \Vert^2 \leq \underbrace{\sum_{k=\tau_t}^{t}\frac{1}{2\beta_k}E[\Vert z_k \Vert^2 - \Vert z_{k+1} \Vert^2] }_{I_1}\\
     & + \underbrace{\sum_{k=\tau_t}^{t}E[\langle z_k ,  \delta_t\phi(s_k) - E_{\theta_k}[\delta_k\phi(s_k)]\rangle]}_{I_2}\\
     &+ \underbrace{\sum_{k=\tau_t}^{t}\frac{1}{\beta_k}E\langle z_k , v^{*}(\theta_k) - v^{*}(\theta_{k+1}) + (\nabla v_k^*)^T(\theta_{k+1} - \theta_k)\rangle}_{I_3} \\
     &+\underbrace{\sum_{k=\tau_t}^{t} \frac{1}{\beta_k}E\langle z_k , (\nabla v_k^*)^T(\theta_{k} - \theta_{k+1}) \rangle}_{I_4}\\ &+\underbrace{\sum_{k=\tau_t}^{t}\beta_kE[\delta_k^2\Vert \phi(s_k)\Vert^2]}_{I_5}+ \underbrace{\sum_{k=\tau_t}^{t}\frac{1}{\beta_k}E\Vert v^{*}(\theta_k) - v^{*}(\theta_{k+1})  \Vert^2}_{I_6},
\end{align*}
where $-\lambda = \sup_\theta \lambda_{\theta}$, see \cref{assum:negative-definite}.
%$\lambda_{\theta}$ is largest eigenvalue of $\Ab$ defined in \ref{critic_conv_point}.
After analysing the terms $I_1, \ldots, I_6$, we get the desired result. Please refer to  \cref{critic_convergence_proof} for the details.

From Theorems \ref{average_reward_convergencee}, \ref{actor_convergence} and  \ref{critic_convergence}, it is clear that (as also shown in Figure 1), the critic error depends on actor error and the average reward estimate error. Moreover, actor error and average reward estimate error are dependent. Hence, Theorem \ref{critic_convergence} relies on the results of Theorems \ref{actor_convergence} and \ref{average_reward_convergencee}.
From Theorem \ref{actor_convergence}, we can observe that $E\Vert M(\theta_k,v_k)\Vert^2 \rightarrow 0$ as $k \rightarrow \infty$.
Now as actor recursions proceed  on the faster timescale as compared to the critic, the latter appears to be quasi-static to the actor, cf.~Chapter 6 of \citep{borkar-SA}. Hence, we can say that from the timescale of the actor recursion, $v_t = v , \forall t \geq 0$.%Therefore,  we can say that the actor ($\theta_t$) converges to $\theta(v)$ such that:
Therefore the point of convergence of the actor $\theta_t$ will be $\theta(v)$ such that
\begin{align*}
     &E_{\theta(v)}[( r(s,a)- L(\theta(v))  + \phi(s^{'})^{\top} v\\
     &\qquad - \phi(s)^{\top} v)\nabla \log\pi_{\theta(v)}(a|s)]
     = 0.
 \end{align*}
Now since the critic is on the slower timescale compared to the actor, $\theta_t$ tracks $\theta(v_t)$ at time instant $t$ when viewed from the timescale of the critic. Moreover from Theorem \ref{critic_convergence}, we have $\|v_k-v^{*}(\theta_k)\| \rightarrow 0$ as $k \rightarrow \infty$. Hence, we can conclude that $v_k$ converges to a point $\omega$ such that $\omega - v^{*}(\theta(\omega)) = 0$.

Optimizing over the values of $\nu$ and $\sigma$ in theorem \ref{critic_convergence}, we have $\nu = 0.5$ and $\sigma = 0.5 + \beta$, where $\beta >0$ can be made arbitrarily close to zero. Hence we have  the following: 
\begin{align*}
    \frac{1}{1+t-\tau_t}\sum_{k=\tau_t}^{t}E\Vert z_k \Vert^2 &= \mathcal{O}(\log^2 t \cdot t^{(2\beta - 0.5)})
\end{align*}
Therefore in order for the mean squared error of the critic to be upper bounded by $\epsilon$, namely,

\begin{align*}
     \frac{1}{1+t-\tau_t}\sum_{k=\tau_t}^{t}E\Vert z_k \Vert^2  =  \mathcal{O}(\log^2 T \cdot T^{(2\beta - 0.5)}) \leq \epsilon,
\end{align*}
we need to set { $T = \tilde{\mathcal{O}}(\epsilon^{-(2 +\delta)})$ where $\delta >0$ can be made arbitrarily close to zero}. For instance, $\nu=0.5$ and $\sigma=0.51$ gives a sample complexity $\tilde{\mathcal{O}}(\epsilon^{-2.08})$. We refer the reader to \cref{critic_convergence_proof} for the detailed analysis.

\begin{remark}
    The analysis of single time-scale AC cannot be easily applied here. While carrying out the finite time analysis, we are focusing on finding the upper bound on various terms that should approach zero as time tends to infinity. If we apply the analysis of single time-scale AC  to that of CA, some terms will not have such upper bounds due to the difference in timescales.
\end{remark}

\begin{figure*}
   \vskip 0.2in
\begin{center}
\centerline{\includegraphics[scale = 0.2]{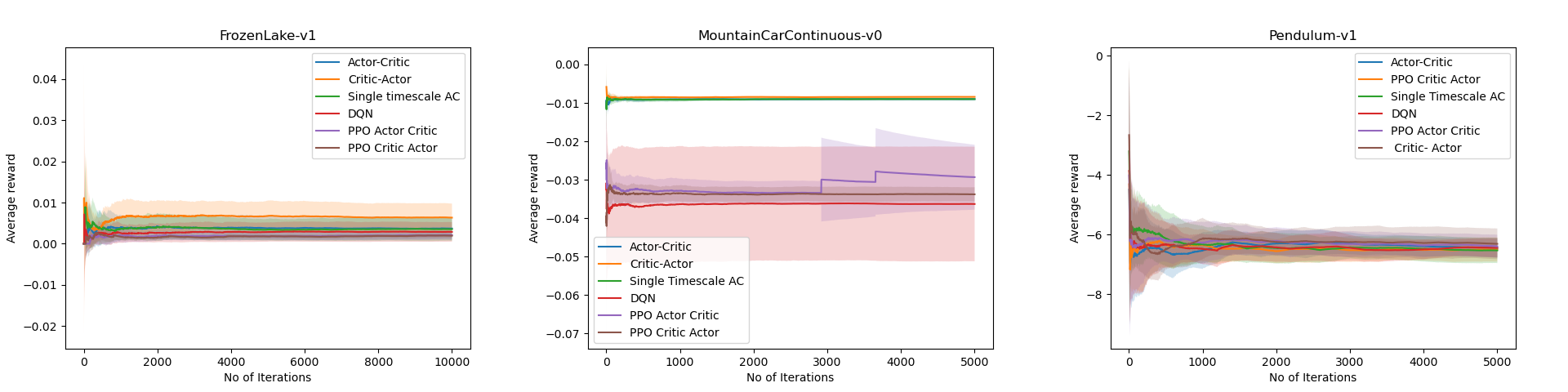}}
    \caption{Comparison of Critic-Actor with few other algorithms }
    \label{fig:critic_actor_comparision}
\end{center}
\vskip -0.2in
\end{figure*}

%\section{Experimental Results}\label{exp_results}
{\small 
\begin{table*}
  \centering
  \caption{Comparison of Critic-Actor with different algorithms in terms of average reward}
  \begin{tabular}{|p{2 cm}|p{2 cm}|p{2 cm}|p{2 cm}|p{2 cm}|p{2 cm}| p{2 cm}|}
    \hline
    \textbf{Environment} & \textbf{Critic-Actor} & \textbf{Actor-Critic} & \textbf{DQN} & \textbf{PPO AC} & \textbf{PPO CA} & \textbf{Single Timescale AC}\\
    \hline
    Frozen Lake & $0.00633 \pm 0.0034$  & $0.0036 \pm 0.0027$ & $0.0028 \pm 0.0023$  & $0.00207 \pm 0.0009$ & $0.00194 \pm 0.0005$ & $0.0035 \pm 0.0028$\\ \hline
    Pendulum & $-6.30 \pm 0.41$ & $-6.45 \pm 0.324$ & $-6.45 \pm 0.316$ & $-6.39 \pm 0.38$ & $-6.44 \pm 0.51$ & $-6.53 \pm 0.42$\\ \hline
    Mountain Car Continuous & $-0.0084 \pm 0.0001$  & $-0.009 \pm 0.0002$ & $-0.036 \pm 0.014$ & $-0.029 \pm 0.0084$ & $-0.0337 \pm 0.0018$ & $ -0.009 \pm 0.0002$\\
    \hline
  \end{tabular}
  \label{tab:experiment}
\end{table*}}

\section{Asymptotic Convergence Analysis}

We first note that some assumptions needed for the finite-time analysis are not required for the asymptotic convergence analysis.
In particular, we do not need Assumption~\ref{assum:ergodicity} on the exponential mixing of Markov noise. Our differential inclusions based asymptotic analysis that we present, unlike many other references that assume i.i.d sampling from the stationary distribution, is powerful enough to carry through under Assumptions~\ref{assum:negative-definite}, \ref{assum:policy-lipschitz-bounded} as well as Assumptions \ref{assum-ergodic-chain} and \ref{assum:ss} below. Let $\theta$ take values in a compact set $D\subset\mathbb{R}^{d_2}$.
%We make here the following assumptions:
\begin{assumption}
\label{assum-ergodic-chain}
The Markov chain $\{s_t\}$ under any policy $\pi^\theta$ is ergodic for any fixed $\theta\in D$.
\end{assumption}

This assumption is routinely made for analysis of RL algorithms with Markov noise \citep{tsitsiklisroy2, konda_onactorcritic, BHATNAGAR20092471} and guarantees existence of a unique stationary distribution $\mu_\theta$ for any fixed $\theta \in D$. We shall replace here the stronger requirement in Assumption~\ref{assum:ergodicity} by Assumption~\ref{assum-ergodic-chain}.

\begin{assumption}
\label{assum:ss}
The step-size sequences $\{\alpha_t\}$, $\{\beta_t\}$ and $\{\gamma_t\}$ satisfy the following conditions:
\begin{itemize}
\item[(i)] $\alpha_t,\beta_t,\gamma_t >0$ for all $t$ with $\gamma_t = K\alpha_t$ for some $K>0$.
\item[(ii)] ${\displaystyle \sum_t \alpha_t = \sum_t \beta_t =\infty}$.
\item[(iii)] ${\displaystyle \sum_t (\alpha_t^2+\beta_t^2) <\infty}$.
\end{itemize}
\end{assumption}

For asymptotic convergence, we first analyse in the appendix, CA for the average reward objective, with function approximation, by incorporating a projection on the actor in addition to the critic update. Subsequently, we remove the projection on the critic update and prove the asymptotic stability and convergence of the algorithm. This algorithm is then similar to the standard AC algorithms that have been well-studied in the literature, where also one projects the actor but not the critic, cf.~\cite{BHATNAGAR20092471}, except that now the time scales of the two recursions are reversed. We refer the reader to Appendix~\ref{appendix} for the detailed analysis.
% We present the key results here while the detailed analysis including the other results is presented in Appendix~\ref{appendix}. %and analyses where projection is not taken on the critic recursion but on the actor update.

We prove the stability and convergence of our two-timescale CA algorithm by proving that the critic recursion asymptotically tracks a compact connected internally chain transitive invariant set of an associated differential inclusion (DI) \citep{aubin, benaim-inclusions}. A DI-based analysis is a generalization of the ODE approach to stochastic approximation and is necessitated because we allow for multiple local maxima for the actor-recursion for any given critic update. In the context of AC or CA algorithms, ours is the first analysis that incorporates this level of sophistication and generality. 
 
The critic update then takes the following form (see Appendix~\ref{appendix} for details of the derivation):
\begin{equation}
\label{v1-re10_mainpaper}
v_{t+1} = v_t +\beta_t (y_t + Y_t +\kappa_t),
\end{equation}
where ${\displaystyle
y_t = \sum_{s} \mu_{\theta_t}(s) \sum_a\pi^{\theta_t}(s,a)(R(s,a) - L^{\theta_t} - v_t^T \phi(s)}$ 
${\displaystyle + v_t^T \sum_{s'} p(s,a,s')\phi(s'))\phi(s),
}$
 \[Y_{t} = -E[\delta_t\phi(s_t)|{\cal F}_2(t)]+\delta_t \phi(s_t),\] 
${\displaystyle \kappa_t = E[(R(s_t,a_t)-L^{\theta_t}+ v_t^T \sum_{s_{t+1}} p(s_t,a_t,s_{t+1})\phi(s_{t+1})}$
$- v_t^T \phi(s_t))\phi(s_t)|\mathcal{F}_2(t)] - y_t$. 
\begin{theorem}[Stability of the Critic Recursion]
\label{di-stability}
Under Assumptions~\ref{assum:negative-definite}, \ref{assum:policy-lipschitz-bounded}, \ref{assum-ergodic-chain} and \ref{assum:ss}, the recursion (\ref{v1-re10_mainpaper}) remains uniformly bounded almost surely, i.e., $\sup_{n\rightarrow\infty} \|v_n\| <\infty$, w.p.1
\end{theorem}
\begin{proof}
See Appendix~\ref{appendix}.
\end{proof}
% The conclusions of Theorem~\ref{thm1}-\ref{thm2} continue to hold.
% We finally have the following main result.

Consider the following ODE associated with the slower recursion:
\begin{equation}
\label{ode2_mainpaper}
\dot{\theta} = \hat{\Gamma}_2\left(\nabla L^{\theta} + e^{\pi^{\theta}} \right),
\end{equation}
where  
${\displaystyle 
\hat{\Gamma}_2(v(y)) = \lim\limits_{0 < \eta \rightarrow \infty}\bigg( \frac{\Gamma_2(y + \eta v(y)) - y}{\eta}\bigg)}$ and $e^{\pi^{\theta}}$ is an error term, see the appendix.
% \end{align*}
Consider also the DI associated with the faster recursion:
\begin{equation}
\label{di-c2}
\dot{v} \in h(v),
\end{equation}
where $h(v)$ denotes the following set-valued function of $v$:
\[
h(v) = \{ \sum_s \mu_{\theta} (s)\sum_a \pi^{\theta} (s,a) (R(s,a) - L^{\theta}\]\[ + v^T \sum_{s'} p(s,a,s')\phi(s')
-v^T \phi(s))\phi(s) | \theta \in \bar{\theta^*}(v) \}.
\]

\begin{theorem}
\label{main-thm2}
Suppose the ODE (\ref{ode2_mainpaper}) has isolated local maxima $\theta^*$. Correspondingly suppose $v^*\in \mathbb{R}^{d_1}$ is a limit point of the solution to the DI (\ref{di-c2}). Then under Assumptions~\ref{assum:negative-definite}, \ref{assum:policy-lipschitz-bounded}, \ref{assum-ergodic-chain} and \ref{assum:ss}, %$\{(v_t,\theta_t)\}$ satisfy that 
$\sup_t \|v_t\| <\infty$ and $\sup_t \|\theta_t\| < \infty$ w.p.1 respectively. In addition, 
$(v_t,\theta_t) \rightarrow (v^*,\theta^*)$ almost surely, where $\theta^*$ is a local maximum of (\ref{ode2_mainpaper}) and $v^*$ is the unique solution to the projected Bellman equation corresponding to the policy $\pi^{\theta^*}$, i.e., the two together satisfy
\begin{equation}
\label{pbe2}
\Phi^T D^{\theta^*}\Phi v^* = 
\Phi^T D^{\theta^*} T_{\theta^*}(\Phi v^*).
\end{equation}
\end{theorem}
% \begin{proof}
% The proof follows from Theorem 2 of \citep{ramaswamy-bhatnagar}. 
% \end{proof}
\begin{remark}
\label{isolated-max}
We assume isolated local maxima for (\ref{ode2_mainpaper}) in Theorem~\ref{main-thm2} as it helps uniquely identify the converged policy. In the absence of this assumption, one will again obtain a DI (instead of the ODE), whose limit points the algorithm will asymptotically converge to almost surely. 
\end{remark}

% \begin{equation}
% \label{di_mainpaper}
% \dot{v}(t) \in \hat{\Gamma}_v(h(v(t))).
% \end{equation}

% where the operator $\hat{\Gamma}_2(\cdot)$ is defined as 
% \begin{align*}
%     \hat{\Gamma}_2(v(y)) = \lim\limits_{0 < \eta \rightarrow \infty}\bigg( \frac{\Gamma_2(y + \eta v(y)) - y}{\eta}\bigg).
% \end{align*}

% \begin{theorem}
% Suppose the ODE (\ref{ode2_mainpaper}) has isolated local maxima $\theta^*$. Correspondingly suppose $v^*\in C^o$ (the interior of $C$) is a limit point of the solution to the DI (\ref{di_mainpaper}). Then under Assumptions \ref{assum:policy-lipschitz-bounded}, \ref{assum:negative-definite}, \ref{assum-ergodic-chain} and \ref{assum:ss}, $\{(v_t,\theta_t)\}$ governed according to the updates in Algorithm \ref{algo} satisfy
% $(v_t,\theta_t) \rightarrow (v^*,\theta^*)$ almost surely, where $\theta^*$ is a local maximum of (\ref{ode2_mainpaper}) and $v^*$ is the unique solution to the projected Bellman equation corresponding to the policy $\pi^{\theta^*}$, i.e., the two together satisfy
% \begin{equation}
% \label{pbe}
% \Phi^T D^{\theta^*}\Phi v^* = 
% \Phi^T D^{\theta^*} T_{\theta^*}(\Phi v^*).
% \end{equation}
% \end{theorem}

\section{Experimental Results\footnote{Please fine the code at https://github.com/prashu1306/Critic-Actor.}}
\label{exp_results}

We present here the results of experiments on three different (open source) OpenAI Gym environments %(\url{https://github.com/openai/gym/}) 
namely  Frozen Lake, Pendulum  and 
Mountain Car Continuous,
respectively, 
%Detailed descriptions of these environments can be found  on https://gymnasium.farama.org/.
over which we compare the performance of CA with AC as well as the Deep Q-Network (DQN) \citep{DQN} in the average reward setting, and PPO \citep{PPO}. While \citep{bhatnagar2023actorcritic} analyzes the asymptotic convergence of the full-state CA (FS-CA) in the discounted cost setting, for experiments, they also incorporate a setting with function approximation. For the CA and AC implementations, we have thus used their code\footnote{ https://github.com/gsoumyajit/Actor-Critic-Critic-Actor}  but made changes to incorporate the average reward setting. For DQN, we have used the original code from the paper %\footnote{ https://pytorch.org/tutorials/intermediate/reinforcement\_q\_learning.html} 
and made changes to incorporate the average reward setting. For PPO, we implement two variants, namely, PPO-AC and PPO-CA, %As with CA, the timescales of PPO-CA are reversed when compared with PPO-AC. 
wherein both algorithms, clipping has been used in the actor updates and the advantage function is estimated using the critic parameter and we have used two separate losses (the actor-loss and the critic-loss), to train the actor and the critic networks respectively. We have used the average reward setting for implementing PPO (actor and critic) unlike the base implementation %https://openai.com/index/openai-baselines-ppo/ 
that considers discounted reward. 

The plots of our experiments are averaged over 10 different initial seeds after training the agent for 10,000 steps. %The performance of our algorithm is compared with the other algorithms by plotting the average reward along with the standard errors. 
Table~\ref{tab:experiment} presents the average reward along with standard error (obtained upon convergence) for all the five algorithms in the three environments while Table~\ref{sample-table2} in the Appendix presents their training time (in seconds). It can be seen from Table~\ref{tab:experiment} that CA shows the best results in all environments, though by small margins. Please note that performance may vary depending on the initial seeds. In our experiments, we used a randomly generated set of seeds, and the reported results correspond to that particular selection %Amongst the two PPO algorithms, the PPO CA is the better performer in two of the environments. 
In terms of training time performance (Table~\ref{sample-table2}), CA is better than AC and single-timescale AC on all three environments and in fact, it takes about half the run-time on two of the environments and is also better than the other algorithms as well except DQN. The latter has the best training time performance though it loses out on accuracy. %Tables~\ref{table_hyperparameter_ca}-\ref{table_hyperparameter_ppo_ca} in Appendix describe the hyper-parameters used for each algorithm.

\section{Future Work}

We used a projected critic like \citep{wu2022finite, olshevsky, chen2023finitetime}, for our non-asymptotic analysis.  
It would thus be of theoretical interest to derive similar bounds on the critic as we did but when projection is not used. 
%In our numerical experiments, we observed that while CA requires less training time than AC (almost by half in a couple of environments), it is DQN that shows the best results in terms of training time performance even though DQN loses out to CA and AC in terms of average reward performance. 
It would also be of interest to develop potentially more efficient algorithms of the CA type, such as Natural CA, Soft CA etc., and study their theoretical convergence properties as well as empirical performance.

% We further conducted experiments on three different settings in the OpenAI Gym environment and observed that our CA algorithm shows better results by small margins than the AC and the other algorithms DQN and PPO, and in fact, is the best performer on all environments. We also reversed the timescales in the PPO AC scheme and observed that PPO CA is better than PPO-AC on two of these environments. We expect our work to lead to more research on this previously unstudied class of CA algorithms. 

\section*{Acknowledgments}
The authors were supported by
the Walmart Centre for Tech Excellence, Indian Institute of Science. S.~Bhatnagar was supported additionally by a J.C. Bose Fellowship, the Kotak-IISc AI/ML Centre, Indian Institute of Science, Project No.~DFTM/02/3125/M/04/AIR-04 from DRDO under DIA-RCOE, and the Robert Bosch Centre for Cyber Physical Systems, Indian Institute of Science.

\bibliography{aaai25}

\newpage

\onecolumn
\appendix

\section{Appendix}
\label{appendix}

The appendix comprises of three parts. First, we present the complete finite-time (non-asymptotic) analysis of the CA algorithm where we show that it achieves a sample complexity of $\mathcal{\tilde{O}}(\epsilon^{-(2+\delta)})$ for the mean-squared error of the critic to be upper bounded by $\epsilon$. 

Subsequently, we present the complete analysis of asymptotic convergence of the two-timescale CA algorithm. In particular, we show that the iterates of the algorithm remain uniformly bounded almost surely, i.e., the iterates are stable, and in addition the algorithm is almost surely convergent. It is important to note that asymptotic convergence guarantees for many algorithms in the literature such as the single-timescale AC algorithms of \citep{olshevsky, chen2023finitetime} are not available. In fact, it may not be possible to show such guarantees in the case of single-timescale AC algorithms because of the potential violation of the inherent nested loop structure required in policy iteration algorithms that gets mimicked via a difference in timescales. 

Finally, we provide details of the hyper-parameter settings used for the various algorithms as well as the performance comparisons of all the algorithms in terms of the training time required in seconds.

\subsection{Finite-Time Analysis}
\label{finitetimeanalysis}

We present here the details of the finite time analysis of our two-timescale CA algorithm. Recall that our algorithm comprises of three recursions, viz., the average reward recursion, the actor update and the critic update, respectively. The actor update in our algorithm proceeds faster than the critic update. Further, the average reward estimate is used in the temporal difference term $\delta_t$, that in turn is used in both the actor and the critic updates. Moreover, there is only a one-way coupling between the average reward estimate and the actor/critic estimates in the sense that the actor and critic estimates depend on the average reward but not vice versa. Hence, we use the actor's timescale to update the average reward recursion as well. Thus, in our algorithm, even though there are three recursions, the average reward and actor recursions together proceed on the same timescale that is faster than the timescale of the critic update (which is slower). In Sections \ref{average_reward_convergence}--\ref{critic_convergence_proof} below, we present the analysis for these three recursions and obtain the sample complexity estimate for the algorithm.
Notice the difference of our scheme with standard AC algorithms for average reward MDPs such as those in \citep{wu2022finite, BHATNAGAR20092471}, where the average reward recursion proceeds on the same (faster) timescale as the critic recursion while the actor recursion proceeds slower.

We start off with a basic result that provides some sufficient conditions that imply  Assumption~\ref{smoothness_mu}.

\begin{theorem}
\label{thm-stationary}
Under Assumptions~\ref{assum:policy-lipschitz-bounded} and \ref{assum-ergodic-chain}, 
the stationary distribution $\mu_\theta$ of the Markov chain $\{s_t\}$ is continuously differentiable in $\theta \in \mathbb{R}^d$, where $\theta$ is the policy parameter. Further, if $\nabla^2 \mu_\theta(s)$ exists for each $\theta\in \mathbb{R}^d$ and $s\in S$, and further,  $\sup_{\theta,s} \|\nabla^2 \mu_\theta(s)\| \leq L_\mu$ for some $L_\mu>0$, then $\mu_\theta$ is $L_\mu$-smooth.   
\end{theorem}

\begin{proof}
Let $P(\theta)$ denote the transition probability matrix with policy parameter $\theta$. Also, let   \[Z(\theta)= [I-P(\theta)+P^\infty(\theta)]^{-1},
\]
where ${\displaystyle P^\infty(\theta) = \frac{1}{m}\sum_{k=1}^{m} P^k(\theta)}$ is the time averaged transition probability matrix, where $P^k(\theta)$ is the $k$-step transition probability matrix. Since the state-valued process is ergodic Markov for any $\theta$, it follows that $P_{ij}^\infty(\theta) = \mu_j(\theta)$, $\forall i,j=1,\ldots,n$. From Assumption~\ref{assum:policy-lipschitz-bounded}, $\nabla \pi_\theta$ exists and is in fact uniformly bounded over all $\theta \in \mathbb{R}^d$. Thus, $\nabla P(\theta)$ exists as well (and is also uniformly bounded). It now follows from Theorem 2 of \citep{schweitzer}, that $\mu_\theta$ is continuously differentiable and in fact,
\[
\nabla \mu_\theta = \mu_\theta \nabla P(\theta) Z(\theta).
\] 
Now observe that from the mean-value theorem, for any $s\in S$,
\[
\|\nabla\mu_{\theta_1}(s)-\nabla\mu_{\theta_2}(s)\| \leq \|\nabla^2 \mu_\xi(s)\| \|\theta_1-\theta_2\|,
\]
where $\xi = \alpha \theta_1 + (1-\alpha)\theta_2$ for some $\alpha \in [0,1]$. The claim now follows from the fact that $\sup_{\xi,s} \|\nabla^2\mu_\xi(s)\| \leq L_\mu$. 
\end{proof}

\subsubsection{Convergence of the Average Reward Estimate}\label{average_reward_convergence}

Notations:-
\begin{align}
    \begin{split}
        O_t :&= (s_t , a_t , s_{t+1})\\
        y_t :&= (L_t - L(\theta_t))\\
        M(\theta_t,v_t) :& = E_{s_t \sim \mu_{\theta_t},a_t \sim \pi_{\theta_t},s_{t+1} \sim p}[( r(s_t,a_t)- L(\theta_t) + \phi(s_{t+1})^{\top} v_{t} \\
        &\qquad- \phi(s_t)^{\top} v_{t})\nabla \log\pi_{\theta_t}(a_t|s_t)]\\
        W(v,\theta) :&= E_{s \sim \mu_{\theta},a \sim \pi_{\theta},s^{'} \sim P}[(V^{\theta}(s^{'}) - v^T\phi(s^{'}) - V^{\theta}(s) + v^T\phi(s))\nabla \log\pi_{\theta}(a|s)]\\
        N(O_t ,\theta_t,v_t,L_t) :& = ( r(s_t,a_t)- L_t  + \phi(s_{t+1})^{\top} v_{t} - \phi(s_t)^{\top} v_{t})\nabla \log\pi_{\theta_t}(a_t|s_t)\\
        \Omega(O_t,\theta_t,v_t,L_t) :&= y_t\langle W(v_t,\theta_t) , -N(O_t,\theta_t,v_t,L_t) + E_{\theta_t}[N(O_t,\theta_t,v_t,L_t)] \rangle\\
        U_w :&=  2B(U_{v} +  \bar{U}_{v})\\
        G :&= 2B(U_r + U_v)
    \end{split}
\end{align}

We have , $\vert V^{\theta}(s) \vert \leq \bar{U}_{v} ,\forall \theta \in \RR^{d},\forall s \in S$.

The following lemmas will be useful in proving the convergence of the average reward estimate.

\begin{lemma}\label{L_smoothness}
For the performance function $L(\theta)$, there exists a constant $L_{J'}>0$ such that for all $\theta_1,\theta_2\in\mathbb{R}^d$, it holds that
\begin{align}\label{LJ'}
    \Vert \nabla L(\theta_1)-\nabla L(\theta_2)\Vert \leq L_{J'}\Vert \theta_1-\theta_2\Vert,
\end{align}
which further implies
\begin{align}
    L(\theta_2)\ge\  &L(\theta_1)+\langle \nabla L(\theta_1),\theta_2-\theta_1\rangle -\frac{L_{J'}}{2}\Vert \theta_1-\theta_2\Vert^2,\label{l-smooth1}\\
    L(\theta_2)\leq\  &L(\theta_1)+\langle \nabla L(\theta_1),\theta_2-\theta_1\rangle +\frac{L_{J'}}{2}\Vert \theta_1-\theta_2\Vert^2\label{l-smooth2}.
\end{align}

\end{lemma}

\begin{proof}
    Please refer proof of Lemma C.1 in \citep{wu2022finite}.
\end{proof}

\begin{lemma}\label{omega}

Under assumptions \ref{assum:bounded_feature_norm}, \ref{assum:ergodicity}, \ref{assum:policy-lipschitz-bounded}, \ref{V_lipschitz}, for any $t \geq \tau > 0$ , we have
\begin{align*}
    E[\Omega(O_t,\theta_t,v_t,L_t)] &\leq (2U_{w}G + 4U_{r}U_{w}B)\vert L_t - L_{t-\tau} \vert + 8BU_{r}(G + U_{w})\Vert v_t - v_{t-\tau} \Vert\\
    &\qquad + M_1\Vert \theta_t - \theta_{t-\tau} \Vert
    + M_2\sum\limits_{i = t-\tau}^{t}E\Vert \theta_{i} - \theta_{t-\tau} \Vert +  M_3bk^{\tau-1},
\end{align*}
for some $M_1 > 0$, $M_2 > 0$ and $M_3 > 0$.
\end{lemma}

\begin{proof}
    We can write $ E[\Omega(O_t,\theta_t,v_t,L_t)]$ as 
    \begin{align*}
        &E[\Omega(O_t,\theta_t,v_t,L_t)] \\
        =& E[\Omega(O_t,\theta_t,v_t,L_t)] - E[\Omega(O_t,\theta_t,v_t,L_{t-\tau})]   + E[\Omega(O_t,\theta_t,v_t,L_{t-\tau})]\\
        &\qquad - E[\Omega(O_t,\theta_t,v_{t-\tau},L_{t-\tau})]
       + E[\Omega(O_t,\theta_t,v_{t-\tau},L_{t-\tau})] - E[\Omega(O_t,\theta_{t-\tau},v_{t-\tau},L_{t-\tau})]\\
       &\qquad + E[\Omega(O_t,\theta_{t-\tau},v_{t-\tau},L_{t-\tau})]
        - E[\Omega(\tilde{O_t},\theta_{t-\tau},v_{t-\tau},L_{t-\tau})] + E[\Omega(\tilde{O_t},\theta_{t-\tau},v_{t-\tau},L_{t-\tau})] \\
        &\qquad- E[\Omega(O_t^{'},\theta_{t-\tau},v_{t-\tau},L_{t-\tau})] 
         + E[\Omega(O_t^{'},\theta_{t-\tau},v_{t-\tau},L_{t-\tau})]
    \end{align*}

In the above equality, $\tilde{O}_{t}$ represents the tuple $(\widetilde{s}_t,\widetilde{a}_t,\widetilde{s}_{t+1})$, which is generated in the following manner :    
\begin{align}\label{chain-au}
    s_{t-\tau}\xrightarrow{\theta_{t-\tau}}a_{t-\tau}\xrightarrow{P}s_{t-\tau+1}\xrightarrow{\theta_{t-\tau}}\widetilde{a}_{t-\tau+1}\xrightarrow{P}\widetilde{s}_{t-\tau+2}\xrightarrow{\theta_{t-\tau}}\widetilde{a}_{t-\tau+2}\cdots \xrightarrow{P}\widetilde{s}_t\xrightarrow{\theta_{t-\tau}}\widetilde{a}_t\xrightarrow{P}\widetilde{s}_{t+1}.
\end{align}
Thus, the policy parameter $\theta_{t-\tau}$ is held fixed for $\tau$ instants starting from the state $s_{t-\tau}$ in the original process. 
Here, for any time instant $k > t-\tau$, $\widetilde{a}_{k}$ denotes the action taken under $\theta_{t-\tau}$. Similarly, for any time instant $l > t-\tau + 1$, $\widetilde{s}_{l}$ denotes the state with actions chosen under the policy parameter $\theta_{t-\tau}$ held fixed. In this auxiliary chain, policy $\pi_{\theta_{t-\tau}}$ is repeatedly applied starting from state $s_{t-\tau }$.

Note that the original Markov chain has the following transitions:
\begin{align}\label{chain-or}
    s_{t-\tau}\xrightarrow{\theta_{t-\tau}}a_{t-\tau}\xrightarrow{\mathcal{P}}s_{t-\tau+1}\xrightarrow{\theta_{t-\tau+1}}a_{t-\tau+1}\xrightarrow{\mathcal{P}}s_{t-\tau+2}\xrightarrow{\theta_{t-\tau+2}}a_{t-\tau+2}\cdots \xrightarrow{\mathcal{P}}s_t\xrightarrow{\theta_{t}}a_t\xrightarrow{\mathcal{P}}s_{t+1}.
\end{align}

Moreover, $O_t^{'} = (s^{'}_t , a^{'}_t ,s^{'}_{t+1} )$, where $ s^{'}_t \sim \mu_{\theta_{t-\tau}} , a^{'}_t \sim \pi_{\theta_{t-\tau}}$ and $s^{'}_{t+1} \sim P(\cdot|s^{'}_t, a^{'}_t)$.
%In the above equality $\tilde{O_t}$ is the auxiliary markov chain defined in \ref{chain-au} and $O_t^{'} = (s^{'}_t , a^{'}_t ,s^{'}_{t+1} )$ where $ s^{'}_t \sim \mu_{\theta_{t-\tau}} , a^{'}_t \sim \pi_{\theta_{t-\tau}}$ and $s^{'}_{t+1} \sim P(.|s^{'}_t, a^{'}_t)$.
The remainder of the proof of Lemma \ref{omega} is based on the results of the auxiliary lemmas \ref{omega1}- \ref{omega5} below (that we now show).

We let $t \geq \tau > 0$ in the following auxiliary lemmas.
\begin{sublemma}\label{omega1}
    \begin{align*}
        E[\Omega(O_t,\theta_t,v_t,L_t)] - E[\Omega(O_t,\theta_t,v_t,L_{t-\tau})] \leq (2U_{w}G + 4U_{r}U_{w}B)\vert L_t - L_{t-\tau} \vert
    \end{align*}
\end{sublemma}

\begin{proof}
\begin{align*}
    &E[\Omega(O_t,\theta_t,v_t,L_t)] - E[\Omega(O_t,\theta_t,v_t,L_{t-\tau})]\\
    &= E[y_t\langle W(v_t,\theta_t) , -N(O_t,\theta_t,v_t,L_t) + E_{\theta_t}[N(O_t,\theta_t,v_t,L_t)] \rangle\\
    &\qquad - (L_{t-\tau} - L(\theta_t))\langle W(v_t,\theta_t) , -N(O_t,\theta_t,v_t,L_{t-\tau}) + E_{\theta_t}[N(O_t,\theta_t,v_t,L_{t-\tau})] \rangle]\\
    &= E[y_t\langle W(v_t,\theta_t) , -N(O_t,\theta_t,v_t,L_t) + E_{\theta_t}[N(O_t,\theta_t,v_t,L_t)] \rangle\\
    &\qquad - (L_{t-\tau} - L(\theta_t))\langle W(v_t,\theta_t) , -N(O_t,\theta_t,v_t,L_{t}) + E_{\theta_t}[N(O_t,\theta_t,v_t,L_{t})] \rangle]\\
    &\qquad + (L_{t-\tau} - L(\theta_t))\langle W(v_t,\theta_t) , -N(O_t,\theta_t,v_t,L_{t}) + E_{\theta_t}[N(O_t,\theta_t,v_t,L_{t})] \rangle]\\
    &\qquad - (L_{t-\tau} - L(\theta_t))\langle W(v_t,\theta_t) , -N(O_t,\theta_t,v_t,L_{t-\tau}) + E_{\theta_t}[N(O_t,\theta_t,v_t,L_{t-\tau})] \rangle]\\
    & = E[(L_t - L_{t-\tau})\langle W(v_t,\theta_t) , -N(O_t,\theta_t,v_t,L_t) + E_{\theta_t}[N(O_t,\theta_t,v_t,L_t)] \rangle]\\
    &\qquad + E[(L_{t-\tau} - L(\theta_t))\langle W(v_t,\theta_t) , N(O_t,\theta_t,v_t,L_{t-\tau}) -N(O_t,\theta_t,v_t,L_{t}) \\
    &\qquad + E_{\theta_t}[N(O_t,\theta_t,v_t,L_{t}) - N(O_t,\theta_t,v_t,L_{t-\tau})] \rangle]\\
    &\leq 2U_{w}G\vert L_t - L_{t-\tau} \vert + 4U_{r}U_{w}B\vert L_t - L_{t-\tau} \vert\\
    &=2U_{w}(G + 2U_{r}B)\vert L_t - L_{t-\tau} \vert.
\end{align*}
The claim follows. 
\end{proof}

\begin{sublemma}\label{omega2}
    \begin{align*}
        E[\Omega(O_t,\theta_t,v_t,L_{t-\tau})] - E[\Omega(O_t,\theta_t,v_{t-\tau},L_{t-\tau})] \leq 8BU_{r}(G + U_{w})\Vert v_t - v_{t-\tau} \Vert.
    \end{align*}
\end{sublemma}

\begin{proof}
    \begin{align*}
        &E[\Omega(O_t,\theta_t,v_t,L_{t-\tau})] - E[\Omega(O_t,\theta_t,v_{t-\tau},L_{t-\tau})]\\
        &= E[(L_{t-\tau} - L(\theta_t))\langle W(v_t,\theta_t) , -N(O_t,\theta_t,v_t,L_{t-\tau}) + E_{\theta_t}[N(O_t,\theta_t,v_t,L_{t-\tau})] \rangle]\\
        &\qquad- E[(L_{t-\tau} - L(\theta_t))\langle W(v_{t-\tau},\theta_t) , -N(O_t,\theta_t,v_{t-\tau},L_{t-\tau})\\
        &\qquad+ E_{\theta_t}[N(O_t,\theta_t,v_{t-\tau},L_{t-\tau})] \rangle]\\
        &= E[(L_{t-\tau} - L(\theta_t))\langle W(v_t,\theta_t) , -N(O_t,\theta_t,v_t,L_{t-\tau}) + E_{\theta_t}[N(O_t,\theta_t,v_t,L_{t-\tau})] \rangle]\\
        &\qquad - E[(L_{t-\tau} - L(\theta_t))\langle W(v_{t-\tau},\theta_t) , -N(O_t,\theta_t,v_t,L_{t-\tau})\\
        &\qquad+ E_{\theta_t}[N(O_t,\theta_t,v_t,L_{t-\tau})] \rangle]\\
        &\qquad + E[(L_{t-\tau} - L(\theta_t))\langle W(v_{t-\tau},\theta_t) , -N(O_t,\theta_t,v_t,L_{t-\tau}) + E_{\theta_t}[N(O_t,\theta_t,v_t,L_{t-\tau})] \rangle]\\
        &\qquad- E[(L_{t-\tau} - L(\theta_t))\langle W(v_{t-\tau},\theta_t) , -N(O_t,\theta_t,v_{t-\tau},L_{t-\tau}) + E_{\theta_t}[N(O_t,\theta_t,v_{t-\tau},L_{t-\tau})] \rangle]\\
        &=E[(L_{t-\tau} - L(\theta_t))\langle W(v_t,\theta_t) - W(v_{t-\tau},\theta_t)  , -N(O_t,\theta_t,v_t,L_{t-\tau})\\
        &\qquad + E_{\theta_t}[N(O_t,\theta_t,v_t,L_{t-\tau})] \rangle] + E[(L_{t-\tau} - L(\theta_t))\langle W(v_{t-\tau},\theta_t) , N(O_t,\theta_t,v_{t-\tau},L_{t-\tau})\\
        &\qquad -N(O_t,\theta_t,v_t,L_{t-\tau})
       + E_{\theta_t}[N(O_t,\theta_t,v_t,L_{t-\tau}) - N(O_t,\theta_t,v_{t-\tau},L_{t-\tau})] \rangle]\\
        & \leq 4U_{r}G\Vert W(v_t,\theta_t) - W(v_{t-\tau},\theta_t)  \Vert + 4U_{r}U_{w}\Vert N(O_t,\theta_t,v_t,L_{t-\tau}) - N(O_t,\theta_t,v_{t-\tau},L_{t-\tau})\Vert\\
        &\leq 8BU_{r}G\Vert v_t - v_{t-\tau} \Vert + 8U_{r}U_{w}B\Vert v_t - v_{t-\tau}\Vert.
    \end{align*}
\end{proof}

\begin{sublemma}\label{omega3}
    \begin{align*}
       E[\Omega(O_t,\theta_t,v_{t-\tau},L_{t-\tau})] - E[\Omega(O_t,\theta_{t-\tau},v_{t-\tau},L_{t-\tau})] \leq M_1\Vert \theta_t - \theta_{t-\tau} \Vert.
    \end{align*}

for some $M_1 > 0$.
\end{sublemma}

\begin{proof}
\begin{align*}
     &E[\Omega(O_t,\theta_t,v_{t-\tau},L_{t-\tau})] - E[\Omega(O_t,\theta_{t-\tau},v_{t-\tau},L_{t-\tau})] \\
     & = E[(L_{t-\tau} - L(\theta_t))\langle W(v_{t-\tau},\theta_t) , -N(O_t,\theta_t,v_{t-\tau},L_{t-\tau}) + E_{\theta_t}[N(O_t,\theta_t,v_{t-\tau},L_{t-\tau})] \rangle]\\
     &\qquad - E[(L_{t-\tau} - L(\theta_{t-\tau}))\langle W(v_{t-\tau},\theta_{t-\tau}) , -N(O_t,\theta_{t-\tau},v_{t-\tau},L_{t-\tau})\\
     &\qquad+ E_{\theta_{t-\tau}}[N(O_t,\theta_{t-\tau},v_{t-\tau},L_{t-\tau})] \rangle]\\
     &=  E[(L_{t-\tau} - L(\theta_t))\langle W(v_{t-\tau},\theta_t) , -N(O_t,\theta_t,v_{t-\tau},L_{t-\tau}) + E_{\theta_t}[N(O_t,\theta_t,v_{t-\tau},L_{t-\tau})] \rangle]\\
     &\qquad - E[(L_{t-\tau} - L(\theta_{t-\tau}))\langle W(v_{t-\tau},\theta_t) , -N(O_t,\theta_t,v_{t-\tau},L_{t-\tau}) + E_{\theta_t}[N(O_t,\theta_t,v_{t-\tau},L_{t-\tau})] \rangle]\\
     &\qquad + E[(L_{t-\tau} - L(\theta_{t-\tau}))\langle W(v_{t-\tau},\theta_t) , -N(O_t,\theta_t,v_{t-\tau},L_{t-\tau}) + E_{\theta_t}[N(O_t,\theta_t,v_{t-\tau},L_{t-\tau})] \rangle]\\
     &\qquad - E[(L_{t-\tau} - L(\theta_{t-\tau}))\langle W(v_{t-\tau},\theta_{t-\tau}) , -N(O_t,\theta_{t-\tau},v_{t-\tau},L_{t-\tau})\\
     &\qquad + E_{\theta_{t-\tau}}[N(O_t,\theta_{t-\tau},v_{t-\tau},L_{t-\tau})] \rangle]\\
     &= E[(L(\theta_{t-\tau}) - L(\theta_t))\langle W(v_{t-\tau},\theta_t) , -N(O_t,\theta_t,v_{t-\tau},L_{t-\tau}) + E_{\theta_t}[N(O_t,\theta_t,v_{t-\tau},L_{t-\tau})] \rangle]\\
     &\qquad + E[(L_{t-\tau} - L(\theta_{t-\tau}))\langle W(v_{t-\tau},\theta_t) , -N(O_t,\theta_t,v_{t-\tau},L_{t-\tau}) + E_{\theta_t}[N(O_t,\theta_t,v_{t-\tau},L_{t-\tau})] \rangle]\\
     &\qquad - E[(L_{t-\tau} - L(\theta_{t-\tau}))\langle W(v_{t-\tau},\theta_{t-\tau}) , -N(O_t,\theta_t,v_{t-\tau},L_{t-\tau})\\
     &\qquad \qquad+ E_{\theta_t}[N(O_t,\theta_t,v_{t-\tau},L_{t-\tau})] \rangle]\\
     &\qquad + E[(L_{t-\tau} - L(\theta_{t-\tau}))\langle W(v_{t-\tau},\theta_{t-\tau}) , -N(O_t,\theta_t,v_{t-\tau},L_{t-\tau}) \\
     &\qquad \qquad + E_{\theta_t}[N(O_t,\theta_t,v_{t-\tau},L_{t-\tau})] \rangle]\\
     &\qquad - E[(L_{t-\tau} - L(\theta_{t-\tau}))\langle W(v_{t-\tau},\theta_{t-\tau}) , -N(O_t,\theta_{t-\tau},v_{t-\tau},L_{t-\tau})\\
     &\qquad + E_{\theta_{t-\tau}}[N(O_t,\theta_{t-\tau},v_{t-\tau},L_{t-\tau})] \rangle]\\
     &\leq 2GU_{w}\Vert L_{t-\tau} - L(\theta_{t-\tau}) \Vert + 4U_{r}U_{G}\Vert W(v_{t-\tau},\theta_{t}) - W(v_{t-\tau},\theta_{t-\tau}) \Vert\\
     &\qquad + 4U_{r}U_{w}\Vert N(O_t,\theta_{t},v_{t-\tau},L_{t-\tau}) - N(O_t,\theta_{t-\tau},v_{t-\tau},L_{t-\tau}) \Vert\\
     &= \mathcal{O}(\Vert \theta_t - \theta_{t-\tau} \Vert).
\end{align*}
The claim follows. 
\end{proof}

\begin{sublemma}\label{omega4}
    \begin{align*}
        E[\Omega(O_t,\theta_{t-\tau},v_{t-\tau},L_{t-\tau})]  -  E[\Omega(\tilde{O_t},\theta_{t-\tau},v_{t-\tau},L_{t-\tau})] \leq M_2\sum\limits_{i = t-\tau}^{t}E\Vert \theta_{i} - \theta_{t-\tau} \Vert
    \end{align*}
for some  $M_2 > 0$.
\end{sublemma}

\begin{proof}
    The proof is similar to that of Lemma D.10 in \citep{wu2022finite}, hence we omit the same here.
\end{proof}

\begin{sublemma}\label{omega5}
    \begin{align*}
        E[\Omega(\tilde{O_t},\theta_{t-\tau},v_{t-\tau},L_{t-\tau})] - E[\Omega(O_t^{'},\theta_{t-\tau},v_{t-\tau},L_{t-\tau})] \leq M_3bk^{\tau-1},
    \end{align*}
for some $M_3 > 0$.
\end{sublemma}

\begin{proof}
    The proof is similar to that of Lemma D.11 in \citep{wu2022finite}, hence we omit the same here.
\end{proof}

It is clear from the definition in Section \ref{average_reward_convergence} that  $ E[\Omega(O_t^{'},\theta_{t-\tau},v_{t-\tau},L_{t-\tau})] = 0$. Now collecting the results from lemmas \ref{omega1} - \ref{omega5}, we have,

\begin{align*}
    E[\Omega(O_t,\theta_t,v_t,L_t)] &\leq (2U_{w}G + 4U_{r}U_{w}B)\vert L_t - L_{t-\tau} \vert + 8BU_{r}(G + U_{w})\Vert v_t - v_{t-\tau} \Vert \\
    &\qquad + M_1\Vert \theta_t - \theta_{t-\tau} \Vert
   + M_2\sum\limits_{i = t-\tau}^{t}E\Vert \theta_{i} - \theta_{t-\tau} \Vert +  M_3bk^{\tau-1}.
\end{align*}
The claim in Lemma~\ref{omega} now follows. 
\end{proof}

  \subsubsection*{Proof of convergence of the average reward estimate.}

From the update rule of the reward estimation recursion in  Algorithm \ref{algo}, we have
\begin{align*}
    L_{t+1}-L(\theta_{t+1})=L_t-L(\theta_t)+L(\theta_t)-L(\theta_{t+1})+\gamma_t(r_t-L_t).
\end{align*}
We then have
\begin{align*}
    y_{t+1}^2=&\ (y_t+L(\theta_t)-L(\theta_{t+1})+\gamma_t(r_t-L_t))^2\\
    \leq &\ y_t^2+2y_t(L(\theta_t)-L(\theta_{t+1}))+2\gamma_t y_t(r_t-L_t)+2(L(\theta_t)-L(\theta_{t+1}))^2+2\gamma_t^2(r_t-L_t)^2\\
    = &\  (1-2\gamma_t)y_t^2+2\gamma_ty_t(r_t-L(\theta_t))+2y_t(L(\theta_t)-L(\theta_{t+1}))+2(L(\theta_t)-L(\theta_{t+1}))^2\\
    &\qquad+2\gamma_t^2(r_t-L_t)^2.
\end{align*}

Taking expectations, rearranging  and summing from $\tau_t$ to $t$ we obtain,

\begin{align*}
    \sum\limits_{k=\tau_t}^{t} \mathbb{E}[y_k^2]\leq &\  \underbrace{\sum\limits_{t=\tau_t}^t\frac{1}{2\gamma_k}\mathbb{E}(y_k^2-y^2_{k+1})}_{I_1}+\underbrace{\sum\limits_{k=\tau_t}^{t}\mathbb{E}[y_k(r_k-L(\theta_k))]}_{I_2}+\underbrace{\sum\limits_{k=\tau_t}^{t}\frac{1}{\gamma_k}\mathbb{E}[y_k(L(\theta_k)-L(\theta_{k+1})]}_{I_3}\\
    &\  +\underbrace{\sum\limits_{k=\tau_t}^{t}\frac{1}{\gamma_k}\mathbb{E}[(L(\theta_k)-L(\theta_{k+1}))^2]}_{I_4}+\underbrace{\sum\limits_{k=\tau_t}^{t}\gamma_k\mathbb{E}[(r_k-L_k)^2]}_{I_5}.
\end{align*}

For term $I_1$, from Abel summation by parts, we have
\begin{align*}
    I_1=& \sum\limits_{k=\tau_t}^{t}\frac{1}{2\gamma_k}(y_k^2-y_{k+1}^2)\\
    =&\  \sum\limits_{k=\tau_t+1}^{t}y_k^2(\frac{1}{2\gamma_k}-\frac{1}{2\gamma_{k-1}})+\frac{1}{2\gamma_{\tau_t}}y_{\tau_t}^2-\frac{1}{\gamma_{t}}y_{t+1}^2\\
    \leq &\ \frac{2U_r^2}{\gamma_{t}}\\
    = &\  2c_{\gamma}U_r^2(1+t)^{\nu}.
\end{align*}

For term $I_2$, we have
\begin{align*}
    \sum\limits_{k=\tau_t}^{t}\mathbb{E}[y_k(r_k-L(\theta_k))] = \mathcal{O}(\log^2 t \cdot t^{1-\nu}).
\end{align*}
The analysis of part $I_2$ will be similar to the one of part $I_2$ in section C.2 of \citep{wu2022finite}.

For $I_3$, if $y_t>0$, from \eqref{l-smooth1}, we have
\begin{align*}
    &y_t(L(\theta_t)-L(\theta_{t+1}))\\
    \leq &\ y_t(\frac{L_{J'}}{2}\Vert \theta_t-\theta_{t+1}\Vert^2+\langle \nabla L(\theta_t), \theta_t-\theta_{t+1}\rangle)\\
    \leq &\ L_{J'}U_r\Vert \theta_t-\theta_{t+1}\Vert^2 + y_t\langle M(\theta_t,v_t),\theta_t - \theta_{t+1}\rangle\\
    &\qquad+ y_t\langle E_{\theta_t}[(V^{\theta_{t}}(s_{t+1}) - v(t)^T\phi(s_{t+1}) - V^{\theta_{t}}(s_{t}) + v(t)^T\phi(s_{t}))\nabla \log\pi_{\theta_t}(a_t|s_t)] \\
    &\qquad, \theta_t - \theta_{t+1} \rangle 
\end{align*}
If $y_t\leq 0$, from \eqref{l-smooth2}, we have
\begin{align*}
   & y_t(L(\theta_t)-L(\theta_{t+1}))\\
   \leq&\  y_t(-\frac{L_{J'}}{2}\Vert \theta_t-\theta_{t+1}\Vert^2+\langle \nabla L(\theta_t), \theta_t-\theta_{t+1}\rangle)\\
    \leq &\ L_{J'}U_r\Vert \theta_t-\theta_{t+1}\Vert^2 + y_t\langle M(\theta_t,v_t),\theta_t - \theta_{t+1}\rangle\\
    &+ y_t\langle E_{\theta_t}[(V^{\theta_{t}}(s_{t+1}) - v(t)^T\phi(s_{t+1}) - V^{\theta_{t}}(s_{t}) + v(t)^T\phi(s_{t}))\nabla \log\pi_{\theta_t}(a_t|s_t)], 
    \theta_t - \theta_{t+1} \rangle.
\end{align*}

Overall, we get
\begin{align*}
    I_3=& \sum\limits_{k=\tau_t}^{t}\frac{1}{\gamma_k}\mathbb{E}[y_k(L(\theta_k)-L(\theta_{k+1}))]\\
    \leq &\ \sum\limits_{k=\tau_t}^{t}\frac{1}{\gamma_k}\mathbb{E}[L_{J'}U_r\Vert \theta_k-\theta_{k+1}\Vert^2+|y_k|\Vert \theta_k-\theta_{k+1}\Vert \Vert M(\theta_k,v_k)\Vert]\\
     &\qquad + \sum\limits_{k=\tau_t}^{t}\frac{1}{\gamma_k}E[y_k\langle E_{\theta_k}[(V^{\theta_{k}}(s_{k+1}) - v(k)^T\phi(s_{k+1}) - V^{\theta_{k}}(s_{k}) \\
     &\qquad+ v(k)^T\phi(s_{k}))\nabla \log\pi_{\theta_k}(a_k|s_k)] , \theta_k - \theta_{k+1} \rangle]\\
    \leq &\ \sum\limits_{k=\tau_t}^{t}\mathbb{E}[L_{J'}U_rG^2\frac{\alpha_k^2}{\gamma_k} + G\frac{c_\alpha}{c_\gamma}|y_k|\Vert M(\theta_k,v_k)\Vert]\\
     &\qquad + \sum\limits_{k=\tau_t}^{t}\frac{1}{\gamma_k}E[y_k\langle E_{\theta_k}[(V^{\theta_{k}}(s_{k+1}) - v(k)^T\phi(s_{k+1}) - V^{\theta_{k}}(s_{k})\\
     &\qquad+ v(k)^T\phi(s_{k}))\nabla \log\pi_{\theta_k}(a_k|s_k)] , \theta_k - \theta_{k+1} \rangle]\\
    \leq &\ \frac{2L_{J'}U_rG^2c_\alpha^2}{c_\gamma}(1+t-\tau_t)^{1-\nu}+ G\frac{c_\alpha}{c_\gamma}(\sum\limits_{k=\tau_t}^{t}\mathbb{E}y_t^2)^{\frac{1}{2}}(\sum\limits_{k=\tau_t}^{t}\mathbb{E}\Vert M(\theta_k,v_k)\Vert^2)^\frac{1}{2}\\
    &\qquad + \sum\limits_{k=\tau_t}^{t}\frac{1}{\gamma_k}E[y_k\langle E_{\theta_k}[(V^{\theta_{k}}(s_{k+1}) - v(k)^T\phi(s_{k+1}) - V^{\theta_{k}}(s_{k})\\
    &\qquad+ v(k)^T\phi(s_{k}))\nabla \log\pi_{\theta_k}(a_k|s_k)] , \theta_k - \theta_{k+1} \rangle]\\
    = &\ \frac{2L_{J'}U_rG^2c_\alpha^2}{c_\gamma}(1+t-\tau_t)^{1-\nu}+ G\frac{c_\alpha}{c_\gamma}(\sum\limits_{k=\tau_t}^{t}\mathbb{E}y_t^2)^{\frac{1}{2}}(\sum\limits_{k=\tau_t}^{t}\mathbb{E}\Vert M(\theta_k,v_k)\Vert^2)^\frac{1}{2}\\
    &\qquad +  \underbrace{\sum\limits_{k=\tau_t}^{t}\frac{c_\alpha}{c_\gamma}E[y_k\langle W(v_k,\theta_k) , -\delta_k \nabla_{\theta} \log \pi_{\theta_{k}}(s_k|a_k) + E_{\theta_k}[\delta_k \nabla_{\theta} \log \pi_{\theta_{k}}(s_k|a_k)] \rangle]}_{I_a}\\
    &\qquad + \underbrace{\sum\limits_{k=\tau_t}^{t}\frac{c_\alpha}{c_\gamma}E[y_k\langle W(v_k,\theta_k) , -E_{\theta_k}[\delta_k \nabla_{\theta} \log \pi_{\theta_{k}}(s_k|a_k)] \rangle]}_{I_b}
\end{align*}

For term $I_a$, we have,

\begin{align*}
    &\frac{c_\alpha}{c_\gamma}\sum\limits_{k=\tau_t}^{t}E[y_k\langle W(v_k,\theta_k) , -N(O_k,\theta_k,v_k,L_k) + E_{\theta_k}[N(O_k,\theta_k,v_k,L_k)] \rangle]\\
    & = \frac{c_\alpha}{c_\gamma}\sum\limits_{k=\tau_t}^{t}\bigg((2U_{w}G + 4U_{r}U_{w}B)\vert L_k - L_{k-\tau} \vert + 8BU_{r}(G + U_{w})\Vert v_k - v_{k-\tau} \Vert \\
    &\qquad + M_1\Vert \theta_k - \theta_{k-\tau} \Vert
   + M_2\sum\limits_{i = k-\tau}^{k}E\Vert \theta_{i} - \theta_{k-\tau} \Vert +  M_3bk^{m-1}\bigg).
\end{align*}

Taking $\tau := \tau_t$, we have,
\begin{align*}
    I_a = \mathcal{O}(\tau_t^2 \cdot t^{1-\nu}).
\end{align*}

For term $I_b$, we have,

\begin{align*}
    &\sum\limits_{k=\tau_t}^{t}\frac{c_\alpha}{c_\gamma}E[y_k\langle W(v_k,\theta_k) , -E_{\theta_k}[\delta_k \nabla_{\theta} \log \pi_{\theta_{k}}(s_k|a_k)] \rangle]\\
    &= \frac{c_\alpha}{c_\gamma}\sum\limits_{k=\tau_t}^{t}E[y_k\langle W(v_k,\theta_k) , -M(\theta_k,v_k) \rangle] + \frac{c_\alpha}{c_\gamma}\sum\limits_{k=\tau_t}^{t}E[y_k\langle W(v_k,\theta_k) , y_kE_{\theta_k}[\nabla_{\theta} \log \pi_{\theta_{k}}(s_k|a_k)] \rangle]\\
    &\leq U_w\frac{c_\alpha}{c_\gamma}(\sum\limits_{k=\tau_t}^{t}\mathbb{E}y_t^2)^{\frac{1}{2}}(\sum\limits_{k=\tau_t}^{t}\mathbb{E}\Vert M(\theta_k,v_k)\Vert^2)^\frac{1}{2} + \frac{c_\alpha}{c_\gamma}\sum\limits_{k=\tau_t}^{t}E[y_k^2\langle W(v_k,\theta_k) , E_{\theta_k}[\nabla_{\theta} \log \pi_{\theta_{k}}(s_k|a_k)] \rangle]\\
    &\leq U_w\frac{c_\alpha}{c_\gamma}(\sum\limits_{k=\tau_t}^{t}\mathbb{E}y_t^2)^{\frac{1}{2}}(\sum\limits_{k=\tau_t}^{t}\mathbb{E}\Vert M(\theta_k,v_k)\Vert^2)^\frac{1}{2} + \frac{c_\alpha}{c_\gamma}U_w B\sum\limits_{k=\tau_t}^{t}E[y_k^2].
\end{align*}

Hence collecting all the terms, we have,
\begin{align*}
    I_3= &\ \frac{2L_{J'}U_rG^2c_\alpha^2}{c_\gamma}(1+t-\tau_t)^{1-\nu}+ G\frac{c_\alpha}{c_\gamma}(\sum\limits_{k=\tau_t}^{t}\mathbb{E}y_t^2)^{\frac{1}{2}}(\sum\limits_{k=\tau_t}^{t}\mathbb{E}\Vert M(\theta_k,v_k)\Vert^2)^\frac{1}{2} +  \mathcal{O}(\log^2 t \cdot t^{1-\nu})\\
    &\qquad + U_w\frac{c_\alpha}{c_\gamma}(\sum\limits_{k=\tau_t}^{t}\mathbb{E}y_t^2)^{\frac{1}{2}}(\sum\limits_{k=\tau_t}^{t}\mathbb{E}\Vert M(\theta_k,v_k)\Vert^2)^\frac{1}{2} + \frac{c_\alpha}{c_\gamma}U_w B\sum\limits_{k=\tau_t}^{t}E[y_k^2]
\end{align*}
where $G = 2B(U_r + U_v)$.

For term $I_4$, we have
\begin{align*}
    I_4=& \sum\limits_{k=\tau_t}^{t}\frac{1}{\alpha_k}\mathbb{E}[(L(\theta_k)-L(\theta_{k+1}))^2]\\
    \leq &\ \sum\limits_{k=\tau_t}^{t}\frac{1}{\alpha_k}L^2_J\mathbb{E}\Vert \theta_k-\theta_{k+1}\Vert^2\\
    \leq &\ \sum\limits_{k=\tau_t}^{t}\frac{1}{\alpha_k}L_J^2G^2\alpha_k^2\\
    = &\ L_J^2G^2\sum\limits_{k=\tau_t}^{t}\alpha_k\\
    \leq &\ L_J^2G^2(1+t)^{1-\nu}.
\end{align*}

For term $I_5$, we have
\begin{align*}
    I_5=& \sum\limits_{k=\tau_t}^{t}\alpha_k\mathbb{E}[(r_k-L(\theta_k))^2]\\
    \leq &\  \sum\limits_{k=\tau_t}^{t} 4U_r^2\alpha_k\\
    \leq &\  4 U_r^2(1+t)^{1-\nu}.
\end{align*}

After combining all of the terms, we have,

\begin{align*}
    \sum\limits_{k=\tau_t}^{t} \mathbb{E}[y_k^2]\leq &\   \mathcal{O}(\log^2 t \cdot t^{1-\nu}) + \mathcal{O}(t^\nu) + (G + U_{w})\frac{c_\alpha}{c_\gamma}(\sum\limits_{k=\tau_t}^{t}\mathbb{E}[y_k^2])^{\frac{1}{2}}(\sum\limits_{k=\tau_t}^{t}\mathbb{E}\Vert M(\theta_k,v_k)\Vert^2)^\frac{1}{2}\\
    &\qquad+ \frac{c_\alpha}{c_\gamma}U_w B\sum\limits_{k=\tau_t}^{t}E[y_k^2].
\end{align*}

After rearranging terms above, we obtain,

\begin{align*}
    \bigg( 1 - \frac{c_\alpha}{c_\gamma}U_w B\bigg)\sum\limits_{k=\tau_t}^{t} \mathbb{E}[y_k^2]\leq &\   \mathcal{O}(\log^2 t \cdot t^{1-\nu}) + \mathcal{O}(t^\nu)\\
    &\qquad+ (G + U_{w})\frac{c_\alpha}{c_\gamma}(\sum\limits_{k=\tau_t}^{t}\mathbb{E}[y_k^2])^{\frac{1}{2}}(\sum\limits_{k=\tau_t}^{t}\mathbb{E}\Vert M(\theta_k,v_k)\Vert^2)^\frac{1}{2}.
\end{align*}

Now we require the condition ${\displaystyle\bigg(1 - \frac{c_\alpha}{c_\gamma}U_w B \bigg)> 0}$ to be satisfied for the left hand side term to be positive. Hence, we need to choose the values of $c_\alpha$ and $c_\gamma$ such that ${\displaystyle\bigg(1 - \frac{c_\alpha}{c_\gamma}U_w B \bigg)> 0}$.
We thus have,

\begin{align*}
   \sum\limits_{k=\tau_t}^{t} \mathbb{E}[y_k^2]\leq &\   \mathcal{O}(\log^2 t \cdot t^{1-\nu}) + \mathcal{O}(t^\nu) + \frac{(G + U_{w})}{ \bigg( 1 - \frac{c_\alpha}{c_\gamma}U_w B\bigg)}\frac{c_\alpha}{c_\gamma}(\sum\limits_{k=\tau_t}^{t}\mathbb{E}[y_k^2])^{\frac{1}{2}}(\sum\limits_{k=\tau_t}^{t}\mathbb{E}\Vert M(\theta_k,v_k)\Vert^2)^\frac{1}{2}.
\end{align*}

After applying the squaring technique (see page 23 of \citep{wu2022finite}), we have,

\begin{align}\label{ineq1}
    \sum\limits_{k=\tau_t}^{t} \mathbb{E}[y_k^2] \leq \mathcal{O}(\log^2 t \cdot t^{1-\nu}) + \mathcal{O}(t^\nu) + 2\frac{(G + U_w)^2}{(1 - \frac{c_\alpha}{c_\gamma}U_w B)^2}\frac{c_\alpha^2}{c_\gamma^2}\sum\limits_{k=\tau_t}^{t}\mathbb{E}\Vert M(\theta_k,v_k)\Vert^2.
\end{align}

\subsubsection{Convergence of the actor}\label{actor_convergence_proof}

Notations used here:

\begin{align}
    \begin{split}
        O_t :&= (s_t , a_t , s_{t+1})\\
        h(O_t,\theta_t,L_t,v_t) :&= ( r(s_t,a_t)- L_t  + \phi(s_{t+1})^{\top} v_{t} - \phi(s_t)^{\top} v_{t})\nabla \log\pi_{\theta_t}(a_t|s_t)\\
        I(O_t,L_t,\theta_t,v_t) :&= \langle \nabla L(\theta_t) ,h(O_t,\theta_t,L_t,v_t) - E_{s_t \sim \mu_{\theta_t},a_t \sim \pi_{\theta_t},s_{t+1} \sim p}[h(O_t,\theta_t,L_t,v_t)] \rangle  \\
        \bar{h}(O_t,\theta_t,v_t) :&= ( r(s_t,a_t)- L(\theta_t)  + \phi(s_{t+1})^{\top} v_{t} - \phi(s_t)^{\top} v_{t})\nabla \log\pi_{\theta_t}(a_t|s_t)\\
        M(\theta_t,v_t) :& = E_{s_t \sim \mu_{\theta_t},a_t \sim \pi_{\theta_t},s_{t+1} \sim p}[\bar{h}(O_t,\theta_t,v_t)]\\
        \bar{W}(O_t,\theta_t,v_t) :&= (V^{\theta_t}(s_{t+1}) -  \phi(s_{t+1})^Tv_t - V^{\theta_t}(s_t) + \phi(s_{t})^Tv_t)\nabla \log\pi_{\theta_t}(a_t|s_t)\\
        \Xi(O_t,\theta_t,v_t) :&= \langle  E_{\theta_t}[\bar{W}(O_t,\theta_t,v_t)] , E_{\theta_t}[\bar{h}(O_t,\theta_t,v_t)] \rangle - \langle \bar{W}(O_t,\theta_t,v_t) , E_{\theta_t}[\bar{h}(O_t,\theta_t,v_t)] \rangle .
    \end{split}
\end{align}
The following supporting lemmas will help in the proof.
\begin{lemma}\label{Gamma_actor}
Under assumptions \ref{assum:bounded_feature_norm}, \ref{assum:ergodicity}, \ref{assum:policy-lipschitz-bounded}, \ref{V_lipschitz},for any $t \geq \tau >  0$ , we have
   \begin{align*}
       E[I(O_t,L_t,\theta_t,v_t)] &\geq-(D_1(\tau + 1)\sum_{k=t-\tau + 1}^{t} E\Vert \theta_k - \theta_{k-1} \Vert + D_2bk^{\tau - 1} + D_3E\Vert v_t - v_{t-\tau} \Vert\\
       &\qquad+ D_4E\vert L_t - L_{t-\tau} \vert ),
   \end{align*}
for some $D_1 > 0 , D_2 > 0,D_3 > 0$ and $D_4 > 0$.
\end{lemma}

\begin{proof}
    We can decompose $ E[I(O_t,L_t,\theta_t,v_t)]$ as:
    \begin{align*}
        &E[I(O_t,L_t,\theta_t,v_t)] \\
        &= E[I(O_t,L_t,\theta_t,v_t) - I(O_t,L_t,\theta_t,v_{t-\tau})] + E[I(O_t,L_t,\theta_t,v_{t-\tau}) - I(O_t,L_t,\theta_{t-\tau},v_{t-\tau})]\\
        &\qquad + E[I(O_t,L_t,\theta_{t-\tau},v_{t-\tau}) - I(O_t,L_{t-\tau},\theta_{t-\tau},v_{t-\tau})] \\
        &\qquad + E[I(O_t,L_{t-\tau},\theta_{t-\tau},v_{t-\tau}) - I(\tilde{O_t},L_{t-\tau},\theta_{t-\tau},v_{t-\tau})]\\
        &\qquad + E[I(\tilde{O_t},L_{t-\tau},\theta_{t-\tau},v_{t-\tau}) - I(O^{'}_t,L_{t-\tau},\theta_{t-\tau},v_{t-\tau})]\\
        &\qquad + E[I(O^{'}_t,L_{t-\tau},\theta_{t-\tau},v_{t-\tau})].
    \end{align*}

In the above equality, $\tilde{O_t} = (\tilde{s}_t, \tilde{a}_t, \tilde{s}_{t+1})$ is from the auxiliary Markov chain defined in \cref{chain-au} and $O_t^{'} = (s^{'}_t , a^{'}_t ,s^{'}_{t+1} )$ where $ s^{'}_t \sim \mu_{\theta_{t-\tau}} , a^{'}_t \sim \pi_{\theta_{t-\tau}}$ and $s^{'}_{t+1} \sim P(.|s^{'}_t, a^{'}_t)$.    
The proof of Lemma \ref{Gamma_actor} is based on the results of lemmas \ref{gamma-actor1} - \ref{gamma-actor5} below that we now state and prove.
In the auxiliary lemmas below, we let 
$t \geq \tau > 0$.

\begin{sublemma}\label{gamma-actor1}
    \begin{align*}
        \vert E[I(O_t,L_t,\theta_t,v_t) - I(O_t,L_t,\theta_t,v_{t-\tau})]\vert \leq 4BG\Vert v_t - v_{t-\tau} \Vert.
    \end{align*}
\end{sublemma}

\begin{proof}
    \begin{align*}
        &\vert E[I(O_t,L_t,\theta_t,v_t) - I(O_t,L_t,\theta_t,v_{t-\tau})] \vert \\
        =& \vert E[\langle \nabla L(\theta_t) ,h(O_t,\theta_t,L_t,v_t) - h(O_t,\theta_t,L_t,v_{t-\tau})\\
        &\qquad- E_{s_t \sim \mu_{\theta_t},a_t \sim \pi_{\theta_t},s_{t+1} \sim p}[h(O_t,\theta_t,L_t,v_t)-h(O_t,\theta_t,L_t,v_{t-\tau})] \rangle ] \vert. \\
    \end{align*}
Further,
\begin{align*}
    \Vert h(O_t,\theta_t,L_t,v_t) - h(O_t,\theta_t,L_t,v_{t-\tau})\Vert &= \Vert ((\phi(s_{t+1})^{\top}  - \phi(s_t)^{\top}) (v_{t} - v_{t-\tau}))\nabla \log\pi_{\theta_t}(a_t|s_t) \Vert\\
    &\leq 2B\Vert v_t - v_{t-\tau}\Vert.
\end{align*}

Hence,
\begin{align*}
    \vert E[I(O_t,L_t,\theta_t,v_t) - I(O_t,L_t,\theta_t,v_{t-\tau})] \vert \leq 4BG\Vert v_t - v_{t-\tau} \Vert.
\end{align*}
The claim follows. 
\end{proof}

\begin{sublemma}\label{gamma-actor2}
    \begin{align*}
    \vert E[I(O_t,L_t,\theta_t,v_{t-\tau}) - I(O_t,L_t,\theta_{t-\tau},v_{t-\tau})] \vert \leq \overline{C}\Vert \theta_t - \theta_{t-\tau} \Vert 
\end{align*}
for some $\overline{C} > 0$.
\end{sublemma}

\begin{proof}
    \begin{align*}
        &\vert E[I(O_t,L_t,\theta_t,v_{t-\tau}) - I(O_t,L_t,\theta_{t-\tau},v_{t-\tau})] \vert\\
        =&\vert E[ \langle \nabla L(\theta_t) ,h(O_t,\theta_t,L_t,v_{t-\tau}) - E_{\theta_t}[h(O_t,\theta_t,L_t,v_{t-\tau})] \rangle \\
        &\qquad - \langle \nabla L(\theta_{t-\tau}) ,h(O_t,\theta_{t-\tau},L_t,v_{t-\tau}) - E_{\theta_{t-\tau}}[h(O_t,\theta_{t-\tau},L_t,v_{t-\tau})] \rangle ]\vert\\
        \leq &\vert E[\langle \nabla L(\theta_t) ,h(O_t,\theta_t,L_t,v_{t-\tau}) - E_{\theta_t}[h(O_t,\theta_t,L_t,v_{t-\tau})] \rangle \\
        & \qquad - \langle \nabla L(\theta_{t-\tau}) ,h(O_t,\theta_t,L_t,v_{t-\tau}) - E_{\theta_t}[h(O_t,\theta_t,L_t,v_{t-\tau})] \rangle] \vert\\
        & + \vert E[\langle \nabla L(\theta_{t-\tau}) ,h(O_t,\theta_t,L_t,v_{t-\tau}) - E_{\theta_t}[h(O_t,\theta_t,L_t,v_{t-\tau})] \rangle\\
        &\qquad - \langle \nabla L(\theta_{t-\tau}) ,h(O_t,\theta_{t-\tau},L_t,v_{t-\tau}) - E_{\theta_{t-\tau}}[h(O_t,\theta_{t-\tau},L_t,v_{t-\tau})] \rangle ]\vert \\
        =& \vert E[\langle \nabla L(\theta_t) - \nabla L(\theta_{t-\tau}) ,h(O_t,\theta_t,L_t,v_{t-\tau}) - E_{\theta_t}[h(O_t,\theta_t,L_t,v_{t-\tau})] \rangle]  \vert \\
        &\qquad + \vert E[\langle \nabla L(\theta_{t-\tau}) ,h(O_t,\theta_t,L_t,v_{t-\tau})- h(O_t,\theta_{t-\tau},L_t,v_{t-\tau}) \\
        &\qquad - E_{\theta_t}[h(O_t,\theta_t,L_t,v_{t-\tau})] + E_{\theta_{t-\tau}}[h(O_t,\theta_{t-\tau},L_t,v_{t-\tau})] \rangle] \vert\\
        \leq& 2L_{J}G\Vert \theta_t - \theta_{t-\tau} \Vert \\
        &+ \vert E[\Vert \nabla L(\theta_{t-\tau})\Vert (\underbrace{\Vert h(O_t,\theta_t,L_t,v_{t-\tau})- h(O_t,\theta_{t-\tau},L_t,v_{t-\tau})\Vert}_{I_a} \\
        &\qquad + \underbrace{\Vert E_{\theta_t}[h(O_t,\theta_t,L_t,v_{t-\tau})] - E_{\theta_{t-\tau}}[h(O_t,\theta_{t-\tau},L_t,v_{t-\tau})]\Vert}_{I_b}) ] \vert,      
    \end{align*}
where $E_{\theta_t}[.] $ denotes the expectation with respect to $s_t \sim \mu_{\theta_t},a_t \sim \pi_{\theta_t},s_{t+1} \sim p$.

Now, for the  term $I_a$, note that 

\begin{align*}
    &\Vert h(O_t,\theta_t,L_t,v_{t-\tau})- h(O_t,\theta_{t-\tau},L_t,v_{t-\tau})\Vert \\
    &= \Vert ( r(s_t,a_t)- L_t  + \phi(s_{t+1})^{\top} v_{t-\tau} - \phi(s_{t})^{\top} v_{t-\tau})(\nabla \log\pi_{\theta_t}(a_t|s_t) - \nabla \log\pi_{\theta_{t-\tau}}(a_t|s_t) ) \Vert\\
    &\leq  G K\Vert \theta_t - \theta_{t-\tau} \Vert.
\end{align*}

For the term $I_b$, we have ,
\begin{align}
    &\Vert E_{\theta_t}[h(O_t,\theta_t,L_t,v_{t-\tau})] - E_{\theta_{t-\tau}}[h(O_t,\theta_{t-\tau},L_t,v_{t-\tau})]\Vert \notag\\
    &\leq \Vert E_{\theta_t}[h(O_t,\theta_t,L_t,v_{t-\tau})] - E_{\theta_{t-\tau}}[h(O_t,\theta_t,L_t,v_{t-\tau})]\Vert \notag\\
    &\qquad + \Vert E_{\theta_{t-\tau}}[h(O_t,\theta_t,L_t,v_{t-\tau})] - E_{\theta_{t-\tau}}[h(O_t,\theta_{t-\tau},L_t,v_{t-\tau})]\Vert \notag\\
    &\leq 2Gd_{TV}(\mu_{\theta_t} \otimes \pi_{\theta_t} ,\mu_{\theta_{t-\tau}} \otimes \pi_{\theta_{t-\tau}} ) + GK\Vert \theta_t - \theta_{t-\tau} \Vert \notag\\
    &\leq 2G\vert A \vert L \bigg( 1 + \lceil \log_{k}b^{-1} \rceil + 1/(1-k) \bigg )\Vert \theta_{t} - \theta_{t-\tau} \Vert + GK\Vert \theta_t - \theta_{t-\tau} \Vert \label{ineq_lemmaB.1}\\
    &= C_{1}\Vert \theta_t - \theta_{t-\tau} \Vert\notag
\end{align}
where $C_1 = 2G\vert A \vert L \bigg( 1 + \lceil \log_{k}b^{-1} \rceil + 1/(1-k) \bigg ) + GK $. The 
inequality in \cref{ineq_lemmaB.1} follows from Lemma B.1 of \citep{wu2022finite}. 
Hence, after putting the results back, we obtain,
\begin{align*}
    \vert E[I(O_t,L_t,\theta_t,v_{t-\tau}) - I(O_t,L_t,\theta_{t-\tau},v_{t-\tau})] \vert \leq \overline{C}\Vert \theta_t - \theta_{t-\tau} \Vert 
\end{align*}
for some $\overline{C} > 0$. The claim follows.
\end{proof}

\begin{sublemma}\label{gamma-actor3}
    \begin{align*}
        \vert E[I(O_t,L_t,\theta_{t-\tau},v_{t-\tau}) - I(O_t,L_{t-\tau},\theta_{t-\tau},v_{t-\tau})] \vert \leq 2BG\vert L_t - L_{t-\tau}\vert.
    \end{align*}
\end{sublemma}

\begin{proof}
    \begin{align*}
        &\vert E[I(O_t,L_t,\theta_{t-\tau},v_{t-\tau}) - I(O_t,L_{t-\tau},\theta_{t-\tau},v_{t-\tau})] \vert\\
        &= \vert E[\langle \nabla L(\theta_{t-\tau}) ,h(O_t,\theta_{t-\tau},L_t,v_{t-\tau}) - h(O_t,\theta_{t-\tau},L_{t-\tau},v_{t-\tau}) \\
        &\qquad- E_{\theta_{t-\tau}}[h(O_t,\theta_{t-\tau},L_t,v_{t-\tau})] \rangle] + E_{\theta_{t-\tau}}[h(O_t,\theta_{t-\tau},L_{t-\tau},v_{t-\tau})] \rangle] \vert\\
        &\leq 2BG\vert L_t - L_{t-\tau}\vert.
    \end{align*}
\end{proof}

\begin{sublemma}\label{gamma-actor4}
    \begin{align*}
        \vert E[I(O_t,L_{t-\tau},\theta_{t-\tau},v_{t-\tau}) - I(\tilde{O_t},L_{t-\tau},\theta_{t-\tau},v_{t-\tau})]\vert \leq \check{K}\sum\limits_{i=t-\tau}^{t}E\Vert \theta_{i} - \theta_{t-\tau} \Vert. 
    \end{align*}
for some  $\check{K} > 0$.
\end{sublemma}

\begin{proof}
    The proof is as in lemma D.2 in \citep{wu2022finite}.
\end{proof}

\begin{sublemma}\label{gamma-actor5}
    \begin{align*}
        \vert E[I(\tilde{O_t},L_{t-\tau},\theta_{t-\tau},v_{t-\tau}) - I(O^{'}_t,L_{t-\tau},\theta_{t-\tau},v_{t-\tau})] \vert \leq \overline{K}bk^{\tau-1}.
    \end{align*}
for some $\overline{K} > 0$.
\end{sublemma}

\begin{proof}
    The proof is as in lemma D.3 in \citep{wu2022finite}.
\end{proof}

Now collecting the results of lemmas \ref{gamma-actor1} - \ref{gamma-actor5}, we have,

\begin{align*}
   & E[I(O_t,L_t,\theta_t,v_t)] \\
   &\geq -4BG\Vert v_t - v_{t-\tau} \Vert - \overline{C}\Vert \theta_t - \theta_{t-\tau} \Vert - 2BG\vert L_t - L_{t-\tau}\vert - \check{K}\sum\limits_{i=t-\tau}^{t}E\Vert \theta_{i} - \theta_{t-\tau} \Vert - \overline{K}bk^{\tau-1}\\
    &\geq -4BG\Vert v_t - v_{t-\tau} \Vert - \overline{C}\Vert \theta_t - \theta_{t-\tau} \Vert - 2BG\vert L_t - L_{t-\tau}\vert - \check{K}(\tau + 1)E\Vert \theta_{t} - \theta_{t-\tau} \Vert - \overline{K}bk^{\tau-1}\\
    &\geq -4BG\Vert v_t - v_{t-\tau} \Vert - \overline{C}\Vert \theta_t - \theta_{t-\tau} \Vert - 2BG\vert L_t - L_{t-\tau}\vert\\
    &\qquad - \check{K}(\tau + 1)\sum\limits_{t-\tau + 1}^{t}E\Vert \theta_{k} - \theta_{k-\tau} \Vert - \overline{K}bk^{\tau-1}.
\end{align*}
The claim of Lemma~\ref{Gamma_actor} now follows.
\end{proof}

\begin{lemma}\label{Xi_actor}
Under assumptions \ref{assum:bounded_feature_norm}, \ref{assum:ergodicity}, \ref{assum:policy-lipschitz-bounded}, \ref{V_lipschitz}, for any $t \geq \tau >  0$ , we have
    \begin{align*}
        E[\Xi (O_t,\theta_t,v_t)] &\geq -4B(G + B(\bar{U}_{v} + U_v))\Vert v_t - v_{t-\tau}  \Vert - D\Vert \theta_t - \theta_{t-\tau}\Vert \\
        &\qquad- B_1\sum\limits_{i=t-\tau}^{t}E\Vert \theta_{i} - \theta_{t-\tau} \Vert - B_2bk^{\tau-1},
\end{align*}
where $D > 0 , B_1 > 0 $ and $B_2 > 0$ are constants.
\end{lemma}

\begin{proof}
    We can decompose $E[\Xi(O_t,\theta_t,v_t)]$ as :\\
    \begin{align*}
        E[\Xi (O_t,\theta_t,v_t)] &= \underbrace{E[\Xi(O_t,\theta_t,v_t)- \Xi(O_t,\theta_t,v_{t-\tau})]}_{I_a} + \underbrace{E[\Xi(O_t,\theta_t,v_{t-\tau}) - \Xi(O_t,\theta_{t-\tau},v_{t-\tau})]}_{I_b}\\
        &\qquad + \underbrace{E[\Xi(O_t,\theta_{t-\tau},v_{t-\tau}) - \Xi(\tilde{O_t},\theta_{t-\tau},v_{t-\tau})]}_{I_c} + \underbrace{E[\Xi(\tilde{O_t},\theta_{t-\tau},v_{t-\tau})]}_{I_d}.
    \end{align*}

In the above equality, $\tilde{O_t} = (\tilde{s}_t , \tilde{a}_t , \tilde{s}_{t+1})$ is from the auxiliary Markov chain defined in \cref{chain-au} and $O_t^{'} = (s^{'}_t , a^{'}_t ,s^{'}_{t+1} )$, where $ s^{'}_t \sim \mu_{\theta_{t-\tau}} , a^{'}_t \sim \pi_{\theta_{t-\tau}}$ and $s^{'}_{t+1} \sim P(.|s^{'}_t, a^{'}_t)$.     

For the term $I_a$ above, we have,
\begin{align*}
    &\vert \Xi(O_t,\theta_t,v_t)- \Xi(O_t,\theta_t,v_{t-\tau}) \vert\\
    &= \vert \langle  E_{\theta_t}[\bar{W}(O_t,\theta_t,v_t)] , E_{\theta_t}[\bar{h}(O_t,\theta_t,v_t)] \rangle - \langle \bar{W}(O_t,\theta_t,v_t) , E_{\theta_t}[\bar{h}(O_t,\theta_t,v_t)] \rangle \\
    &\qquad - \langle  E_{\theta_t}[\bar{W}(O_t,\theta_t,v_{t-\tau})] , E_{\theta_t}[\bar{h}(O_t,\theta_t,v_{t-\tau})] \rangle + \langle \bar{W}(O_t,\theta_t,v_{t-\tau}) , E_{\theta_t}[\bar{h}(O_t,\theta_t,v_{t-\tau})] \rangle \vert\\
    &= \vert \langle  E_{\theta_t}[\bar{W}(O_t,\theta_t,v_t)] - \bar{W}(O_t,\theta_t,v_t) , E_{\theta_t}[\bar{h}(O_t,\theta_t,v_t)] \rangle\\
    &\qquad  - \langle  E_{\theta_t}[\bar{W}(O_t,\theta_t,v_{t-\tau})] - \bar{W}(O_t,\theta_t,v_{t-\tau}), E_{\theta_t}[\bar{h}(O_t,\theta_t,v_{t-\tau})] \rangle  \vert\\
    &\leq \vert \langle  E_{\theta_t}[\bar{W}(O_t,\theta_t,v_t)] - \bar{W}(O_t,\theta_t,v_t) , E_{\theta_t}[\bar{h}(O_t,\theta_t,v_t)] \rangle\\
    &\qquad - \langle  E_{\theta_t}[\bar{W}(O_t,\theta_t,v_{t-\tau})] - \bar{W}(O_t,\theta_t,v_{t-\tau}) , E_{\theta_t}[\bar{h}(O_t,\theta_t,v_t)] \rangle \vert\\
    &\qquad + \vert \langle  E_{\theta_t}[\bar{W}(O_t,\theta_t,v_{t-\tau})] - \bar{W}(O_t,\theta_t,v_{t-\tau}) , E_{\theta_t}[\bar{h}(O_t,\theta_t,v_t)] \rangle\\
    &\qquad - \langle  E_{\theta_t}[\bar{W}(O_t,\theta_t,v_{t-\tau})] - \bar{W}(O_t,\theta_t,v_{t-\tau}), E_{\theta_t}[\bar{h}(O_t,\theta_t,v_{t-\tau})] \rangle  \vert\\
    &= \vert \langle  E_{\theta_t}[\bar{W}(O_t,\theta_t,v_t)] - E_{\theta_t}[\bar{W}(O_t,\theta_t,v_{t-\tau})] - \bar{W}(O_t,\theta_t,v_t) \\
    &\qquad+ \bar{W}(O_t,\theta_t,v_{t-\tau}) , E_{\theta_t}[\bar{h}(O_t,\theta_t,v_t)] \rangle \vert \\
    &\qquad + \vert  \langle  E_{\theta_t}[\bar{W}(O_t,\theta_t,v_{t-\tau})] - \bar{W}(O_t,\theta_t,v_{t-\tau}) , E_{\theta_t}[\bar{h}(O_t,\theta_t,v_t)] -  E_{\theta_t}[\bar{h}(O_t,\theta_t,v_{t-\tau})]\rangle \vert \\
    &\leq  4BG\Vert v_t - v_{t-\tau} \Vert + 4B^{2}(\bar{U}_{v} + U_v)\Vert v_t - v_{t-\tau}  \Vert\\
    &= 4B(G + B(\bar{U}_{v} + U_v))\Vert v_t - v_{t-\tau}  \Vert.
\end{align*}

For term $I_b$, we have,
\begin{align*}
    &\vert \Xi(O_t,\theta_t,v_{t-\tau}) - \Xi(O_t,\theta_{t-\tau},v_{t-\tau})\vert\\
    &=\vert \langle  E_{\theta_t}[\bar{W}(O_t,\theta_t,v_{t-\tau})] - \bar{W}(O_t,\theta_t,v_{t-\tau}) , E_{\theta_t}[\bar{h}(O_t,\theta_t,v_{t-\tau})] \rangle \\
    &\qquad- \langle  E_{\theta_{t-\tau}}[\bar{W}(O_t,\theta_{t-\tau},v_{t-\tau})] - \bar{W}(O_t,\theta_{t-\tau},v_{t-\tau}) , E_{\theta_{t-\tau}}[\bar{h}(O_t,\theta_{t-\tau},v_{t-\tau})] \rangle   \vert\\
    &=\vert \langle  E_{\theta_t}[\bar{W}(O_t,\theta_t,v_{t-\tau})] - \bar{W}(O_t,\theta_t,v_{t-\tau}) , E_{\theta_t}[\bar{h}(O_t,\theta_t,v_{t-\tau})] \rangle \\
    &\qquad - \langle  E_{\theta_t}[\bar{W}(O_t,\theta_t,v_{t-\tau})] - \bar{W}(O_t,\theta_t,v_{t-\tau}) , E_{\theta_{t-\tau}}[\bar{h}(O_t,\theta_{t-\tau},v_{t-\tau})] \rangle\\
    &\qquad +  \langle  E_{\theta_t}[\bar{W}(O_t,\theta_t,v_{t-\tau})] - \bar{W}(O_t,\theta_t,v_{t-\tau}) , E_{\theta_{t-\tau}}[\bar{h}(O_t,\theta_{t-\tau},v_{t-\tau})] \rangle\\
     &\qquad- \langle  E_{\theta_{t-\tau}}[\bar{W}(O_t,\theta_{t-\tau},v_{t-\tau})] - \bar{W}(O_t,\theta_{t-\tau},v_{t-\tau}) , E_{\theta_{t-\tau}}[\bar{h}(O_t,\theta_{t-\tau},v_{t-\tau})] \rangle   \vert\\
     &\leq \underbrace{\vert \langle  E_{\theta_t}[\bar{W}(O_t,\theta_t,v_{t-\tau})] - \bar{W}(O_t,\theta_t,v_{t-\tau}) , E_{\theta_t}[\bar{h}(O_t,\theta_t,v_{t-\tau})] - E_{\theta_{t-\tau}}[\bar{h}(O_t,\theta_{t-\tau},v_{t-\tau})] \rangle \vert}_{I_{b1}}\\
     &\qquad + \underbrace{\vert \langle  E_{\theta_t}[\bar{W}(O_t,\theta_t,v_{t-\tau})] -  E_{\theta_{t-\tau}}[\bar{W}(O_t,\theta_{t-\tau},v_{t-\tau})]}_{I_{b2}} \\
     &\qquad \underbrace{-\bar{W}(O_t,\theta_t,v_{t-\tau}) + \bar{W}(O_t,\theta_{t - \tau},v_{t-\tau}) , E_{\theta_{t-\tau}}[\bar{h}(O_t,\theta_{t-\tau},v_{t-\tau})] \rangle\vert}_{I_{b3}}.\\
\end{align*}

For term $I_{b1}$, we have ,
\begin{align*}
    &\vert \langle  E_{\theta_t}[\bar{W}(O_t,\theta_t,v_{t-\tau})] - \bar{W}(O_t,\theta_t,v_{t-\tau}) , E_{\theta_t}[\bar{h}(O_t,\theta_t,v_{t-\tau})] - E_{\theta_{t-\tau}}[\bar{h}(O_t,\theta_{t-\tau},v_{t-\tau})] \rangle \vert\\
    &\leq 4B(\overline{U_{v}} + U_{v})\Vert E_{\theta_t}[\bar{h}(O_t,\theta_t,v_{t-\tau})] - E_{\theta_{t-\tau}}[\bar{h}(O_t,\theta_{t-\tau},v_{t-\tau})] \Vert\\
    &\leq 4B(\overline{U_{v}} + U_{v})\Vert E_{\theta_t}[\bar{h}(O_t,\theta_t,v_{t-\tau})] - E_{\theta_{t-\tau}}[\bar{h}(O_t,\theta_t,v_{t-\tau})] \Vert\\
    &\qquad + 4B(\overline{U_{v}} + U_{v})\Vert E_{\theta_{t-\tau}}[\bar{h}(O_t,\theta_t,v_{t-\tau})] - E_{\theta_{t-\tau}}[\bar{h}(O_t,\theta_{t-\tau},v_{t-\tau})] \Vert\\
    &\leq 8GB(\overline{U_{v}} + U_{v})d_{TV}(\mu_{\theta_{t}} \otimes \pi_{\theta_{t}} ,\mu_{\theta_{t-\tau}} \otimes \pi_{\theta_{t-\tau}} )\\
    &\qquad + 4B(\overline{U_{v}} + U_{v}) E_{\theta_{t-\tau}}[\Vert\bar{h}(O_t,\theta_t,v_{t-\tau})] - E_{\theta_{t-\tau}}[\bar{h}(O_t,\theta_{t-\tau},v_{t-\tau})\Vert] \\
    &\leq 8GB(\overline{U_{v}} + U_{v})\vert A \vert L\bigg( 1 + \lceil \log_{k} b^{-1} \rceil + \frac{1}{1 - k} \bigg)\Vert \theta_{t} - \theta_{t-\tau} \Vert\\
    &\qquad + 8B(\overline{U_{v}} + U_{v})(U_{r} + U_{v})K\Vert \theta_{t} - \theta_{t-\tau} \Vert\\
    &= D_{1}\Vert \theta_{t} - \theta_{t-\tau} \Vert,
\end{align*}
where $D_{1} = 8B(\overline{U_{v}} + U_{v})\bigg(  G\vert A \vert L\bigg( 1 + \lceil \log_{k} b^{-1} \rceil + \frac{1}{1 - k} \bigg) +(U_{r} + U_{v})K \bigg)$.\\
The last inequality above follows from Lemma B.1 in \citep{wu2022finite}.

Next, for term $I_{b2}$ + term $I_{b3}$, we have,

\begin{align*}
    &\vert \langle  E_{\theta_t}[\bar{W}(O_t,\theta_t,v_{t-\tau})] -  E_{\theta_{t-\tau}}[\bar{W}(O_t,\theta_{t-\tau},v_{t-\tau})] - \bar{W}(O_t,\theta_t,v_{t-\tau})\\
    &\qquad+ \bar{W}(O_t,\theta_{t - \tau},v_{t-\tau}) , E_{\theta_{t-\tau}}[\bar{h}(O_t,\theta_{t-\tau},v_{t-\tau})] \rangle\vert\\
    &\leq G(\Vert E_{\theta_t}[\bar{W}(O_t,\theta_t,v_{t-\tau})] -  E_{\theta_{t-\tau}}[\bar{W}(O_t,\theta_{t-\tau},v_{t-\tau})]\Vert + \Vert \bar{W}(O_t,\theta_t,v_{t-\tau})\\
    &\qquad- \bar{W}(O_t,\theta_{t - \tau},v_{t-\tau})  \Vert)\\
    &\leq G(\Vert E_{\theta_t}[\bar{W}(O_t,\theta_t,v_{t-\tau})] -  E_{\theta_{t-\tau}}[\bar{W}(O_t,\theta_{t-\tau},v_{t-\tau})]\Vert + 2(\overline{U}_{v} + U_{v})K\Vert \theta_t - \theta_{t-\tau}\Vert)\\
    &\leq  G(\Vert E_{\theta_t}[\bar{W}(O_t,\theta_t,v_{t-\tau})] -  E_{\theta_{t-\tau}}[\bar{W}(O_t,\theta_{t},v_{t-\tau})]\Vert + \Vert E_{\theta_{t-\tau}}[\bar{W}(O_t,\theta_t,v_{t-\tau})]\\
    &\qquad-  E_{\theta_{t-\tau}}[\bar{W}(O_t,\theta_{t-\tau},v_{t-\tau})] \Vert)
    + 2G(\overline{U}_{v} + U_{v})K\Vert \theta_t - \theta_{t-\tau}\Vert\\
    &\leq D_1\Vert \theta_{t} - \theta_{t-\tau} \Vert,
\end{align*}

where $D_1 > 0$.

The last inequality above again follows from Lemma B.1 in \citep{wu2022finite}.\\
Hence after collecting the results of terms $I_{b1}$ and $I_{b2}$ we have,

\begin{align*}
    I_b \geq -D\Vert \theta_t - \theta_{t-\tau}\Vert,
\end{align*}
for some $D > 0$.

Now, for the term $I_c$, we have,
\begin{align*}
    \vert E[\Xi(O_t,\theta_{t-\tau},v_{t-\tau}) - \Xi(\tilde{O_t},\theta_{t-\tau},v_{t-\tau})]\vert \leq B_1\sum\limits_{i=t-\tau}^{t}E\Vert \theta_{i} - \theta_{t-\tau} \Vert,
\end{align*}
for some $B_1 > 0$.

For term $I_d$, we have,
\begin{align*}
    \vert E[\Xi(\tilde{O_t},\theta_{t-\tau},v_{t-\tau})]\vert \leq B_2bk^{\tau-1}.
\end{align*}
for some $B_2 > 0$.
\end{proof}

For analysis of terms $I_c$ and $I_d$, please see lemmas D.10 and D.11 in \citep{wu2022finite}.

Thus, after collecting all the terms, we have,
\begin{align*}
    E[\Xi (O_t,\theta_t,v_t)] &\geq -4B(G + B(\bar{U}_{v} + U_v))\Vert v_t - v_{t-\tau}  \Vert - D\Vert \theta_t - \theta_{t-\tau}\Vert \\
    &\qquad- B_1\sum\limits_{i=t-\tau}^{t}E\Vert \theta_{i} - \theta_{t-\tau} \Vert - B_2bk^{\tau-1},
\end{align*}
where $D > 0 , B_1 > 0 $ and $B_2 > 0$.

\subsubsection*{Proof of convergence of the actor.}

After applying Lemma \ref{L_smoothness} to the update rule of the actor, we have,

\begin{align*}
    L(\theta_{t+1}) &\geq L(\theta_t) + \alpha_t \langle \nabla L(\theta_t) ,\delta_{t}\nabla \log\pi_{\theta_t}(a_t|s_t) \rangle 
- M_{L}\alpha_t^2\Vert\delta_{t}\nabla \log\pi_{\theta_t}(a_t|s_t)\Vert^2.
\end{align*}

For the term $\langle \nabla L(\theta_t) ,\delta_{t}\nabla \log\pi_{\theta_t}(a_t|s_t) \rangle $, we have,
\begin{align*}
    &\langle \nabla L(\theta_t) ,\delta_{t}\nabla \log\pi_{\theta_t}(a_t|s_t) \rangle \\
    &= \langle \nabla L(\theta_t) ,( r(s_t,a_t) - L_t + \phi(s_{t+1})^{\top} v_{t} - \phi(s_t)^{\top} v_{t})\nabla \log\pi_{\theta_t}(a_t|s_t) \rangle\\
    & = I(O_t,\theta_t,L_t,v_t) + \langle \nabla L(\theta_t) , E_{s_t \sim \mu_{\theta_t},a_t \sim \pi_{\theta_t},s_{t+1} \sim p}[h(O_t,\theta_t,L_t,v_t)] \rangle.
\end{align*}

Hence,
\begin{align}\label{inequality_L(theta)}
    & L(\theta_{t+1})\\
    &\geq L(\theta_t) + \alpha_t I(O_t,\theta_t,L_t,v_t) + \alpha_t\langle \nabla L(\theta_t) , M(\theta_t,v_t) \rangle\notag\\
&\qquad + \alpha_t\langle \nabla L(\theta_t) , E_{\theta_t}[(L(\theta_t) - L_t)\nabla \log\pi_{\theta_t}(a_t|s_t)] \rangle -M_{L}\alpha_t^2\Vert\delta_{t}\nabla \log\pi_{\theta_t}(a_t|s_t)\Vert^2\notag\\
& = L(\theta_t) + \alpha_t I(O_t,\theta_t,L_t,v_t) + \alpha_t\Vert M(\theta_t,v_t)\Vert^2\notag \\
&\qquad+\alpha_t\langle  E_{\theta_t}[(V^{\theta_t}(s_{t+1}) -  \phi(s_{t+1})^Tv_t - V^{\theta_t}(s_t) + \phi(s_{t})^Tv_t)\nabla \log\pi_{\theta_t}(a_t|s_t)] ,\\
&\qquad \qquad \qquad \qquad E_{\theta_t}[\bar{h}(O_t,\theta_t,v_t)] \rangle\notag\\
&\qquad - \alpha_t\langle (V^{\theta_t}(s_{t+1}) -  \phi(s_{t+1})^Tv_t - V^{\theta_t}(s_t) + \phi(s_{t})^Tv_t)\nabla \log\pi_{\theta_t}(a_t|s_t) , E_{\theta_t}[\bar{h}(O_t,\theta_t,v_t)] \rangle \notag\\
&\qquad + \underbrace{\alpha_t\langle (V^{\theta_t}(s_{t+1}) -  \phi(s_{t+1})^Tv_t - V^{\theta_t}(s_t) + \phi(s_{t})^Tv_t)\nabla \log\pi_{\theta_t}(a_t|s_t) , E_{\theta_t}[\bar{h}(O_t,\theta_t,v_t)] \rangle}_{I_1}\notag\\
&\qquad + \alpha_t\langle \nabla L(\theta_t) , E_{\theta_t}[(L(\theta_t) - L_t)\nabla \log\pi_{\theta_t}(a_t|s_t)] \rangle - M_{L}\alpha_t^2\Vert\delta_{t}\nabla \log\pi_{\theta_t}(a_t|s_t)\Vert^2.
\end{align}

Now,

\begin{align*}
    &\alpha_t \langle(V^{\theta_t}(s_{t+1}) - \phi(s_{t+1})^T v_t)\nabla \log\pi_{\theta_t}(a_t|s_t) ,  M(\theta_t,v_t)\rangle\\
    & = \alpha_t \langle(V^{\theta_t}(s_{t+1})-V^{\theta_{t+1}}(s_{t+1}) + V^{\theta_{t+1}}(s_{t+1}) - \phi(s_{t+1})^T v_t)\nabla \log\pi_{\theta_t}(a_t|s_t) ,  M(\theta_t,v_t)\rangle\\
    & = \alpha_t \langle(V^{\theta_t}(s_{t+1})-V^{\theta_{t+1}}(s_{t+1}))\nabla\log\pi_{\theta_t}(a_t|s_t) ,  M(\theta_t,v_t)\rangle\\
    &\qquad +  \alpha_t\langle (V^{\theta_{t+1}}(s_{t+1}) - \phi(s_{t+1})^T v_t)\nabla \log\pi_{\theta_t}(a_t|s_t) ,  M(\theta_t,v_t)\rangle\\
    & = \alpha_t \langle(V^{\theta_t}(s_{t+1})-V^{\theta_{t+1}}(s_{t+1}))\nabla\log\pi_{\theta_t}(a_t|s_t) ,  M(\theta_t,v_t)\rangle\\
    &\qquad +  \alpha_t\langle (V^{\theta_{t+1}}(s_{t+1}) - \phi(s_{t+1})^T v_{t+1} + \phi(s_{t+1})^T v_{t+1} - \phi(s_{t+1})^T v_{t})\nabla \log\pi_{\theta_t}(a_t|s_t)\\
    &\qquad,  M(\theta_t,v_t)\rangle\\
    & = \alpha_t \langle(V^{\theta_t}(s_{t+1})-V^{\theta_{t+1}}(s_{t+1}))\nabla\log\pi_{\theta_t}(a_t|s_t) ,  M(\theta_t,v_t)\rangle\\
    &\qquad + \alpha_t\langle ( \phi(s_{t+1})^T v_{t+1} - \phi(s_{t+1})^T v_{t})\nabla \log\pi_{\theta_t}(a_t|s_t) ,  M(\theta_t,v_t)\rangle\\
    &\qquad + \alpha_t \langle(V^{\theta_{t+1}}(s_{t+1}) - \phi(s_{t+1})^T v_{t+1})\nabla \log\pi_{\theta_t}(a_t|s_t) ,  M(\theta_t,v_t) \rangle \\
    & = \alpha_t \langle(V^{\theta_t}(s_{t+1})-V^{\theta_{t+1}}(s_{t+1}))\nabla\log\pi_{\theta_t}(a_t|s_t) ,  M(\theta_t,v_t)\rangle\\
    &\qquad + \alpha_t\langle ( \phi(s_{t+1})^T v_{t+1} - \phi(s_{t+1})^T v_{t})\nabla \log\pi_{\theta_t}(a_t|s_t) ,  M(\theta_t,v_t)\rangle\\
    &\qquad + \alpha_{t + 1} \langle(V^{\theta_{t+1}}(s_{t+1}) - \phi(s_{t+1})^T v_{t+1})\nabla \log\pi_{\theta_{t+1}}(a_{t+1}|s_{t+1}) ,  M(\theta_{t + 1},v_{t+1}) \rangle \\
    &\qquad + \alpha_t \langle(V^{\theta_{t+1}}(s_{t+1}) - \phi(s_{t+1})^T v_{t+1})\nabla \log\pi_{\theta_t}(a_t|s_t) ,  M(\theta_t,v_t) \rangle\\
    & \qquad - \alpha_{t + 1} \langle(V^{\theta_{t+1}}(s_{t+1}) - \phi(s_{t+1})^T v_{t+1})\nabla \log\pi_{\theta_{t+1}}(a_{t+1}|s_{t+1}) ,  M(\theta_{t + 1},v_{t+1}) \rangle. 
\end{align*}

Hence for the term $I_1$, we have,
\begin{align*}
    I_1 &= \alpha_t \langle(V^{\theta_t}(s_{t+1})-V^{\theta_{t+1}}(s_{t+1}))\nabla\log\pi_{\theta_t}(a_t|s_t) ,  M(\theta_t,v_t)\rangle\\
    &\qquad + \alpha_t\langle ( \phi(s_{t+1})^T v_{t+1} - \phi(s_{t+1})^T v_{t})\nabla \log\pi_{\theta_t}(a_t|s_t) ,  M(\theta_t,v_t)\rangle\\
    &\qquad + \alpha_{t + 1} \langle(V^{\theta_{t+1}}(s_{t+1}) - \phi(s_{t+1})^T v_{t+1})\nabla \log\pi_{\theta_{t+1}}(a_{t+1}|s_{t+1}) ,  M(\theta_{t + 1},v_{t+1}) \rangle \\
    &\qquad + \alpha_t \langle(V^{\theta_{t+1}}(s_{t+1}) - \phi(s_{t+1})^T v_{t+1})\nabla \log\pi_{\theta_t}(a_t|s_t) ,  M(\theta_t,v_t) \rangle\\
    & \qquad - \alpha_{t + 1} \langle(V^{\theta_{t+1}}(s_{t+1}) - \phi(s_{t+1})^T v_{t+1})\nabla \log\pi_{\theta_{t+1}}(a_{t+1}|s_{t+1}) ,  M(\theta_{t + 1},v_{t+1}) \rangle\\
    & \qquad +  \alpha_t \langle (- V^{\theta_t}(s_t) + \phi(s_{t})^Tv_t)\nabla \log\pi_{\theta_t}(a_t|s_t) , M(\theta_t,v_t)\rangle.
\end{align*}

Putting this back in \cref{inequality_L(theta)}, we obtain,

\begin{align*}
     L(\theta_{t+1}) & \geq L(\theta_t) + \alpha_t I(O_t,\theta_t,L_t,v_t) + \alpha_t\Vert M(\theta_t,v_t)\Vert^2 + \alpha_t\Xi(O_t,\theta_t,v_t)\\
&\qquad+\alpha_t \langle(V^{\theta_t}(s_{t+1})-V^{\theta_{t+1}}(s_{t+1}))\nabla\log\pi_{\theta_t}(a_t|s_t) ,  M(\theta_t,v_t)\rangle\\
    &\qquad + \alpha_t\langle ( \phi(s_{t+1})^T v_{t+1} - \phi(s_{t+1})^T v_{t})\nabla \log\pi_{\theta_t}(a_t|s_t) ,  M(\theta_t,v_t)\rangle\\
    &\qquad + \alpha_{t + 1} \langle(V^{\theta_{t+1}}(s_{t+1}) - \phi(s_{t+1})^T v_{t+1})\nabla \log\pi_{\theta_{t+1}}(a_{t+1}|s_{t+1}) ,  M(\theta_{t + 1},v_{t+1}) \rangle \\
    &\qquad + \alpha_t \langle(V^{\theta_{t+1}}(s_{t+1}) - \phi(s_{t+1})^T v_{t+1})\nabla \log\pi_{\theta_t}(a_t|s_t) ,  M(\theta_t,v_t) \rangle\\
    & \qquad - \alpha_{t + 1} \langle(V^{\theta_{t+1}}(s_{t+1}) - \phi(s_{t+1})^T v_{t+1})\nabla \log\pi_{\theta_{t+1}}(a_{t+1}|s_{t+1}) ,  M(\theta_{t + 1},v_{t+1}) \rangle\\
    & \qquad +  \alpha_t \langle (- V^{\theta_t}(s_t) + \phi(s_{t})^Tv_t)\nabla \log\pi_{\theta_t}(a_t|s_t) , M(\theta_t,v_t)\rangle\\
&\qquad + \alpha_t\langle \nabla L(\theta_t) , E_{\theta_t}[(L(\theta_t) - L_t)\nabla \log\pi_{\theta_t}(a_t|s_t)] \rangle - M_{L}\alpha_t^2\Vert\delta_{t}\nabla \log\pi_{\theta_t}(a_t|s_t)\Vert^2.
\end{align*}
\begin{align*}
    &\Rightarrow \Vert M(\theta_t,v_t)\Vert^2\\ & \leq \frac{L(\theta_{t+1}) - L(\theta_t)}{\alpha_t} - I(O_t,\theta_t,L_t,v_t) - \Xi(O_t,\theta_t,v_t)\\
    &\qquad-  \langle(V^{\theta_t}(s_{t+1})-V^{\theta_{t+1}}(s_{t+1}))\nabla\log\pi_{\theta_t}(a_t|s_t) ,  M(\theta_t,v_t)\rangle\\
    &\qquad - \langle ( \phi(s_{t+1})^T v_{t+1} - \phi(s_{t+1})^T v_{t})\nabla \log\pi_{\theta_t}(a_t|s_t) ,  M(\theta_t,v_t)\rangle\\
    &\qquad - \frac{1}{\alpha_t}\alpha_{t + 1} \langle(V^{\theta_{t+1}}(s_{t+1}) - \phi(s_{t+1})^T v_{t+1})\nabla \log\pi_{\theta_{t+1}}(a_{t+1}|s_{t+1}) ,  M(\theta_{t + 1},v_{t+1}) \rangle \\
    &\qquad -\frac{1}{\alpha_t}( \alpha_t \langle(V^{\theta_{t+1}}(s_{t+1}) - \phi(s_{t+1})^T v_{t+1})\nabla \log\pi_{\theta_t}(a_t|s_t) ,  M(\theta_t,v_t) \rangle)\\
    & \qquad +\frac{1}{\alpha_t} \alpha_{t + 1} \langle(V^{\theta_{t+1}}(s_{t+1}) - \phi(s_{t+1})^T v_{t+1})\nabla \log\pi_{\theta_{t+1}}(a_{t+1}|s_{t+1}) ,  M(\theta_{t + 1},v_{t+1}) \rangle\\
    & \qquad - \frac{1}{\alpha_t} (\alpha_t \langle (- V^{\theta_t}(s_t) + \phi(s_{t})^Tv_t)\nabla \log\pi_{\theta_t}(a_t|s_t) , M(\theta_t,v_t)\rangle)\\
    &\qquad - \langle \nabla L(\theta_t) , (L(\theta_t) - L_t)\nabla \log\pi_{\theta_t}(a_t|s_t) \rangle + M_{L}\alpha_t\Vert\delta_{t}\nabla \log\pi_{\theta_t}(a_t|s_t)\Vert^2\\\\
 & \leq (L(\theta_{t+1}) - L(\theta_t) + Q_{t} - Q_{t+1})/\alpha_t - I(O_t,\theta_t,L_t,v_t) -\Xi(O_t,\theta_t,v_t) \\
 &\qquad-  \langle(V^{\theta_t}(s_{t+1})-V^{\theta_{t+1}}(s_{t+1}))\nabla\log\pi_{\theta_t}(a_t|s_t) ,  M(\theta_t,v_t)\rangle\\
 &\qquad  - \langle ( \phi(s_{t+1})^T v_{t+1} - \phi(s_{t+1})^T v_{t})\nabla \log\pi_{\theta_t}(a_t|s_t) ,  M(\theta_t,v_t)\rangle\\
 &\qquad-\frac{1}{\alpha_t} \alpha_t \langle(V^{\theta_{t+1}}(s_{t+1}) - \phi(s_{t+1})^T v_{t+1})\nabla \log\pi_{\theta_t}(a_t|s_t) ,  M(\theta_t,v_t) \rangle\\
  & \qquad +\frac{1}{\alpha_t} \alpha_{t + 1} \langle(V^{\theta_{t+1}}(s_{t+1}) - \phi(s_{t+1})^T v_{t+1})\nabla \log\pi_{\theta_{t+1}}(a_{t+1}|s_{t+1}) ,  M(\theta_{t + 1},v_{t+1}) \rangle\\
  &\qquad - \langle \nabla L(\theta_t) , (L(\theta_t) - L_t)\nabla \log\pi_{\theta_t}(a_t|s_t) \rangle + M_{L}\alpha_t\Vert\delta_{t}\nabla \log\pi_{\theta_t}(a_t|s_t)\Vert^2,
\end{align*}
where, in the above, $Q_t = \alpha_t \langle (V^{\theta_t}(s_t) - \phi(s_{t})^Tv_t)\nabla \log\pi_{\theta_t}(a_t|s_t) , M(\theta_t,v_t)\rangle$. Taking expectations on both sides and summing from $\tau_t$ to $t$, we obtain,
\begin{align*}
    &\sum\limits_{k=\tau_t}^{t}E\Vert M(\theta_k,v_k)\Vert^2\\ 
    & \leq \underbrace{\sum\limits_{k=\tau_t}^{t}E[(L(\theta_{k+1}) - L(\theta_k) + Q_{k} - Q_{k+1})/\alpha_k]}_{I_1}  -\underbrace{\sum\limits_{k=\tau_t}^{t}E[I(O_k,\theta_k,L_k,v_k)]}_{I_2}\\
 &\qquad - \underbrace{\sum\limits_{k=\tau_t}^{t}E[\Xi(O_k,\theta_k,v_k)]}_{I_3} - \underbrace{\sum\limits_{k=\tau_t}^{t}E[ \langle(V^{\theta_k}(s_{k+1})-V^{\theta_{k+1}}(s_{k+1}))\nabla\log\pi_{\theta_k}(a_k|s_k) ,  M(\theta_k,v_k)\rangle]}_{I_4}\\
 &\qquad  - \underbrace{\sum\limits_{k=\tau_t}^{t}E[\langle ( \phi(s_{k+1})^T v_{k+1} - \phi(s_{k+1})^T v_{k})\nabla \log\pi_{\theta_k}(a_k|s_k) ,  M(\theta_k,v_k)\rangle]}_{I_5}\\
 &\qquad-\underbrace{\sum\limits_{k=\tau_t}^{t}\frac{1}{\alpha_k} E[\alpha_k \langle(V^{\theta_{k+1}}(s_{k+1}) - \phi(s_{k+1})^T v_{k+1})\nabla \log\pi_{\theta_k}(a_k|s_k) ,  M(\theta_k,v_k) \rangle]}_{I_6}\\
  & \qquad +\underbrace{\sum\limits_{k=\tau_t}^{t}\frac{1}{\alpha_k}E[ \alpha_{k + 1} \langle(V^{\theta_{k+1}}(s_{k+1}) - \phi(s_{k+1})^T v_{k+1})\nabla \log\pi_{\theta_{k+1}}(a_{k+1}|s_{k+1}) ,  M(\theta_{k + 1},v_{k+1}) \rangle]}_{I_7}\\
  &\qquad - \underbrace{\sum\limits_{k=\tau_t}^{t}E[\langle \nabla L(\theta_k) , (L(\theta_k) - L_k)\nabla \log\pi_{\theta_k}(a_k|s_k)] \rangle]}_{I_8} + \underbrace{\sum\limits_{k=\tau_t}^{t}M_{L}\alpha_k E[\Vert\delta_{k}\nabla \log\pi_{\theta_k}(a_k|s_k)\Vert^2]}_{I_9}.\\ 
\end{align*}

Now, for term $I_1$ we have,
\begin{align*}
    &\sum\limits_{k=\tau_t}^{t}E[(L(\theta_{k+1}) - L(\theta_k) + Q_{k} - Q_{k+1})/\alpha_k]\\
    &= \sum\limits_{k=\tau_t}^{t}E[(A_{k+1}- A_k)/\alpha_k]\\
    &= \mathcal{O}(t^{\nu}),
\end{align*}
where $A_k = L(\theta_k)- Q_k$.

The analysis of term $I_1$ is similar to that of term $I_1$ in Section \ref{average_reward_convergence}.

For term $I_2$, we have,
\begin{align*} 
    -E[I(O_t,\theta_t,L_t,v_t)] &\leq D_1(\tau + 1)\sum_{k=t-\tau + 1}^{t} E\Vert \theta_k - \theta_{k-1} \Vert + D_2bk^{\tau - 1} + D_3E\Vert v_t - v_{t-\tau} \Vert\\
    &\qquad+ D_4E\vert L_t - L_{t-\tau} \vert. 
\end{align*}
This inequality comes from lemma \ref{Gamma_actor}.

After summing both sides from $\tau_t$ to t and taking $\tau = \tau_t$, we will get,

\begin{align*}
    I_2 = \mathcal{O}(\log^2 t \cdot t^{1-\nu}).
\end{align*}

For term $I_3$, we have,
\begin{align*}
    -E[\Xi (O_t,\theta_t,v_t)] &\leq 4B(G + B(\bar{U}_{v} + U_v))\Vert v_t - v_{t-\tau}  \Vert + D\Vert \theta_t - \theta_{t-\tau}\Vert\\
    &\qquad +B_1\sum\limits_{i=t-\tau}^{t}E\Vert \theta_{i} - \theta_{t-\tau} \Vert + B_2bk^{\tau-1}.
\end{align*}
This inequality is a result of lemma \ref{Xi_actor}.

After taking $\tau = \tau_t$ and summing the expectation on both the sides from $\tau_t$ to t, we get,

\begin{align*}
    I_3 = \mathcal{O}(\log^2 t \cdot t^{1-\nu}),
\end{align*}
where 
the $\log^2 t$ term arises here because of the definition of $\tau_t$ (see \cref{eq:def_mixing_time}).

For term $I_4$, we have,
\begin{align*}
    &-\sum\limits_{k=\tau_t}^{t}E[E_{\theta_k}[ \langle(V^{\theta_k}(s_{k+1})-V^{\theta_{k+1}}(s_{k+1}))\nabla\log\pi_{\theta_k}(a_k|s_k) ,  M(\theta_k,v_k)\rangle]]\\
    & \leq \sum\limits_{k=\tau_t}^{t}E[E_{\theta_k}[ \Vert(V^{\theta_k}(s_{k+1})-V^{\theta_{k+1}}(s_{k+1}))\nabla\log\pi_{\theta_k}(a_k|s_k)\Vert \Vert M(\theta_k,v_k)\Vert]]\\
    & \leq 4B^3(U_r + U_v)^2 L_{v}\sum\limits_{k=\tau_t}^{t}\alpha_k.
\end{align*}

The last inequality follows from Assumption \ref{V_lipschitz}.

Next, for the term $I_5$, we have, 
\begin{align*}
    &-\sum\limits_{k=\tau_t}^{t}E[E_{\theta_k}[\langle ( \phi(s_{k+1})^T v_{k+1} - \phi(s_{k+1})^T v_{k})\nabla \log\pi_{\theta_k}(a_k|s_k) ,  M(\theta_k,v_k)\rangle]]\\
    & \leq \sum\limits_{k=\tau_t}^{t}E[E_{\theta_k}[\Vert ( \phi(s_{k+1})^T v_{k+1} - \phi(s_{k+1})^T v_{k})\nabla \log\pi_{\theta_k}(a_k|s_k)\Vert \Vert  M(\theta_k,v_k)\Vert]]\\
    & \leq 4B^2(U_r + U_v)^2\sum\limits_{k=\tau_t}^{t}\beta_k.
\end{align*}

For terms $I_6$ and $I_7$ summed together, we have,
\begin{align*}
    &I_6 + I_7\\
    &= \sum\limits_{k=\tau_t}^{t}\frac{1}{\alpha_k}E[ \alpha_{k + 1} \langle(V^{\theta_{k+1}}(s_{k+1}) - \phi(s_{k+1})^T v_{k+1})\nabla \log\pi_{\theta_{k+1}}(a_{k+1}|s_{k+1}) ,  M(\theta_{k + 1},v_{k+1}) \rangle]\\
    &\qquad - \sum\limits_{k=\tau_t}^{t}\frac{1}{\alpha_k} E[\alpha_k \langle(V^{\theta_{k+1}}(s_{k+1}) - \phi(s_{k+1})^T v_{k+1})\nabla \log\pi_{\theta_k}(a_k|s_k) ,  M(\theta_k,v_k) \rangle]\\
    &= \sum\limits_{k=\tau_t}^{t}\frac{1}{\alpha_k}E[ \alpha_{k + 1} \langle(V^{\theta_{k+1}}(s_{k+1}) - \phi(s_{k+1})^T v_{k+1})\nabla \log\pi_{\theta_{k+1}}(a_{k+1}|s_{k+1}) ,  M(\theta_{k + 1},v_{k+1}) \rangle]\\
    &\qquad - \sum\limits_{k=\tau_t}^{t}\frac{1}{\alpha_k}E[ \alpha_{k} \langle(V^{\theta_{k+1}}(s_{k+1}) - \phi(s_{k+1})^T v_{k+1})\nabla \log\pi_{\theta_{k+1}}(a_{k+1}|s_{k+1}) ,  M(\theta_{k + 1},v_{k+1}) \rangle]\\
     &\qquad + \sum\limits_{k=\tau_t}^{t}\frac{1}{\alpha_k}E[ \alpha_{k} \langle(V^{\theta_{k+1}}(s_{k+1}) - \phi(s_{k+1})^T v_{k+1})\nabla \log\pi_{\theta_{k+1}}(a_{k+1}|s_{k+1}) ,  M(\theta_{k + 1},v_{k+1}) \rangle]\\
      &\qquad - \sum\limits_{k=\tau_t}^{t}\frac{1}{\alpha_k}E[ \alpha_{k} \langle(V^{\theta_{k+1}}(s_{k+1}) - \phi(s_{k+1})^T v_{k+1})\nabla \log\pi_{\theta_{k}}(a_{k}|s_{k}) ,  M(\theta_{k + 1},v_{k+1}) \rangle]\\
       &\qquad + \sum\limits_{k=\tau_t}^{t}\frac{1}{\alpha_k}E[ \alpha_{k} \langle(V^{\theta_{k+1}}(s_{k+1}) - \phi(s_{k+1})^T v_{k+1})\nabla \log\pi_{\theta_{k}}(a_{k}|s_{k}) ,  M(\theta_{k + 1},v_{k+1}) \rangle]\\
    &\qquad - \sum\limits_{k=\tau_t}^{t}\frac{1}{\alpha_k} E[\alpha_k \langle(V^{\theta_{k+1}}(s_{k+1}) - \phi(s_{k+1})^T v_{k+1})\nabla \log\pi_{\theta_k}(a_k|s_k) ,  M(\theta_k,v_k) \rangle]\\
    &\leq GB(U_v + \bar{U}_{v})\sum_{k=\tau_t}^{t}\frac{\alpha_{k+1}-\alpha_k}{\alpha_k} + \mathcal{O}\bigg(\sum_{k=\tau_t}^{t}\Vert \theta_k - \theta_{k+1}\Vert \bigg)\\
    &\qquad+ \mathcal{O}\bigg( \sum_{k=\tau_t}^{t}\Vert M(\theta_{k+1},v_{k+1}) - M(\theta_k,v_k)\Vert\bigg)\\
    &= \mathcal{O}(t^{1-\nu}).
\end{align*}

For term $I_8$, we have,
\begin{align}
    -&\ \sum\limits_{k=\tau_t}^{t}E[\langle \nabla L(\theta_k) , (L(\theta_k) - L_k)\nabla \log\pi_{\theta_k}(a_k|s_k)] \rangle]\notag\\
    =&\ \sum\limits_{k=\tau_t}^{t}E[\langle E_{\theta_k}[(r(s,a) - L(\theta_k) + V^{\theta_k}(s^{'}) - V^{\theta_k}(s))\nabla \log \pi_{\theta_k}(a|s)] \\
    &\qquad \qquad \qquad, ( L_k - L(\theta_k))\nabla \log\pi_{\theta_k}(a_k|s_k)] \rangle]\notag\\
    = &\ \sum\limits_{k=\tau_t}^{t}E[\langle E_{\theta_k}[(r(s,a) - L(\theta_k) + (\phi(s^{'}) - \phi(s))^Tv(k))\nabla \log \pi_{\theta_k}(a|s)]\\
    &\qquad, ( L_k - L(\theta_k))\nabla \log\pi_{\theta_k}(a_k|s_k)] \rangle]\notag\\
    &+ \sum\limits_{k=\tau_t}^{t}E[\langle E_{\theta_k}[( V^{\theta_k}(s^{'})- \phi(s^{'})^Tv_{k} + \phi(s)^Tv_{k} - V^{\theta_k}(s))\nabla \log \pi_{\theta_k}(a|s)] \\
    &\qquad \qquad, ( L_k - L(\theta_k))\nabla \log\pi_{\theta_k}(a_k|s_k)] \rangle]\notag\\
    \leq  &\ B\sqrt{\sum\limits_{k=\tau_t}^{t}E\Vert M(\theta_k,v_k)\Vert^2}\sqrt{\sum\limits_{k=\tau_t}^{t}E\vert L_k - L(\theta_k) \vert^2} + I_{8a}.\label{ineqaulity_8a}
\end{align}

where 

\begin{align*}
    I_{8a} &= \sum\limits_{k=\tau_t}^{t}E[\langle E_{\theta_k}[( V^{\theta_k}(s^{'})- \phi(s^{'})^Tv_{k} + \phi(s)^Tv_{k} - V^{\theta_k}(s))\nabla \log \pi_{\theta_k}(a|s)]\\
    &\qquad \qquad , ( L_k - L(\theta_k))\nabla \log\pi_{\theta_k}(a_k|s_k)] \rangle]
\end{align*}

Now, for the term $I_{8a}$, we have,

\begin{align*}
   I_{8a}= I_{8a1} + I_{8a2}.
\end{align*}

where,

\begin{align*}
     I_{8a1} = \sum\limits_{k=\tau_t}^{t}E[\langle E_{\theta_k}[ \bar{W}(O_k,\theta_k,v_k)] - \bar{W}(O_k,\theta_k,v_k)]  , ( L_k - L(\theta_k))\nabla \log\pi_{\theta_k}(a_k|s_k)] \rangle]
\end{align*}

and,

\begin{align*}
    I_{8a2} &= \sum\limits_{k=\tau_t}^{t}E[\langle ( V^{\theta_k}(s_{k+1})- \phi(s_{k+1})^Tv_{k} + \phi(s_k)^Tv_{k} - V^{\theta_k}(s_k))\nabla \log \pi_{\theta_k}(a_k|s_k)\\
    &\qquad \qquad , ( L_k - L(\theta_k))\nabla \log\pi_{\theta_k}(a_k|s_k)] \rangle]
\end{align*}

After analysing the term $I_{8a1}$, similar to lemma \ref{Xi_actor}, we get,

\begin{align*}
    I_{8a1} = \mathcal{O}(\log ^2 t \cdot t^{1-\nu}).
\end{align*}

For the term $I_{8a2}$, we have,

\begin{align*}
    &\sum\limits_{k=\tau_t}^{t}E[\langle ( V^{\theta_k}(s_{k+1})- \phi(s_{k+1})^Tv_{k} + \phi(s_k)^Tv_{k} - V^{\theta_k}(s_k))\nabla \log \pi_{\theta_k}(a_k|s_k)\\
    &\qquad, ( L_k - L(\theta_k))\nabla \log\pi_{\theta_k}(a_k|s_k)] \rangle]\\
    &= \sum\limits_{k=\tau_t}^{t}E[\langle ( V^{\theta_k}(s_{k+1})- \phi(s_{k+1})^Tv_{k} )\nabla \log \pi_{\theta_k}(a_k|s_k) , ( L_k - L(\theta_k))\nabla \log\pi_{\theta_k}(a_k|s_k)] \rangle]\\
    &\qquad - \sum\limits_{k=\tau_t}^{t}E[\langle ( V^{\theta_k}(s_{k})- \phi(s_{k})^Tv_{k} )\nabla \log \pi_{\theta_k}(a_k|s_k) , ( L_k - L(\theta_k))\nabla \log\pi_{\theta_k}(a_k|s_k)] \rangle]\\
    & = \sum\limits_{k=\tau_t}^{t}E[\langle ( V^{\theta_k}(s_{k+1})- \phi(s_{k+1})^Tv_{k} )\nabla \log \pi_{\theta_k}(a_k|s_k) , ( L_k - L(\theta_k))\nabla \log\pi_{\theta_k}(a_k|s_k)] \rangle]\\
    &\qquad - \sum\limits_{k=\tau_t}^{t}E[\langle ( V^{\theta_{k+1}}(s_{k+1})- \phi(s_{k+1})^Tv_{k+1} )\nabla \log \pi_{\theta_{k+1}}(a_{k+1}|s_{k+1}) ,\\
    &\qquad \qquad( L_{k+1} - L(\theta_{k+1}))\nabla \log\pi_{\theta_{k+1}}(a_{k+1}|s_{k+1})] \rangle]\\
    &\qquad + \sum\limits_{k=\tau_t}^{t}E[\langle ( V^{\theta_{k+1}}(s_{k+1})- \phi(s_{k+1})^Tv_{k+1} )\nabla \log \pi_{\theta_{k+1}}(a_{k+1}|s_{k+1}),\\
    &\qquad \qquad( L_{k+1} - L(\theta_{k+1}))\nabla \log\pi_{\theta_{k+1}}(a_{k+1}|s_{k+1})] \rangle]\\
     &\qquad - \sum\limits_{k=\tau_t}^{t}E[\langle ( V^{\theta_k}(s_{k})- \phi(s_{k})^Tv_{k} )\nabla \log \pi_{\theta_k}(a_k|s_k) , ( L_k - L(\theta_k))\nabla \log\pi_{\theta_k}(a_k|s_k)] \rangle]\\
     &= \mathcal{O}(\sum\limits_{k=\tau_t}^{t}\Vert \theta_{k+1} - \theta_{k} \Vert) + \mathcal{O}(\sum\limits_{k=\tau_t}^{t}\Vert v_{k+1} - v_{k} \Vert) + E\sum\limits_{k=\tau_t}^{t}(P_{k+1} - P_{k} )\\
     & =  \mathcal{O}(\sum\limits_{k=\tau_t}^{t}\alpha_{k} ) + \mathcal{O}(\sum\limits_{k=\tau_t}^{t} \beta_{k} ) + E\sum\limits_{k=\tau_t}^{t}(\alpha_{k} P_{k+1} - \alpha_{k}P_{k} )/\alpha_k\\
     &= \mathcal{O}(t^{1-\nu}) + E\sum\limits_{k=\tau_t}^{t}(\alpha_{k+1} P_{k+1} - \alpha_{k}P_{k} )/\alpha_k + E\sum\limits_{k=\tau_t}^{t}(\alpha_{k} - \alpha_{k+1} )P_{k+1}/\alpha_k\\
     &= \mathcal{O}( t^{1-\nu}) + \mathcal{O}(t^{\nu}).
\end{align*}

Hence, putting all these results back in \cref{ineqaulity_8a}, we obtain,

\begin{align*}
    I_8 \leq  &\ B\sqrt{\sum\limits_{k=\tau_t}^{t}E\Vert M(\theta_k,v_k)\Vert^2}\sqrt{\sum\limits_{k=\tau_t}^{t}E\vert L_k - L(\theta_k) \vert^2} + \mathcal{O}(\log ^2 t \cdot t^{1-\nu})+ \mathcal{O}(t^{\nu}).
\end{align*}

For term $I_9$, we have,

\begin{align*}
    \sum\limits_{k=\tau_t}^{t}M_{L}\alpha_k E[\Vert\delta_{k}\nabla \log\pi_{\theta_k}(a_k|s_k)\Vert^2] &= \mathcal{O}(\sum\limits_{k=\tau_t}^{t}\alpha_k) \\
    & = \mathcal{O}(t^{1-\nu}).
\end{align*}

Hence after collecting all the terms, we obtain,

\begin{align*}
    \sum\limits_{k=\tau_t}^{t}E\Vert M(\theta_k,v_k)\Vert^2 &= \mathcal{O}(t^\nu) + \mathcal{O}(\log^2 t \cdot t^{1-\nu})\\
    &\qquad +B\sqrt{\sum\limits_{k=\tau_t}^{t}E\Vert M(\theta_k,v_k)\Vert^2}\sqrt{\sum\limits_{k=\tau_t}^{t}E\vert L_k - L(\theta_k) \vert^2}.
\end{align*}

After applying the squaring technique, we obtain,

\begin{align*}
    \sum\limits_{k=\tau_t}^{t}E\Vert M(\theta_k,v_k)\Vert^2 &= \mathcal{O}(t^\nu) + \mathcal{O}(\log^2 t \cdot t^{1-\nu}) +2B^2\sum\limits_{k=\tau_t}^{t}E\vert L_k - L(\theta_k) \vert^2\\
    & \leq \mathcal{O}(t^\nu) + \mathcal{O}(\log^2 t \cdot t^{1-\nu})  + 4B^{2}\frac{(G + U_w)^2}{(1 - \frac{c_\alpha}{c_\gamma}U_w B)^2}\frac{c_\alpha^2}{c_\gamma^2}\sum\limits_{k=\tau_t}^{t}\mathbb{E}\Vert M(\theta_k,v_k)\Vert^2.
\end{align*}

The last inequality follows from \cref{ineq1}.

 Now if we select the values for $c_\alpha$ and $c_\gamma$ such that ${\displaystyle\frac{4B^{2}(G + U_w)^2}{(1 - \frac{c_\alpha}{c_\gamma}U_w B)^2}\frac{c_\alpha^2}{c_\gamma^2} < 1}$, we shall obtain,
 \begin{align*}
     \sum\limits_{k=\tau_t}^{t}E\Vert M(\theta_k,v_k)\Vert^2 &= \mathcal{O}(t^\nu) + \mathcal{O}(\log^2 t \cdot t^{1-\nu}).
 \end{align*}
 Dividing by $(1 + t - \tau_t)$  and assuming $t \geq 2\tau_t + 1$, we have,
 \begin{align}\label{ineq2}
     \frac{1}{(1 + t - \tau_t)}\sum\limits_{k=\tau_t}^{t}E\Vert M(\theta_k,v_k)\Vert^2 &= \mathcal{O}(t^{\nu - 1}) + \mathcal{O}(\log^2 t \cdot t^{-\nu}).
 \end{align}

 As seen earlier, the inequalities that need to be satisfied for the inequalities (\ref{ineq1}) and (\ref{ineq2}) to hold are the following:
\begin{align}
    &\frac{c_\alpha}{c_\gamma} < \frac{1}{U_{w}B}\label{ineq3},\\
    &\frac{2B(G + U_w)}{(1 - \frac{c_\alpha}{c_\gamma}U_w B)}\frac{c_\alpha}{c_\gamma} <1.
    \label{ineq4}
\end{align} 

Rearranging inequality (\ref{ineq4}), we get

\begin{align}
    &2B(G + U_w)\frac{c_\alpha}{c_\gamma} < 1 - \frac{c_\alpha}{c_\gamma}U_w B\notag\\
    \Rightarrow &(2B(G + U_w) + U_{w}B)\frac{c_\alpha}{c_\gamma} < 1\notag\\
    \Rightarrow & \frac{c_\alpha}{c_\gamma} < \frac{1}{2B(G + U_w) + U_{w}B}.
    \label{ineq5}
\end{align}

Now, from (\ref{ineq3}) and (\ref{ineq5}), we have,
\begin{align*}
    \frac{c_\alpha}{c_\gamma} < \min \bigg( \frac{1}{2B(G + U_w) + U_{w}B} , \frac{1}{U_{w}B} \bigg).
\end{align*}

Since ${\displaystyle \frac{1}{2B(G + U_w) + U_{w}B} < \frac{1}{U_{w}B}}$, we need to choose $c_\alpha$ and $c_\gamma$ such that ${\displaystyle \frac{c_\alpha}{c_\gamma} <  \frac{1}{2B(G + U_w) + U_{w}B}}$.

 Now getting back to inequality (\ref{ineq2}), we can observe that $E\Vert M(\theta_k,v_k)\Vert^2 \rightarrow 0$ as $k \rightarrow \infty$.
Finally, 
as the actor is on the faster timescale compared to the critic, the critic appears to be quasi-static from the viewpoint of the actor. Hence, we may let $v_t \equiv v, \forall t \geq 0$ (i.e., $v$ independent of $t$). Therefore, the point of convergence of the actor parameter sequence $\theta_t$ will be $\theta(v)$ such that :
 \begin{align*}
     E_{\theta(v)}[( r(s,a)- L(\theta(v))  + \phi(s^{'})^{\top} v - \phi(s)^{\top} v)\nabla \log\pi_{\theta(v)}(a|s)] = 0.
\end{align*}

\subsubsection{Convergence of the Critic}
\label{critic_convergence_proof}
Recall that we have the following update rule for the critic:
\begin{align*}
    v_{t+1} = v_{t} + \beta_{t} \delta_t \phi(s_t).
\end{align*}

Notations:
\begin{align}
    \begin{split}
         O_t :&= (s_t , a_t , s_{t+1})\\
        z_t &:= v_t - v^{*}(\theta_t)\\
        g(O_t,v_t,\theta_t) &:= (r_t - L(\theta_t) + \phi(s_{t+1})^{\top} v_{t} - \phi(s_t)^{\top} v_{t})\phi(s_t)\\
        \bar{g}(v_t,\theta_t) &:= E_{s \sim \mu_{\theta_t},a \sim \pi_{\theta_t},s^{'} \sim p(.|s,a)}[(r(s,a) - L(\theta_t) + \phi(s^{'})^{\top} v_{t} - \phi(s)^{\top} v_{t})\phi(s)]\\
        \bar{Q}(O_t,v_t,\theta_t) &:= \langle z_t , g(O_t,v_t,\theta_t) - \bar{g}(v_t,\theta_t)\rangle \\
        \bar{U}(O_t,v_t,\theta_t) &:= (\nabla v_t^*)^T (r(s_t,a_t) - L(\theta_t) + \phi(s_{t+1})^{\top} v_{t} - \phi(s_t)^{\top} v_{t}) \nabla_{\theta} \log \pi_{\theta_{t}}(a_t|s_t)\\
        \Psi(O_t,v_t,\theta_t) &:= \langle z_t ,  E_{\theta_t}[\bar{U}(O_t,v_t,\theta_t)] - \bar{U}(O_t,v_t,\theta_t)\rangle.
    \end{split}
\end{align}

The proof of convergence of the critic is established through the following lemmas:

\begin{lemma}[\citep{wu2022finite}, Proposition 4.4]\label{ae3}
There exists a constant $L_{\ast}>0$ such that
\begin{align*}
    \Vert v^\ast(\theta_1)-v^\ast(\theta_2)\Vert\leq L_{\ast}\Vert \theta_1-\theta_2\Vert, \forall \theta_1,\theta_2\in\mathbb{R}^d.
\end{align*}
\end{lemma}

\begin{lemma}[\citep{chen2023finitetime} , Lemma B.4]\label{ae4}
    For any $\theta_1 , \theta_2 \in \RR^{d}$ , we have 
    \begin{align*}
        \Vert \nabla v^{*}(\theta_1) - \nabla v^{*}(\theta_2) \Vert \leq L_{m}\Vert \theta_1 - \theta_2\Vert,
    \end{align*}
    where $L_{m}$ is a positive constant.
\end{lemma}

%\begin{proof}
 %   For proof please see the proof of Lemma B.4 in \citep{chen2023finitetime}.
%\end{proof}

\begin{lemma}\label{psi_bound}
Under assumptions \ref{assum:bounded_feature_norm}, \ref{assum:ergodicity}, \ref{assum:policy-lipschitz-bounded}, \ref{V_lipschitz}, for any $t \geq \tau > 0$ , we have
    \begin{align*}
    \vert E [\Psi(O_t,v_t,\theta_t)]\vert  \leq G_{1}\Vert \theta_t - \theta_{t-\tau}\Vert + G_{2}\Vert v_t - v_{t-\tau} \Vert + G_{3}\sum\limits_{i = t-\tau}^{t}E\Vert \theta_{i} - \theta_{t-\tau}\Vert + G_{4}bk^{\tau-1},
\end{align*}
where $G_{1} > 0, G_{2} > 0 , G_{3} > 0 $ and $G_{4} > 0$ are constants.
\end{lemma}

\begin{proof}
    We can write $\vert E \Psi(O_t,v_t,\theta_t)\vert$ as follows:
    \begin{align*}
        &\vert E [\Psi(O_t,v_t,\theta_t)]\vert\\
        &= \vert E [\Psi(O_t,v_t,\theta_t)] - E [\Psi(O_t,v_t,\theta_{t-\tau})] + E[\Psi(O_t,v_t,\theta_{t-\tau})] - E[\Psi(O_t,v_{t-\tau},\theta_{t-\tau})] \\
        &\qquad + E[\Psi(O_t,v_{t-\tau},\theta_{t-\tau})] - E[\Psi(\tilde{O_t},v_{t-\tau},\theta_{t-\tau})] + E[\Psi(\tilde{O_t},v_{t-\tau},\theta_{t-\tau})]\vert\\
        &\leq \underbrace{\vert E [\Psi(O_t,v_t,\theta_t)] - E [\Psi(O_t,v_t,\theta_{t-\tau})] \vert}_{I_a} + \underbrace{\vert E[\Psi(O_t,v_t,\theta_{t-\tau})] - E[\Psi(O_t,v_{t-\tau},\theta_{t-\tau})]\vert}_{I_b} \\
        &\qquad + \underbrace{\vert E[\Psi(O_t,v_{t-\tau},\theta_{t-\tau})] - E[\Psi(\tilde{O_t},v_{t-\tau},\theta_{t-\tau})] \vert}_{I_c} + \underbrace{\vert E[\Psi(\tilde{O_t},v_{t-\tau},\theta_{t-\tau})] \vert}_{I_d}
    \end{align*}

In the above inequality, $\tilde{O_t} = (\tilde{s}_t , \tilde{a}_t, \tilde{s}_{t+1})$ is from the auxiliary Markov chain defined in \cref{chain-au} and $O_t^{'} = (s^{'}_t , a^{'}_t ,s^{'}_{t+1} )$, where $ s^{'}_t \sim \mu_{\theta_{t-\tau}} , a^{'}_t \sim \pi_{\theta_{t-\tau}}$ and $s^{'}_{t+1} \sim P(.|s^{'}_t, a^{'}_t)$. 

For term $I_a$, we have,
\begin{align}
    &\vert E [\Psi(O_t,v_t,\theta_t)] - E [\Psi(O_t,v_t,\theta_{t-\tau})] \vert \notag\\
    &= \vert E[\langle z_t ,  E_{\theta_t}[\bar{U}(O_t,v_t,\theta_t)] - \bar{U}(O_t,v_t,\theta_t)\rangle] - E[\langle v_t - v^{*}(\theta_{t-\tau})) ,  E_{\theta_{t-\tau}}[\bar{U}(O_t,v_t,\theta_{t-\tau}))] \\
    &\qquad - \bar{U}(O_t,v_t,\theta_{t-\tau}))\rangle] \vert \notag\\
    &\leq \vert  E[\langle z_t ,  E_{\theta_t}[\bar{U}(O_t,v_t,\theta_t)] - \bar{U}(O_t,v_t,\theta_t)\rangle] -  E[\langle v_t - v^{*}(\theta_{t-\tau}) ,  E_{\theta_t}[\bar{U}(O_t,v_t,\theta_t)] \\
    &\qquad - \bar{U}(O_t,v_t,\theta_t)\rangle] \vert \notag\\
    &\qquad + \vert E[\langle v_t - v^{*}(\theta_{t-\tau}) ,  E_{\theta_t}[\bar{U}(O_t,v_t,\theta_t)] - \bar{U}(O_t,v_t,\theta_t)\rangle] \\
    &\qquad - E[\langle v_t - v^{*}(\theta_{t-\tau})) ,  E_{\theta_{t-\tau}}[\bar{U}(O_t,v_t,\theta_{t-\tau}))] - \bar{U}(O_t,v_t,\theta_{t-\tau}))\rangle] \vert \notag\\
    &\leq \vert  E[\langle v^{*}(\theta_{t-\tau}) - v^{*}(\theta_t) ,  E_{\theta_t}[\bar{U}(O_t,v_t,\theta_t)] - \bar{U}(O_t,v_t,\theta_t)\rangle]\vert\notag\\
    &\qquad + \vert E[\langle v_t - v^{*}(\theta_{t-\tau}) ,  E_{\theta_t}[\bar{U}(O_t,v_t,\theta_t)] - E_{\theta_{t-\tau}}[\bar{U}(O_t,v_t,\theta_{t-\tau}))] - \bar{U}(O_t,v_t,\theta_t) \\
    &\qquad + \bar{U}(O_t,v_t,\theta_{t-\tau})\rangle]  \vert\notag\\
    &\leq \vert  E[\langle v^{*}(\theta_{t-\tau}) - v^{*}(\theta_t) ,  E_{\theta_t}[\bar{U}(O_t,v_t,\theta_t)] - \bar{U}(O_t,v_t,\theta_t)\rangle]\vert\notag\\
    &\qquad + 2U_{v}\vert E[ \Vert E_{\theta_t}[\bar{U}(O_t,v_t,\theta_t)] - E_{\theta_{t-\tau}}[\bar{U}(O_t,v_t,\theta_{t-\tau}))]\Vert  \\
    &\qquad + \Vert \bar{U}(O_t,v_t,\theta_t) - \bar{U}(O_t,v_t,\theta_{t-\tau})\Vert]  \vert\label{inequality_I_a}.
\end{align}
Now,

\begin{align*}
     \Vert \bar{U}(O_t,v_t,\theta_t) - \bar{U}(O_t,v_t,\theta_{t-\tau})\Vert \leq A_{1}\Vert \theta_t - \theta_{t-\tau}\Vert,
\end{align*}
where $A_{1} > 0$ is some constant.
This inequality follows from Lemma \ref{ae4} as well as Lemma B.1 of \citep{wu2022finite}. 
Also,
\begin{align*}
    \Vert E_{\theta_t}[\bar{U}(O_t,v_t,\theta_t)] - E_{\theta_{t-\tau}}[\bar{U}(O_t,v_t,\theta_{t-\tau}))]\Vert \leq A_{2}\Vert \theta_t - \theta_{t-\tau}\Vert,
\end{align*}

for some $A_{2} > 0$. Lemma B.1 of \citep{wu2022finite} is used in obtaining the above inequality.

Hence putting these results back in inequality (\ref{inequality_I_a}), we have,

\begin{align*}
    I_a &\leq 2LG\Vert v^{*}(\theta_{t-\tau}) - v^{*}(\theta_t) \Vert + 2U_{v}(A_{1} + A_{2})\Vert \theta_t - \theta_{t-\tau}\Vert\\
    &\leq (2LGL_{*} + 2U_{v}(A_{1} + A_{2}))\Vert \theta_t - \theta_{t-\tau}\Vert.
\end{align*}
The last inequality is because of Lemma \ref{ae3}.

For term $I_b$, we have,
\begin{align*}
    &\vert E[\Psi(O_t,v_t,\theta_{t-\tau})] - E[\Psi(O_t,v_{t-\tau},\theta_{t-\tau})]\vert\\
    & = \vert E[\langle v_t - v^{*}(\theta_{t-\tau}) ,  E_{\theta_{t-\tau}}[\bar{U}(O_t,v_t,\theta_{t-\tau})] - \bar{U}(O_t,v_t,\theta_{t-\tau})\rangle]\\
    &\qquad- E[\langle v_{t-\tau} - v^{*}(\theta_{t-\tau}) ,  E_{\theta_{t-\tau}}[\bar{U}(O_t,v_{t-\tau},\theta_{t-\tau})] - \bar{U}(O_t,v_{t-\tau},\theta_{t-\tau})\rangle] \vert\\
    &=\vert E[\langle v_t - v^{*}(\theta_{t-\tau}) ,  E_{\theta_{t-\tau}}[\bar{U}(O_t,v_t,\theta_{t-\tau})] - \bar{U}(O_t,v_t,\theta_{t-\tau})\rangle]\\
    &\qquad - E[\langle v_{t-\tau} - v^{*}(\theta_{t-\tau}) ,  E_{\theta_{t-\tau}}[\bar{U}(O_t,v_t,\theta_{t-\tau})] - \bar{U}(O_t,v_t,\theta_{t-\tau})\rangle]\\
    &\qquad + E[\langle v_{t-\tau} - v^{*}(\theta_{t-\tau}) ,  E_{\theta_{t-\tau}}[\bar{U}(O_t,v_t,\theta_{t-\tau})] - \bar{U}(O_t,v_t,\theta_{t-\tau})\rangle]\\
    &\qquad - E[\langle v_{t-\tau} - v^{*}(\theta_{t-\tau}) ,  E_{\theta_{t-\tau}}[\bar{U}(O_t,v_{t-\tau},\theta_{t-\tau})] - \bar{U}(O_t,v_{t-\tau},\theta_{t-\tau})\rangle] \vert\\
    &\leq \vert E[\langle v_t - v_{t-\tau},  E_{\theta_{t-\tau}}[\bar{U}(O_t,v_t,\theta_{t-\tau})] - \bar{U}(O_t,v_t,\theta_{t-\tau})\rangle] \vert \\
    &\qquad + \vert E[\langle v_{t-\tau} - v^{*}(\theta_{t-\tau}) ,  E_{\theta_{t-\tau}}[\bar{U}(O_t,v_t,\theta_{t-\tau})] - E_{\theta_{t-\tau}}[\bar{U}(O_t,v_{t-\tau},\theta_{t-\tau})] \\
    &\qquad- \bar{U}(O_t,v_t,\theta_{t-\tau}) + \bar{U}(O_t,v_{t-\tau},\theta_{t-\tau})\rangle] \vert\\
    &\leq 2L_{*}G\Vert v_t - v_{t-\tau} \Vert + 8BL_{*}U_{v}\Vert v_t - v_{t-\tau} \Vert.
\end{align*}

For term $I_c$, we have,

\begin{align*}
    \vert E[\Psi(O_t,v_{t-\tau},\theta_{t-\tau})] - E[\Psi(\tilde{O_t},v_{t-\tau},\theta_{t-\tau})] \vert \leq M_{1}\sum\limits_{i = t-\tau}^{t}E\Vert \theta_{i} - \theta_{t-\tau}\Vert,
\end{align*}
for some $M_1 > 0$.\\

For term $I_d$, we have,

\begin{align*}
    \vert E[\Psi(\tilde{O_t},v_{t-\tau},\theta_{t-\tau})] \vert \leq M_{2}bk^{\tau-1},
\end{align*}
for some $M_2 > 0$.

For an analysis of terms $I_c$ and $I_d$, see Lemmas D.10 and D.11 in \citep{wu2022finite}.
Hence, after collecting all the terms, we have,
\begin{align*}
    \vert E [\Psi(O_t,v_t,\theta_t)]\vert  \leq G_{1}\Vert \theta_t - \theta_{t-\tau}\Vert + G_{2}\Vert v_t - v_{t-\tau} \Vert + G_{3}\sum\limits_{i = t-\tau}^{t}E\Vert \theta_{i} - \theta_{t-\tau}\Vert + G_{4}bk^{\tau-1},
\end{align*}
where $G_{1} > 0, G_{2} > 0 , G_{3} > 0 $ and $G_{4} > 0$ are constants.
\end{proof}

\subsubsection*{Proof of convergence of critic}

From the critic update rule, we have,
\begin{align*}
    \Vert z_{t+1} \Vert^2 &= \Vert  v_{t+1} - v^{*}(\theta_{t+1})\Vert^2\\
    &=\Vert \Gamma(v_t +\beta_t\delta_t\phi(s_t)) - v^{*}(\theta_{t+1})  \Vert^2\\
    &\leq \Vert v_t + \beta_t\delta_t\phi(s_t) - v^{*}(\theta_{t+1}) \Vert^2\\
    &= \Vert z_t + \beta_t\delta_t\phi(s_t)  + v^{*}(\theta_t)- v^{*}(\theta_{t+1}) \Vert^2\\
    &\leq  \Vert z_t \Vert^2 + 2\beta_t\langle z_t ,  \delta_t\phi(s_t) \rangle + 2\langle z_t , v^{*}(\theta_t) - v^{*}(\theta_{t+1}) \rangle + 2\beta_t^2\delta_t^2\Vert \phi(s_t)\Vert^2 \\
    &\qquad + 2\Vert v^{*}(\theta_t) - v^{*}(\theta_{t+1})  \Vert^2\\
    &=\Vert z_t \Vert^2 + 2\beta_t\langle z_t ,  \delta_t\phi(s_t) - E_{\theta_t}[\delta_t\phi(s_t)]\rangle +
    2\beta_t\langle z_t ,   E_{\theta_t}[\delta_t\phi(s_t)]\rangle\\
    &\qquad +
    2\langle z_t , v^{*}(\theta_t) - v^{*}(\theta_{t+1}) \rangle
    + 2\beta_t^2\delta_t^2\Vert \phi(s_t)\Vert^2 + 2\Vert v^{*}(\theta_t) - v^{*}(\theta_{t+1})  \Vert^2\\
    &\leq \Vert z_t \Vert^2 + 2\beta_t\langle z_t ,  \delta_t\phi(s_t) - E_{\theta_t}[\delta_t\phi(s_t)]\rangle -
    2\beta_t\lambda\Vert z_t\Vert^2 +
    2\langle z_t , v^{*}(\theta_t) - v^{*}(\theta_{t+1}) \rangle\\
    &\qquad+ 2\beta_t^2\delta_t^2\Vert \phi(s_t)\Vert^2 + 2\Vert v^{*}(\theta_t) - v^{*}(\theta_{t+1})  \Vert^2.
\end{align*}

We assume here that the projection set $C$ is large enough so that $v^{*}(\theta_{t+1})$ lies within the set. Also,  $C$ being both compact and convex guarantees that the point within $C$ where the update with an increment is projected is not only the closest but also unique.
The last inequality follows from Assumption \ref{assum:negative-definite}. Here, $-\lambda = \sup\limits_{\theta}\lambda_{\theta}$, where $\lambda_{\theta}$ is an upper bound on the largest eigenvalue of $\Ab$ defined in \ref{critic_conv_point}.
After rearranging the terms we obtain,

\begin{align*}
    \lambda\Vert z_t \Vert^2 &\leq \frac{1}{2\beta_t}(\Vert z_t \Vert^2 - \Vert z_{t+1} \Vert^2) + \langle z_t ,  \delta_t\phi(s_t) - E_{\theta_t}[\delta_t\phi(s_t)]\rangle + \frac{1}{\beta_t}\langle z_t , v^{*}(\theta_t) - v^{*}(\theta_{t+1}) \\
    &\qquad + (\nabla v_t^*)^T(\theta_{t+1} - \theta_t)\rangle
    + \frac{1}{\beta_t}\langle z_t , (\nabla v_t^*)^T(\theta_{t} - \theta_{t+1}) \rangle +\beta_t\delta_t^2\Vert \phi(s_t)\Vert^2 \\
    &\qquad + \frac{1}{\beta_t}\Vert v^{*}(\theta_t) - v^{*}(\theta_{t+1})  \Vert^2.
\end{align*}

Taking summation of terms from indices $\tau_t$ to $t$ we have,

\begin{align*}
     \lambda\sum_{k=\tau_t}^{t}E\Vert z_k \Vert^2 &\leq \underbrace{\sum_{k=\tau_t}^{t}\frac{1}{2\beta_k}E[\Vert z_k \Vert^2 - \Vert z_{k+1} \Vert^2] }_{I_1}+ \underbrace{\sum_{k=\tau_t}^{t}E[\langle z_k ,  \delta_t\phi(s_k) - E_{\theta_k}[\delta_k\phi(s_k)]\rangle]}_{I_2}\\
     &\qquad+ \underbrace{\sum_{k=\tau_t}^{t}\frac{1}{\beta_k}E\langle z_k , v^{*}(\theta_k) - v^{*}(\theta_{k+1}) + (\nabla v_k^*)^T(\theta_{k+1} - \theta_k)\rangle}_{I_3}\\
    &\qquad +\underbrace{\sum_{k=\tau_t}^{t} \frac{1}{\beta_k}E\langle z_k , (\nabla v_k^*)^T(\theta_{k} - \theta_{k+1}) \rangle}_{I_4} +\underbrace{\sum_{k=\tau_t}^{t}\beta_kE[\delta_k^2\Vert \phi(s_k)\Vert^2]}_{I_5}\\
    &\qquad+ \underbrace{\sum_{k=\tau_t}^{t}\frac{1}{\beta_k}E\Vert v^{*}(\theta_k) - v^{*}(\theta_{k+1})  \Vert^2}_{I_6}\\
\end{align*}

For term $I_1$, we have,

\begin{align*}
    \sum_{k=\tau_t}^{t}\frac{1}{2\beta_k}E[\Vert z_k \Vert^2 - \Vert z_{k+1} \Vert^2] = \mathcal{O}(t^{\sigma})
\end{align*}

The analysis of $I_1$ is similar to that of the  term $I_1$ in Section \ref{average_reward_convergence}.
For term $I_2$ here, we have,
\begin{align*}
    \sum_{k=\tau_t}^{t}E[\langle z_k ,  \delta_t\phi(s_k) - E_{\theta_k}[\delta_k\phi(s_k)]\rangle] = \mathcal{O}(\log^2 t \cdot t^{1-\nu}).
\end{align*}
For a detailed analysis of the term $I_2$, see the analysis of term $I_2$ in \citep{chen2023finitetime}.

For term $I_3$ above, we have,
\begin{align*}
    \sum_{k=\tau_t}^{t}\frac{1}{\beta_k}E\langle z_k , v^{*}(\theta_k) - v^{*}(\theta_{k+1}) + (\nabla v_k^*)^T(\theta_{k+1} - \theta_k)\rangle &\leq \frac{L_{m}}{2}\sum_{k=\tau_t}^{t}\frac{1}{\beta_k}E\Vert z_k\Vert \Vert \theta_{k+1} - \theta_k \Vert^2\\
    &=\mathcal{O}(\sum_{k=\tau_t}^{t}\frac{\alpha_k^2}{\beta_k})\\
    &=\mathcal{O}(t^{\sigma - 2\nu + 1}).
\end{align*}

The above inequality follows from the $L_{m}$-smoothness of $v^{*}$ in Lemma \ref{ae4}.
For term $I_4$, we have,
\begin{align*}
    &\sum_{k=\tau_t}^{t} \frac{1}{\beta_k}E\langle z_k , (\nabla v_k^*)^T(\theta_{k} - \theta_{k+1}) \rangle\\
    &= -\sum_{k=\tau_t}^{t} \frac{1}{\beta_k}E\langle z_k , (\nabla v_k^*)^T \alpha_k \delta_k \nabla_{\theta} \log \pi_{\theta_{k}}(a_k|s_k) \rangle\\
    &= -\sum_{k=\tau_t}^{t} \frac{1}{\beta_k}E\langle z_k , (\nabla v_k^*)^T \alpha_k (r(s_k,a_k) - L_k + \phi(s_{k+1})^{\top} v_{k} - \phi(s_k)^{\top} v_{k}) \nabla_{\theta} \log \pi_{\theta_{k}}(a_k|s_k) \rangle\\
    &= -\sum_{k=\tau_t}^{t} \frac{1}{\beta_k}E\langle z_k , (\nabla v_k^*)^T \alpha_k (r(s_k,a_k) - L(\theta_k) + \phi(s_{k+1})^{\top} v_{k} - \phi(s_k)^{\top} v_{k}) \nabla_{\theta} \log \pi_{\theta_{k}}(a_k|s_k) \rangle\\
    &\qquad - \sum_{k=\tau_t}^{t} \frac{1}{\beta_k}E\langle z_k , (\nabla v_k^*)^T \alpha_k (L(\theta_k) - L_k) \nabla_{\theta} \log \pi_{\theta_{k}}(a_k|s_k) \rangle\\
    & = -\sum_{k=\tau_t}^{t} \frac{\alpha_k}{\beta_k}E\langle z_k , (\nabla v_k^*)^T (r(s_k,a_k) - L(\theta_k) + \phi(s_{k+1})^{\top} v_{k} - \phi(s_k)^{\top} v_{k}) \nabla_{\theta} \log \pi_{\theta_{k}}(a_k|s_k) \rangle\\
    &\qquad + \sum_{k=\tau_t}^{t} \frac{\alpha_k}{\beta_k}E\langle z_k , (\nabla v_k^*)^T E_{\theta_k}[(r(s_k,a_k) - L(\theta_k) + \phi(s_{k+1})^{\top} v_{k} - \phi(s_k)^{\top} v_{k}) \nabla_{\theta} \log \pi_{\theta_{k}}(a_k|s_k) ]\rangle\\
    &\qquad -  \sum_{k=\tau_t}^{t} \frac{\alpha_k}{\beta_k}E\langle z_k , (\nabla v_k^*)^T E_{\theta_k}[(r(s_k,a_k) - L(\theta_k) + \phi(s_{k+1})^{\top} v_{k} - \phi(s_k)^{\top} v_{k}) \nabla_{\theta} \log \pi_{\theta_{k}}(a_k|s_k) ]\rangle\\
    &\qquad - \sum_{k=\tau_t}^{t} \frac{\alpha_k}{\beta_k}E\langle z_k , (\nabla v_k^*)^T  (L(\theta_k) - L_k) \nabla_{\theta} \log \pi_{\theta_{k}}(a_k|s_k) \rangle\\
    &= \sum_{k=\tau_t}^{t}E[\frac{\alpha_k}{\beta_k}\Psi(O_k,v_k,\theta_k)]\\
    &\qquad-  \sum_{k=\tau_t}^{t} \frac{\alpha_k}{\beta_k}E\langle z_k , (\nabla v_k^*)^T E_{\theta_k}[(r(s_k,a_k) - L(\theta_k) + \phi(s_{k+1})^{\top} v_{k} - \phi(s_k)^{\top} v_{k}) \nabla_{\theta} \log \pi_{\theta_{k}}(a_k|s_k) ]\rangle\\
    &\qquad - \sum_{k=\tau_t}^{t} \frac{\alpha_k}{\beta_k}E\langle z_k , (\nabla v_k^*)^T  (L(\theta_k) - L_k) \nabla_{\theta} \log \pi_{\theta_{k}}(a_k|s_k) \rangle\\
    &\leq \frac{c_\alpha}{c_\beta}\sum_{k=\tau_t}^{t}E[(1+k)^{\sigma - \nu}\Psi(O_k,v_k,\theta_k)] + L_{*}\sqrt{\sum\limits_{k=\tau_t}^{t}E\Vert z_k\Vert^2}\sqrt{\sum\limits_{k=\tau_t}^{t}E[\frac{\alpha_k^2}{\beta_k^2}\Vert M(\theta_k,v_k)\Vert^2]}\\
    &\qquad + L_{*}B\sqrt{\sum\limits_{k=\tau_t}^{t}E\Vert z_k\Vert^2}\sqrt{\sum\limits_{k=\tau_t}^{t}E[\frac{\alpha_k^2}{\beta_k^2}(L(\theta_k)- L_k)^2]} \\
    &\leq \frac{c_\alpha}{c_\beta}(1+t)^{\sigma - \nu}\sum_{k=\tau_t}^{t}\vert E[\Psi(O_k,v_k,\theta_k)]\vert + L_{*}\sqrt{\sum\limits_{k=\tau_t}^{t}E\Vert z_k\Vert^2}\sqrt{\sum\limits_{k=\tau_t}^{t}E[\frac{\alpha_k^2}{\beta_k^2}\Vert M(\theta_k,v_k)\Vert^2]}\\
    &\qquad + L_{*}B\sqrt{\sum\limits_{k=\tau_t}^{t}E\Vert z_k\Vert^2}\sqrt{\sum\limits_{k=\tau_t}^{t}E[\frac{\alpha_k^2}{\beta_k^2}(L(\theta_k)- L_k)^2]} \\
    &\leq \mathcal{O}(\log^2 t \cdot t^{\sigma - 2\nu + 1}) + L_{*}\sqrt{\sum\limits_{k=\tau_t}^{t}E\Vert z_k\Vert^2}\sqrt{\sum\limits_{k=\tau_t}^{t}E[\frac{\alpha_k^2}{\beta_k^2}\Vert M(\theta_k,v_k)\Vert^2]}
    \end{align*}
    \[+ L_{*}B\sqrt{\sum\limits_{k=\tau_t}^{t}E\Vert z_k\Vert^2}\sqrt{\sum\limits_{k=\tau_t}^{t}E[\frac{\alpha_k^2}{\beta_k^2}(L(\theta_k)- L_k)^2]}.
\]

The last inequality follows from Lemma \ref{psi_bound}.

For the term $I_5$, we have,
\begin{align*}
    \sum_{k=\tau_t}^{t}\beta_kE[\delta_k^2\Vert \phi(s_k)\Vert^2] = \mathcal{O}(t^{1-\sigma}).
\end{align*}

Next, for the term $I_6$, we have,

\begin{align*}
    \sum_{k=\tau_t}^{t}\frac{1}{\beta_k}E\Vert v^{*}(\theta_k) - v^{*}(\theta_{k+1})  \Vert^2 = \mathcal{O}(t^{1 - 2\nu +\sigma}). 
\end{align*}

For detailed analysis of terms $I_5$ and $I_6$, see section C.2 of \citep{chen2023finitetime}.
Thus, after collecting all the terms we have,

\begin{align*}
    \lambda\sum_{k=\tau_t}^{t}E\Vert z_k \Vert^2 &\leq  \mathcal{O}(t^{\sigma}) + \mathcal{O}(\log^2 t \cdot t^{1-\nu}) + \mathcal{O}(\log^2 t \cdot t^{\sigma - 2\nu + 1})\\
    &\qquad+ L_{*}\sqrt{\sum\limits_{k=\tau_t}^{t}E\Vert z_k\Vert^2}\sqrt{\sum\limits_{k=\tau_t}^{t}E[\frac{\alpha_k^2}{\beta_k^2}M(\theta_k,v_k)^2]} \\
    &\qquad +
     L_{*}B\sqrt{\sum\limits_{k=\tau_t}^{t}E\Vert z_k\Vert^2}\sqrt{\sum\limits_{k=\tau_t}^{t}E[\frac{\alpha_k^2}{\beta_k^2}(L(\theta_k)- L_k)^2]}\\
     &\qquad+ \mathcal{O}(t^{1-\sigma}) .
\end{align*}

After applying the squaring technique, we have,

\begin{align*}
    \sum_{k=\tau_t}^{t}E\Vert z_k \Vert^2 
    &= \mathcal{O}(t^{\sigma}) + \mathcal{O}(\log^2 t \cdot t^{1-\nu}) + \mathcal{O}(\log^2 t \cdot t^{1 + \sigma - 2\nu})\\
    &\qquad + \mathcal{O}(\sum\limits_{k=\tau_t}^{t}E[\frac{\alpha_k^2}{\beta_k^2}M(\theta_k,v_k)^2]) + \mathcal{O}(\sum\limits_{k=\tau_t}^{t}E[\frac{\alpha_k^2}{\beta_k^2}(L(\theta_k)- L_k)^2])\\
    & = \mathcal{O}(t^{\sigma}) + \mathcal{O}(\log^2 t \cdot t^{1-\nu}) + \mathcal{O}(\log^2 t \cdot t^{1 + \sigma - 2\nu}) + \mathcal{O}(t^{2\sigma - \nu}) \\
    &\qquad +\mathcal{O}(\log^2 t \cdot t^{1-3\nu + 2\sigma})\\
    &= \mathcal{O}(\log^2 t \cdot t^{1 + \sigma - 2\nu}) + \mathcal{O}(t^{2\sigma - \nu}) +\mathcal{O}(\log^2 t \cdot t^{1-3\nu + 2\sigma}).
\end{align*}

The second equality above comes from the result of Theorems \ref{average_reward_convergence} and \ref{actor_convergence}.
Assuming $t \geq 2\tau_{t} - 1$, we have,

\begin{align*}
    \frac{1}{1+t-\tau_t}\sum_{k=\tau_t}^{t}E\Vert z_k \Vert^2 &= \mathcal{O}(\log^2 t \cdot t^{ \sigma - 2\nu}) + \mathcal{O}(t^{2\sigma - \nu - 1}) +\mathcal{O}(\log^2 t \cdot t^{-3\nu + 2\sigma}).
\end{align*}

So, we can observe that $E\Vert z_t\Vert^2 \rightarrow 0$ as $t \rightarrow \infty$, if the following conditions are satisfied:
\begin{align*}
    2\sigma - \nu &< 1, \\
    2\sigma &< 3 \nu.
\end{align*}

Hence $(v_t - v^{*}(\theta_t)) \rightarrow 0$ as $t \rightarrow \infty$.

Now as actor is on the faster timescale compared to critic , we can say that when observed from the timescale of critic, $\theta_t$ closely tracks $\theta(v_t)$.Therefore we can conclude that $v_t$ converges to a point w such that $w - v^{*}(\theta(w)) = 0$.\\

Optimising over the values of $\nu$ and $\sigma$ we have $\nu = 0.5$ and $\sigma = 0.5 + \beta$ where $\beta >0$ can be made arbitrarily close to zero. Hence we have  the following :-
\begin{align*}
    \frac{1}{1+t-\tau_t}\sum_{k=\tau_t}^{t}E\Vert z_k \Vert^2 &= \mathcal{O}(\log^2 t \cdot t^{(2\beta - 0.5)})
\end{align*}

Now, 
\begin{align*}
    2\beta & >0 \\
    \Rightarrow 2\beta-0.5 & > -0.5\\
    \Rightarrow \frac{1}{2\beta-0.5} & < -2
\end{align*}
So we can write $\frac{1}{2\beta-0.5} = -2-\delta $ where $\delta >0$ is arbitrarily close to zero as $\beta >0$ is made arbitrarily close to zero.

Therefore in order for the mean squared error of the critic to be upper bounded by $\epsilon$, namely,

\begin{align*}
     \frac{1}{1+t-\tau_t}\sum_{k=\tau_t}^{t}E\Vert z_k \Vert^2  =  \mathcal{O}(\log^2 T \cdot T^{(2\beta - 0.5)}) \leq \epsilon,
\end{align*}
we need to set $T = \tilde{\mathcal{O}}(\epsilon^{-(2+\delta)})$ where $\delta >0$ can be made arbitrarily close to zero.\\

\begin{remark}
    It is important to note here that unlike \cite{wu2022finite}, none of our results have terms corresponding to function approximation errors as we do not encounter the term $\Delta h^{'}(O , \theta)$ in our analysis. For the definition of $\Delta h^{'}(O , \theta)$ please refer the proof of Theorem 4.5 of \cite{wu2022finite}.
\end{remark}

%%%%%%%%%%%%%%%%%%%%%%%%%%%%%%%%%%%
\subsection{Asymptotic Analysis}
\label{asympanalysis}

%\subsubsection{Two timescale CA Algorithm}

% \begin{algorithm}
%    \caption{Two Timescale CA Algorithm}
%    \label{algo}
% \begin{algorithmic}
%    \STATE {\bfseries Input:} initial actor parameter $\theta_0$, initial critic parameter $v_0$, step-size $\alpha_{t}$ for actor, $\beta_{t}$ for critic and $\gamma_{t}$ for the average reward estimator.
%    \STATE Draw $s_{0}$ from some initial distribution.
%    \FOR {$t=0,1,2,\dots$}
%     \STATE Take the action $a_t \sim \pi_{\theta_t}(\cdot|s_t)$
%     \STATE Observe next state $s_{t+1} \sim P(\cdot|s_t,a_t)$ and the reward $r_t = r(s_t,a_t)$
%     \STATE $L_{t+1} = L_t + \gamma_t(r_t - L_t)$
%     \STATE $\delta_t = r_t - L_t + \phi(s_{t+1})^{\top} v_{t} - \phi(s_t)^{\top} v_{t}$
%     \STATE $v_{t+1} = \Gamma_1(v_{t} + \beta_{t} \delta_t \phi(s_t))$\label{algline:critic_update}
%     \STATE $\theta_{t+1} = \Gamma_2(\theta_{t} + \alpha_t \delta_t \nabla_{\theta} \log \pi_{\theta_{t}}(a_t|s_t))$\label{algline:actor_update}
% \ENDFOR 
% \end{algorithmic}
% \end{algorithm}

We now analyse Algorithm \ref{algo} which represents the two-timescale CA algorithm involving linear function approximation for its asymptotic convergence. 
To begin with, we first present the asymptotic analysis of (almost sure) convergence by using two projection operators $\Gamma_1$ and $\Gamma_2$ respectively. Let $C$ and $D$ be compact subsets of $\mathbb{R}^{d_1}$ and $\mathbb{R}^{d_2}$, respectively. Then,
$\Gamma_1: \mathbb{R}^{d_1}\rightarrow C$ and $\Gamma_2: \mathbb{R}^d\rightarrow D$ are the two projection operators. These operators ensure that the algorithm remains stable throughout the run. 

Subsequently, in Section~\ref{unp-critic}, we remove the projection operator $\Gamma_1$ and consider an un-projected critic that can take values in the whole of $\mathbb{R}^{d_1}$. We also prove the stability of the recursion in this case in addition to asymptotic convergence. In other words, we show that $\sup_{t} \|v_t\| <\infty$ w.p.1. Such results on asymptotic stability and almost sure convergence are not available in many non-asymptotic (finite-time) analyses of algorithms in the literature. The resulting scheme that we analyse is then similar to AC algorithms in the literature, where one projects the actor but not the critic, see for instance, \cite{BHATNAGAR20092471}, except that now the time scales of the recursions are reversed. 

We prove the stability and convergence of our two-timescale CA algorithm by proving that the critic recursion asymptotically tracks a compact connected internally chain transitive invariant set of an associated differential inclusion (DI) \citep{aubin, benaim-inclusions}. A DI-based analysis is a generalization of the ordinary differential equation (ODE) approach to stochastic approximation recursions and is necessitated because we allow for multiple local maxima for the actor-recursion for any given critic update that result in an underlying DI instead of an ODE. In the context of AC or CA algorithms, ours is the first analysis that incorporates this level of sophistication and generality. 
 
As noted in Assumption~\ref{assum:ss}, all step-sizes satisfy the standard Robbins-Monro conditions. In addition, $\beta_t = o(\alpha_t)$ for $t \geq 0$  and $\gamma_t = K\alpha_t$ for some $K > 0$, $t \geq 0$. As a result of this, the average reward and actor updates are performed on the faster timescale compared to the critic updates. 

\subsubsection{The Case of Projected Critic}
\label{proj-critic}

To begin with, we consider the case where the critic recursion is projected using 
a projection operator $\Gamma_1(\cdot)$ to a compact and convex set $C\subset \mathbb{R}^{d_1}$. Thus, for any $x\in \mathbb{R}^{d_1}$, $\Gamma_1(x) \in C$. For any vector $y \in C$, we have $\Vert y \Vert \leq U_{\theta}$,  where $U_{\theta} > 0$ is a constant. As mentioned earlier, the single-stage reward is a function of the current state and action taken.%We consider here the following step-sizes: $\alpha_t = c_\alpha/(1+t)^\nu$, $\beta_t = c_\beta/(1+t)^\sigma , \gamma_t = c_\gamma/(1+t)^{\nu}$ with $0 < \nu < \sigma < 1$. Thus, the actor and the average reward recursions proceed here on the same timescale but which is faster than the critic recursion. 

\subsubsection*{Asymptotic Convergence of average reward estimate and the actor}
\label{avg_reward_convergence}

We begin here with the case where both critic and actor recursions are projected to certain compact and convex sets.
In this case, we rewrite the critic recursion as follows:
\begin{eqnarray}
\label{v1} 
v_{t+1} &=& \Gamma_1(v_{t} + \beta_{t}  \delta_t \phi(s_t)),\\
\label{v2}
 &=& \Gamma_1(v_{t} + \alpha_t \left(\frac{\beta_{t}}{\alpha_t}\right)  \delta_t \phi(s_t)),
 \end{eqnarray}
where 
\begin{equation}
\label{del}
\delta_t = r_t - L_t + \phi(s_{t+1})^{\top} v_{t} - \phi(s_t)^{\top} v_{t}.\end{equation}

Observe now that since $\beta_t = o(\alpha_t)$, the above asymptotically tracks the ordinary differential equation (ODE)
\[
\dot{v}(t) =0,
\]
which means we may let $v(t)\equiv v$ (i.e., independent of $t$).
Now consider the average reward recursion, viz.,
\begin{equation}
\label{Leq}
L_{t+1} = L_t + \gamma_t(r_t - L_t),
\end{equation}
where $\gamma_t = K\alpha_t$. For simplicity, let $K=1$. In other words, the timescales of recursions governed by $\gamma_t$ and $\alpha_t$ are the same and faster while the recursion governed by $\beta_t$ is on the slower timescale.
Recall also the (faster) $\theta$-update rule here:

\begin{equation}
\label{theta}
\theta_{t+1} = \Gamma_2(\theta_{t} + \alpha_t \delta_t \psi_{s_ta_t}),    
\end{equation}
where $\psi_{s_ta_t} = \nabla_{\theta} \log \pi_{\theta_{t}}(a_t|s_t)$.

Consider now the following system of ODEs associated with all the above recursions:
\begin{eqnarray}
\label{o1}
\dot{v} &=&0,\\
\label{o2}
\dot{L} &=& -L + \sum_{s}\mu_{\theta}(s)\sum_a \pi^{\theta}(a|s)R(s,a), \\
\label{o3}
\dot{\theta} &=& \hat{\Gamma}_2\left(E[\delta_t \nabla_\theta\log \pi^\theta(a_t|s_t)|\theta]\right),
\end{eqnarray}
where the operator $\hat{\Gamma}_2(\cdot)$ is defined as 
\begin{align*}
    \hat{\Gamma}_2(v(y)) = \lim\limits_{0 < \eta \rightarrow \infty}\bigg( \frac{\Gamma_2(y + \eta v(y)) - y}{\eta}\bigg).
\end{align*}
Here, as before, $\mu_{\theta}(s), s\in S$ denotes the stationary distribution of the state-valued Markov chain $\{s_t\}$ when actions in the MDP are chosen as per the parameterized policy $\pi^\theta$. 

The system of ODEs (\ref{o1})-(\ref{o3}) has the following set of equilibria:
\[\{(v,L^{\theta},\theta)| \theta\in \theta^*(v), v\in \mathbb{R}^{d_1}\},\]
where for any $v\in \mathbb{R}^{d_1}$, $\theta\in \theta^*(v)$,
\begin{equation}
\label{Ltheta}
L^{\theta} = \sum_{s} \mu_{\theta}(s)\sum_a \pi^{\theta}(s,a) R(s,a),
\end{equation}
is the unique globally asymptotically stable equilibrium of the ODE
(\ref{o2}). Note here that for any $v\in \mathbb{R}^{d_1}$, $\theta^*(v)$ is the set of asymptotically stable equilibria of the ODE (\ref{o3}). 
%(assuming $v(t)=v$ in the light of (\ref{o1}). 
Here $R(s,a)$ is the expected single-stage reward when state is $s$ and action $a$ is chosen. Thus, $R(s,a) = E[r_t|s_t=s,a_t=a]$. 

By Assumption~\ref{assum:policy-lipschitz-bounded}, $\pi^\theta$ is continuously differentiable. Let
\[
\check{p}_\theta(s,s') \stackrel{\triangle}{=} \sum_{a\in A(s)} \pi^\theta(s,a)p(s,a,s').
\]
Then, from Assumption~\ref{assum:policy-lipschitz-bounded}, $\check{p}_\theta(s,s')$ are continuously differentiable (in $\theta$) transition probabilities of the resulting Markov chain, $\forall s,s'\in S$. We have the following preliminary result:

\begin{lemma}
\label{lem0}
Under Assumptions~\ref{assum:policy-lipschitz-bounded} and \ref{assum-ergodic-chain}, 
$L^\theta$, $\theta\in C$ is continuously differentiable in $\theta$.
\end{lemma}

\begin{proof}
By Theorem~\ref{thm-stationary}, $\nu_\theta$ is continuously differentiable in $\theta$. Recall that Assumption~\ref{assum:policy-lipschitz-bounded} has been used in that proof. 
Now since the set of states and actions is finite, from the definition of $L^\theta$ in (\ref{Ltheta}), it is easy to see that $L^\theta$ is continuously differentiable as well.
\end{proof}

Notice that the equilibrium of (\ref{o3}) depends on $v$ since $\delta_t$ depends on $v$. Also, in general, for any given $v$, $\theta^*(v)$ will not be a unique point. In fact, this is the set of local minima of the associated performance objective under function approximation. 

For $\theta\in\theta^*(v)$, let 
\[
\delta^{\theta} = r_t - L^{\theta} + \phi(s_{t+1})^Tv - \phi(s_t)^T v.
\]
\begin{lemma}
\label{lem1}
Under Assumptions~\ref{assum:policy-lipschitz-bounded} and \ref{assum-ergodic-chain}, for $\theta\in\theta^*(v)$, we have 
\[E[\delta^{\theta}\psi_{s_ta_t}|\theta] = \nabla_\theta L^{\theta} + \sum_s \mu_{\theta}(s) \nabla \bar{V}^{\pi^{\theta}}(s),
\]  
\end{lemma}
where
\[
\bar{V}^{\pi^{\theta}}(s)
= \sum_{a} \pi^{\theta}(s,a) \left[R(s,a) - L^{\pi^{\theta}} + \sum_{s'} p(s,a,s') v^T\phi(s')
\right].
\]
\begin{proof}
Note that for $\theta\in\theta^*(v)$, 
\[
\nabla_\theta \bar{V}^{\pi^{\theta}}(s)
= \sum_{a} \pi^{\theta}(s,a)( -\nabla_\theta L^{\pi^{\theta}} )\]\[
+
\sum_a \nabla_\theta \pi^{\theta}(s,a) \left(
R(s,a) - L^{\pi^{\theta}} + \sum_{s'} p(s,a,s') v^T \phi(s') \right).
\]
Hence,
\[
\sum_s \mu_{\theta}(s) \nabla_\theta \bar{V}^{\pi^{\theta}}(s)
=
- \sum_s \mu_{\theta}(s) \sum_{a} \pi^{\theta}(s,a)\nabla_\theta L^{\pi^{\theta}}
\]
\[
+ \sum_s \mu_{\theta}(s) \sum_a \nabla_\theta \pi^{\theta}(s,a) \left(
R(s,a) - L^{\pi^{\theta}} + \sum_{s'} p(s,a,s') v^T \phi(s') \right).
\]
\[
= -\nabla_\theta L^{\pi^{\theta}} + E[\delta^{\theta}\psi_{s_ta_t}|\theta]. 
%+ \sum_s d^{\pi^{\theta^*(v)}}(s) \nabla_\theta \bar{V}^{\pi^{\theta^*}(v)}(s).
\]
The claim now follows by rearranging the terms above.    
\end{proof}

\begin{remark}
\label{rem1}
Lemma~\ref{lem1} is similar to Lemma 4 of \citep{BHATNAGAR20092471}. The main difference between these two results is that since in \citep{BHATNAGAR20092471}, the critic is faster than the actor, the actor parameter $\theta$ is held fixed while analysing the critic recursion. As a result the critic parameter $v$ there depends on $\theta$ unlike here where the critic parameter $v$ is held fixed. As a result, there is an additional term that appears on the RHS of Lemma 4 of \citep{BHATNAGAR20092471} contributing to the bias in the estimator that is not present in Lemma~\ref{lem1}. 
Note also that Lemma~\ref{lem1} tells us that for $\theta\in\theta^*(v)$, $\delta^{\theta}\psi_{s_ta_t}$ is a biased gradient estimator of $L^{\theta}$ with $\sum_s \mu_{\theta}(s) \nabla \bar{V}^{\pi^{\theta}}(s)$ as the bias term.
\end{remark}

We now proceed with the analysis of the faster recursions.

\begin{proposition}
\label{lem2}
Under Assumptions \ref{assum:policy-lipschitz-bounded}, \ref{assum-ergodic-chain} and \ref{assum:ss}, for any given $v\in\mathbb{R}^{d_1}$, and $L_t$, $\theta_t$ updated as in (\ref{Leq})-(\ref{theta}), respectively, we have $L_t\rightarrow L^{\theta}$ and $\theta_t\rightarrow\theta$, $\theta\in\theta^*(v)$, where
\[
L^{\theta} = \sum_{s} \mu_{\theta}(s) \sum_a \pi^{\theta}(s,a) R(s,a),
\]
and $\theta^*(v)$ is the set of points $\theta$ where
$\hat{\Gamma}(E[\delta^{\theta}\psi_{s_ta_t}|\theta])=0$.
\end{proposition}

\begin{proof}
As described previously, the ODEs associated with (\ref{Leq})-(\ref{theta}) are (\ref{o2})-(\ref{o3}).
Consider (\ref{o2}) first and let 
\[
f(L) = -L + \sum_{s}\mu_{\theta}(s)\sum_a \pi^{\theta}(a|s)R(s,a),
\]
i.e., the RHS of (\ref{o2}). For any integer $c>0$, let
$f_c(L) \stackrel{\triangle}{=} f(cL)/c$. Since $\max_{s,a} |R(s,a)|<\infty$ as single-stage rewards are finite, $f_c(L)\rightarrow f_\infty(L) = -L$ as $c\rightarrow\infty$. The ODE 
\[
\dot{L} = f_\infty(L) = -L
\]
is clearly globally asymptotically stable to the origin. Thus, from Chapter 6 of \citep{borkar-SA}, it follows that $\sup_t|L_t|<\infty$ and that $L_t\rightarrow L^{\theta}$ almost surely as $t\rightarrow\infty$. 

Recall now the $\theta$-update (\ref{theta}):
\[
\theta_{t+1} = \Gamma_2(\theta_{t} + \alpha_t \delta_t \psi_{s_ta_t}).
\]
We rewrite the same as
\[\theta_{t+1} = \Gamma_2(\theta_{t} + \alpha_t
E[\delta^{\theta^*(v)}\psi_{s_ta_t}|\mathcal{F}_1(t)] + 
\alpha_t \gamma_1(t) + \alpha_t \gamma_2(t)), \]
where 
$\gamma_1(t) = \delta_t\psi_{s_ta_t} - E[\delta_t\psi_{s_ta_t}|\mathcal{F}_1(t)]$ and (for $\theta\in\theta^*(v)$), 
$\gamma_2(t) = E[(\delta_t\psi_{s_ta_t} - \delta^{\theta}\psi_{s_ta_t})|\mathcal{F}_1(t)]$, respectively, and with $\mathcal{F}_1(t) = \sigma(\theta_k,s_k,a_k, k\leq t)$, $t\geq 0$. Observe that $(\gamma_1(t),\mathcal{F}_1(t)), t\geq 0$, is a martingale difference sequence and is uniformly bounded for any given $v\in \mathbb{R}^{d_1}$. This is because rewards are bounded and so is the $L_t$-sequence as argued above. Moreover, since state and action spaces are finite, we have that $\max_{s,a}|\psi_{sa}| \leq \check{K}<\infty$, for some constant $\check{K}<\infty$. 
If we now form the sequence
\[
Z(t) = \sum_{k=0}^{t-1}\alpha_k \gamma_1(k), \mbox{ } t\geq 0,
\]
with $Z(0)=0$,
then it is easy to see that $(Z(t),\mathcal{F}_1(t)),t\geq 0$, forms a martingale. Moreover, 
from the square summability of the step-size sequence $\{\alpha_k\}$ and the fact that $\gamma_1(t)$ is uniformly bounded (see above), it can be shown from a routine application of the martingale convergence theorem for square integrable martingales (cf.~\citep{borkar-probability}), that $\sup_t |Z(t)| <\infty$ with probability one, or in other words, the martingale converges almost surely. Now observe that $\gamma_2(t)\rightarrow 0$ as $t\rightarrow\infty$ a.s. by Assumption~\ref{assum:ergodicity}. 

Now (for $\theta\in\theta^*(v)$), let \[
e^{\pi^{\theta}} \stackrel{\triangle}{=} \sum_s \mu_{\theta}(s) \nabla \bar{V}^{\pi^{\theta}}(s). \]
Consider the ODE:
\begin{equation}
\label{ode2}
\dot{\theta} = \hat{\Gamma}_2\left(\nabla L^{\theta} + e^{\pi^{\theta}} \right).
\end{equation}
%Since for given $v\in\mathbb{R}^d$, $\theta$ is the unique globally asymptotically stable equilibrium of the above ODE, 
It follows from Theorem 5.3.1 of \citep{kushnerclark} that $\theta_t\rightarrow \theta\in \theta^*(v)$ almost surely as $t\rightarrow\infty$. The claim follows.
\end{proof}

\begin{remark}
\label{rem21}
The claim in Proposition~\ref{lem2} indicates that the recursion (\ref{theta}) converges to the stationary points of the ODE (\ref{ode2}). One can now argue as in \citep{BHATNAGAR20092471}, that the convergence is indeed to local maxima of the ODE (\ref{ode2}) which are also stationary points of the above ODE.  
\end{remark}

% \begin{remark}
% \label{rem2}
% Recall that $\theta^*(v)$ is the set of stable equilibria of the ODE (\ref{ode2}. Consider now another associated ODE:
% \begin{equation}
%     \label{ode3}
%     \dot{\theta} = \hat{\Gamma}_2(\nabla L^\theta).
% \end{equation}
% Let the set of asymptotically stable equilibria (i.e., the maxima of $L^\theta$ correspond to $\check{\theta}(v)$. 
% It follows as a consequence of Lemma~\ref{lem2} and the continuity of $\nabla L^{\theta}$ and $e^{\pi^\theta}$ that given $\epsilon>0$, there exists $\delta>0$ such that if $\|e^{\pi^\theta}\|<\delta$, then $\theta_t \rightarrow \check{\theta}(v)^\epsilon \stackrel{\triangle}{=} \{\theta|\|\theta-\theta_0\|<\epsilon, \theta_0\in \check{\theta}(v)$.
% \end{remark}

\subsubsection*{Asymptotic Convergence of the Critic}
\label{critic_convergence_asymptotic}

We now turn to the analysis of the slower (critic) recursion (\ref{v1}). 
Similar to the actor recursion, we define the following projection operator that will be used for the critic ODE, 
where the operator $\hat{\Gamma}_2(\cdot)$ is defined as 
\begin{align*}
    \hat{\Gamma}_1(w(x)) = \lim\limits_{0 < \beta \rightarrow \infty}\bigg( \frac{\Gamma_1(x + \beta w(x)) - x}{\beta}\bigg).
\end{align*} 

Let 
\[
Y_{t} = \delta_t \phi(s_t) - E[\delta_t\phi(s_t)|\mathcal{F}_2(t)], 
\]
where $\mathcal{F}_2(t) =\sigma(v_r,\theta_r,s_r,a_r,r\leq t), t\geq 0$, is a sequence of associated sigma fields. Then $(Y_t, \mathcal{F}_2(t)),t\geq 0$ is a martingale difference sequence. 
Consider now the associated process
\[
\check{Z}(t) = \sum_{m=0}^{t-1} \alpha_m Y_m,
\]
$t\geq 0$ with $\check{Z}(0)=0$. Then $(\check{Z}(t),\mathcal{F}_2(t)),t\geq 0$ is a martingale sequence. Now because of the fact that single-stage rewards are uniformly bounded, the set of states and actions is finite and hence the state features are uniformly bounded, the fact that $v_n\in C$, $\forall n$, a compact set, and $
\sum_m\alpha_m^2<\infty$, it follows from the martingale convergence theorem for square integrable martingales (cf.~Chapter 3 of \citep{borkar-probability}) that $\check{Z}(t),t\geq 0$ is an almost surely convergent martingale sequence. 

We rewrite (\ref{v1}) as 
\begin{equation}
\label{v1-re}
v_{t+1} = \Gamma_1(v_t +\beta_t (y_t + Y_t +\kappa_t)),
\end{equation}
where
\[
y_t = \sum_{s} \mu_{\theta_t}(s) \sum_a\pi^{\theta_t}(s,a)(R(s,a) - L^{\theta_t} + v_t^T \sum_{s'} p(s,a,s')\phi(s') - v_t^T \phi(s))\phi(s),
\]
\[
\kappa_t = E[(R(s_t,a_t)-L^{\theta_t}+ v_t^T \sum_{s_{t+1}} p(s_t,a_t,s_{t+1})\phi(s_{t+1}) - v_t^T \phi(s_t))\phi(s_t)|\mathcal{F}_2(t)] - y_t,
\]
respectively, with $\theta_t \in {\theta^*}(v)$ (cf.~Remark~\ref{rem2}). Now from Assumption~\ref{assum:ergodicity}, it follows that $\kappa_t\rightarrow 0$ almost surely as $t\rightarrow\infty$. Proceeding as in Chapter 6 of \citep{borkar-SA}, one can rewrite 
(\ref{v1-re}) as
\begin{equation}
\label{v1-re2}
v_{t+1} = v_t + \beta_t \left(\frac{\Gamma_1(v_t +\beta_t (y_t + Y_t +\kappa_t))-v_t}{\beta_t}\right)
\end{equation}
\begin{equation}
\label{v1-re3}
= v_t + \beta_t(\check{\gamma}_1(v_t; y_t+Y_t) + \xi_t,
\end{equation}
where
\[
\check{\gamma}_1(v;y) = \lim_{\eta\rightarrow 0} \left(\frac{\Gamma_1(v+\eta y)-v}{\eta}\right),
\]
is the directional derivative of $\Gamma_1$ at $v$ in the direction $y$. Note that in (\ref{v1-re3}) and in what follows, we drop $\kappa_t$ since as explained before, $\kappa_t\rightarrow 0$ a.s.~as $t\rightarrow\infty$. Thus, in the limiting system, this term will not appear. 
From the definition of $\check{\gamma}_1(v,y)$, note that if $v\in C^o$ (i.e., the interior of $C$), then for any $y\in \mathbb{R}^{d_1}$, $\Gamma_1(v+\eta y) = v+\eta y$ for $\eta>0$ `small enough'. Thus, $\check{\gamma}_1(v;y) = y$ if $v\in C^o$.

Suppose now
\[
z_t \stackrel{\triangle}{=} E[\check{\gamma}_1(v_t; y_t+Y_t)|\mathcal{F}_2(t)]
\]
and
\[
\check{Y}_t \stackrel{\triangle}{=}
\check{\gamma}_1(v_t; y_t+Y_t) - z_t.
\]
Thus, (\ref{v1-re3}) can be rewritten as
\begin{equation}
\label{v1-re4}
v_{t+1} = v_t +\beta_t (z_t + \check{Y}_t + \xi_t).
\end{equation}

Let $h(v)$ denote the set-valued map
\[
h(v) = \{ \sum_s \mu_\theta(s)\sum_a \pi^\theta(s,a) (R(s,a) - L^\theta + v^T \sum_{s'} p(s,a,s')\phi(s')\]
\[-v^T \phi(s))\phi(s) | \theta \in \bar{\theta^*}(v) \},
\]
where $\bar{\theta}^*(v)$ is the closure of the set $\theta^*(v)$. Now since $\mu_\theta(s)$, $\pi^\theta(s,a)$, $L^\theta$ are all continuous functions of $\theta$ and $\bar{\theta^*}(v)$ is a compact set for any $v$, it follows that $h(v)$ is a compact set.

Now from the definition of $Y_t$, it can be seen that
\[
Y_t \in \{ v_t^T (\phi(s_{t+1}) - \sum_{s',a_t}\pi^\theta(s_t,a_t)p(s_t,a_t,s')\phi(s'))\phi(s_t) | \theta \in \bar{\theta^*}(v)\}.
\]
Since ${\displaystyle \{\sum_{s',a_t}\pi^\theta(s_t,a_t)p(s_t,a_t,s')|\theta \in \bar{\theta^*}(v) \}}$ is compact for every $v$ and $s_t$ takes only finitely many values, the conditional distribution of $Y_t$ given $\mathcal{F}_2(t)$ has a compact support $\mathcal{B}(v_t)$. %It is also then obvious that $\kappa_t$ has a compact support $\mathcal{A}(v_t)$. 

Let
\[
\hat{\Gamma}_v(h(v)) \stackrel{\triangle}{=} \cap_{\epsilon>0} \bar{co} \left( \cup_{\|w-v\|<\epsilon} \{\gamma_1(w; y+Y)| y\in h(v), Y\in \mathcal{B}(v)\}\right),
\]
where $\bar{co}(\cdot)$ is the closed convex hull of `$\cdot$'.
We now have the following result.
\begin{lemma}
\label{lem3}
We have
\begin{itemize}
\item[(i)] $\hat{\Gamma}_v(h(v))$ is a compact and convex set for any $v\in C$.
\item[(ii)] For all $v\in C$, 
\[
\sup_{w\in \hat{\Gamma}_v(h(v))} \|w\| \leq M(1+\|w\|),
\]
for some $M>0$. 
\item[(iii)] $\hat{\Gamma}_v(h(v))$ is upper semi-continuous, i.e., if $v_n\rightarrow v$ and $w_n\rightarrow w$ with $w_n \in
\hat{\Gamma}_{v_n}(h(v_n))$ for all $n$, then $w \in \hat{\Gamma}_v(h(v))$.    
\end{itemize}   
\end{lemma}

\begin{proof}
\begin{itemize}
\item[(i)] By definition $\hat{\Gamma}_v(h(v))$ is convex and closed. It is also easily seen to be bounded and hence is compact.
\item[(ii)] Let 
\[
Q \stackrel{\triangle}{=} \sup_{v\in C}\sup_{w\in \hat{\Gamma}_v(h(v))} \|w\|
\]
Then since $C$ is compact and $\hat{\Gamma}_v(h(v))$ is compact as well from the foregoing, $0<Q<\infty$. The claim then follows by letting $M=Q$.
\item[(iii)] Let 
\[
g(v) \stackrel{\triangle}{=} \cup \{\gamma_1(w; y+Y)| y\in h(v), Y\in \mathcal{B}(v)\}.
\]
Then,
\[
\hat{\Gamma}_v(h(v)) = \cap_{\epsilon>0} \bar{co}(\{g(v)| \|w-v\| <\epsilon\}).
\]
Note now that the family of sets 
$H(v,\epsilon) \stackrel{\triangle}{=} \bar{co}(\{g(v)| \|w-v\| <\epsilon\})$ with $\epsilon>0$ is a family of diminishing sets in $\epsilon$ and $H(v,\epsilon)\downarrow \hat{\Gamma}_v(h(v))$
as $\epsilon \downarrow 0$. 

Let $w_n \in \hat{\Gamma}_{v_n}(h(v_n))$ for $n$ large enough so that $\|v_n-v\|<\epsilon/3$ for some $\epsilon>0$. It is easy to find such a $v_n$ since $v_n\rightarrow v$ as $n\rightarrow\infty$. Then,
$w_n \in H(v_n, \frac{\epsilon}{3})$. It is then easy to see that
$H(v_n, \frac{\epsilon}{3}) \subset H(v,\frac{2\epsilon}{3})$ for $n$ large. Thus, given $\epsilon$, there exists $N_0>0$ such that for all $n>N_0$, $w_n \in H(v,\frac{2\epsilon}{3})$. Further, $w_n\rightarrow w$ as $n\rightarrow\infty$. Thus, $w\in H(v,\frac{2\epsilon}{3})$ since it is a closed set. The claim now follows since $H(v,\frac{2\epsilon}{3})\downarrow \hat{\Gamma}_v(h(v))$ as $\epsilon\downarrow 0$ implying that $w\in \hat{\Gamma}_{v}(h(v))$.
    \end{itemize}
Thus, all three claims are verified.
\end{proof}

\begin{remark}
\label{rem3}
Any set-valued map $H:\mathcal{R}^l \rightarrow \{\mbox{ subsets of }\mathcal{R}^m$ and satisfying the three claims in Lemma~\ref{lem3} is called a Marchaud or Peano map, cf.~ \citep{aubin}. If such a map is used for the driving vector field of a differential inclusion, the trajectories of the inclusion are absolutely continuous functions. 
\end{remark}

Define now a sequence of time points $\{t(n)\}$, $n\geq 0$ in the following manner: $t(0)=0$ and for $n\geq 1$, ${\displaystyle t(n) = \sum_{\tau=0}^{n-1} \beta_\tau}$. We now define a process $\bar{v}(t)$ obtained from the iterate sequence $\{v_n\}$ as given by the recursion (\ref{v1}) as follows: $v(t(n))= v_n$, $\forall n$ and for $t\in [t(n),t(n+1)]$, $\bar{v}(t)$ is obtained as
\[
\bar{v}(t) = \left(\frac{t(n+1)-t}{\beta_n}\right) v_n + \left(\frac{t-t(n)}{\beta_n}\right) v_{n+1}.
\]
Let 
\begin{equation}
\label{Gv}
G= \cap_{t\geq 0} \overline{\{\bar{v}(t+s)|s\geq 0\}}.
\end{equation}
Consider now the differential inclusion (DI):
\begin{equation}
\label{di}
\dot{v}(t) \in \hat{\Gamma}_v(h(v(t))).
\end{equation}
Any solution to (\ref{di}) as mentioned in Remark~\ref{rem3}, is guaranteed to be absolutely continuous and satisfy (\ref{di}) almost everywhere. 

We recall Definition II of \citep{benaim-inclusions}. By a perturbed solution $\mathbf{y}:[0,\infty)\rightarrow \mathbb{R}^{d_1}$ to the DI (\ref{di}), we mean that $\mathbf{y}$ is absolutely continuous. In addition, there exists a locally integrable function $t\mapsto U(t)$ such that 
\[
\lim_{t\rightarrow\infty} \sup_{0\leq v\leq T} \| \int_{t}^{t+v} U(s) ds\|=0.
\]
Also, $\mathbf{y}$ and $U$ together satisfy
the DI
\[
\dot{\mathbf{y}}(t) -U(t) \in \hat{\Gamma}^{\delta(t)}_v(h(v(t))),
\]
for almost every $t>0$ and where $\delta(t)\rightarrow 0$ as $t\rightarrow\infty$. Here,
\[
\hat{\Gamma}^{\delta(t)}_v(h(v(t)))
= \{y\in\mathbb{R}^{d_1} | \exists z : \|z-x\|<\delta, \mbox{ } d(y, \hat{\Gamma}^{\delta(t)}_v(h(z))) <\delta\}.
\]

We have the following result on the convergence of (\ref{v1}):

\begin{theorem}
\label{thm1}
Under Assumptions \ref{assum:policy-lipschitz-bounded}, \ref{assum:negative-definite}, \ref{assum-ergodic-chain} and \ref{assum:ss}, the iterates $\{v_t, t\geq 0\}$ obtained according to (\ref{v1}) satisfy $v_n\rightarrow G$ almost surely, where $G$ is as in (\ref{Gv}). Further, $G$ is a closed connected internally chain recurrent invariant set of the DI (\ref{di}).
\end{theorem}
\begin{proof}
Note again that as a consequence of the projection operator $\Gamma_1$, $\sup_n \|v_n\| <\infty$ w.p.1. From the foregoing, the process $\bar{v}(t)$ can be shown as in Proposition 1.3 of \citep{benaim-inclusions} to be a bounded perturbation of the DI. The claim now follows from Theorem 3.6 and Lemma 3.8 of \citep{benaim-inclusions}.    
\end{proof}

\begin{remark}
\label{rem2}
Recall that $\theta^*(v)$ is the set of stable equilibria of the ODE (\ref{ode2}. Consider now another associated ODE for the faster (actor) recursion: 
\begin{equation}
    \label{ode3}
    \dot{\theta} = \hat{\Gamma}_2(\nabla L^\theta).
\end{equation}
Let the set of asymptotically stable equilibria (i.e., the maxima of $L^\theta$) correspond to $\check{\theta}(v)$. 
It follows as a consequence of Proposition~\ref{lem2} and the continuity of $\nabla L^{\theta}$ and $e^{\pi^\theta}$ that given $\epsilon>0$, there exists $\delta>0$ such that if $\|e^{\pi^\theta}\|<\delta$, then $\theta_t \rightarrow \check{\theta}(v)^\epsilon \stackrel{\triangle}{=} \{\theta|\|\theta-\theta_0\|<\epsilon, \theta_0\in \check{\theta}(v)\}$.
\end{remark}

As a consequence of Remark~\ref{rem2}, we may let $h(v)$ be defined as
\[
h(v) = \{ \sum_s \mu_\theta(s)\sum_a \pi^\theta(s,a) (R(s,a) - L^\theta + v^T \sum_{s'} p(s,a,s')\phi(s')\]
\[-v^T \phi(s))\phi(s) | \theta \in \bar{\check{\theta}}(v)^\epsilon \},
\]
where $\bar{\check{\theta}}(v)^\epsilon$ is the closure of $\check{\theta}(v)^\epsilon$. 

Let $R^{\pi^\theta} = (\sum_{a} \pi^\theta(s,a)R(s,a), s\in S)^T$ be a column vector of the size of the state space. Also, recall (see Lemma~\ref{lem0}) that  $P_\theta$ denotes the transition probability matrix under policy $\pi^\theta$. For $J\in\mathbb{R}^{|S|}$, let $T_\theta:\mathbb{R}^{|S|} \rightarrow \mathbb{R}^{|S|}$ be the Bellman operator under policy $\pi^\theta$ defined according to 
\[
T_\theta(J) = R^{\pi^\theta} -L^\theta e + P_\theta J,
\]
where $e\in \mathbb{R}^{|S|}$ is the unit vector with all entries one. Let $\Phi$ denote the $|S|\times d_1$ feature matrix with $d_1$-dimensional rows $\phi(s)^T$, $s\in S$. Then, in vector-matrix notation, we have
\[
h(v) = \Phi^T D^\theta (T(\Phi v) - \Phi v).
\]
Notice that $h(v)$ is a set-valued map for any $v$ and $v_0$ will be an equilibrium for the DI (\ref{di}) if $0\in \hat{\Gamma}_v(h(v_0))$. 
Now as per Remark~\ref{rem2}, if $\|e^{\pi^\theta}\|<\delta$, then $\theta \in N^\epsilon(\bar{\check{\theta}}(v))$, an $\epsilon$-neighborhood of the set of local maxima of the function $L^\theta$. 
We now have the following useful result:
 
\begin{theorem}
\label{thm2}
Consider a solution $v(\cdot)$ to the differential inclusion (\ref{di}). Suppose $\lim_{t\rightarrow\infty} v(t)=\hat{v}$. Then $\hat{v}$ is an equilibrium of the DI (\ref{di}), i.e., $0\in \hat{\Gamma}_v(h(\hat{v})$. 
\end{theorem}
\begin{proof}
Recall that by Lemma~\ref{lem3}, the set-valued map $\hat{\Gamma}_v(h(v))$ is a Peano map. The claim now follows by Theorem 10.1.12 of \citep{aubin}.    
\end{proof}

Theorem \ref{main-thm} gives the main result. For any function $f$, we say that it's local maxima are isolated if around each such maximum, one can construct an open ball of small enough radius so that any two balls do not intersect. 

\begin{theorem}
\label{main-thm}
Suppose the ODE (\ref{ode2}) has isolated local maxima $\theta^*$. Correspondingly suppose $v^*\in C^o$ (the interior of $C$) is a limit point of the solution to the DI (\ref{di}). Then under Assumptions \ref{assum:policy-lipschitz-bounded}, \ref{assum:negative-definite}, \ref{assum-ergodic-chain} and \ref{assum:ss}, $\{(v_t,\theta_t)\}$ governed according to (\ref{v1})-(\ref{theta}) satisfy
$(v_t,\theta_t) \rightarrow (v^*,\theta^*)$ almost surely, where $\theta^*$ is a local maximum of (\ref{ode2}) and $v^*$ is the unique solution to the projected Bellman equation corresponding to the policy $\pi^{\theta^*}$, i.e., the two together satisfy
\begin{equation}
\label{pbe_}
\Phi^T D^{\theta^*}\Phi v^* = 
\Phi^T D^{\theta^*} T_{\theta^*}(\Phi v^*).
\end{equation}
\end{theorem}

\begin{proof}
Recall that for any $v$, $\theta_t\rightarrow \theta^*$ almost surely according to Proposition~\ref{lem2}-Remark~\ref{rem21}. Moreover, $v^*\in C^o$ would mean $0\in h(v^*)$ and in fact since $\theta_t\rightarrow \theta^*$ almost surely, as $t\rightarrow\infty$, we have that corresponding to $\theta^*$, $h(v)$ would be a point-to-point map and will no longer be a (nontrivial) point-to-set map as before. Setting $h(v)=0$ in this case would imply the Bellman equation (\ref{pbe_}) for policy $\pi^{\theta^*}$. The claim follows. 
\end{proof}
\begin{remark}
\label{rem-proj}
In Theorem~\ref{main-thm}, (\ref{pbe_}) may or may not hold if $v^*\in \partial C$ (the boundary of $C$). This is because projection set boundaries themselves can induce spurious attractors, see \citep{kushneryin}. We analyze below the case where the critic is not projected to the set $C$ but can take any value in $\mathbb{R}^{d_1}$. We prove the stability of the critic recursion in this case, i.e., that $\sup_t \|v_t\| <\infty$ w.p.1 and Theorem~\ref{main-thm2} presents the main (general) result on convergence of the joint iterate process in this case of unprojected critic iterates.
\end{remark}
\begin{remark}
\label{rem5}
It has been argued in the case of the discounted cost CA algorithm in \citep{bhatnagar2023actorcritic} for the look-up table setting that the algorithm there mimics value iteration unlike AC that mimics the policy iteration procedure. Now by iterating $\theta$ on a faster timescale as opposed to the $v$-update, the resulting scheme is seen to resemble projected value iteration, see  \citep{bertsekas-vol2}.  
\end{remark}

\subsubsection{The Case of Critic without Projection}
\label{unp-critic}

We now consider the case when the critic recursion is unconstrained though the actor recursion continues to be constrained (using the projection operator $\Gamma_2$ as before). The AC analog of this algorithm has for instance been analysed for its asymptotic convergence in \citep{BHATNAGAR20092471}. An important observation here is that since there is no projection now on the critic, one needs to establish explicitly that the critic recursion remains uniformly bounded almost surely. 
In this case, we rewrite (\ref{v1}) as 
\begin{equation}
\label{v1-re10}
v_{t+1} = v_t +\beta_t (y_t + Y_t +\kappa_t),
\end{equation}
where,
\[
y_t = \sum_{s} \mu_{\theta_t}(s) \sum_a\pi^{\theta_t}(s,a)(R(s,a) - L^{\theta_t} + v_t^T \sum_{s'} p(s,a,s')\phi(s') - v_t^T \phi(s))\phi(s),
\]
\[
Y_{t} = \delta_t \phi(s_t) - E[\delta_t\phi(s_t)|\mathcal{F}_2(t)], 
\]
\[
\kappa_t = E[(R(s_t,a_t)-L^{\theta_t}+ v_t^T \sum_{s_{t+1}} p(s_t,a_t,s_{t+1})\phi(s_{t+1}) - v_t^T \phi(s_t))\phi(s_t)|\mathcal{F}_2(t)] - y_t,
\]
as before. Thus, $(Y_t,\mathcal{F}_2(t)), t\geq 0$ is a martingale difference sequence and $\kappa_t,t\geq 0$ constitutes the Markov noise. 
Now from Assumption~\ref{assum-ergodic-chain}, it follows that $\kappa_t\rightarrow 0$ almost surely as $t\rightarrow\infty$. We now prove Theorem \ref{di-stability} concerning the stability of the recursion (\ref{v1-re10}):

\vspace*{0.1in}
\noindent {\em Proof of Theorem~\ref{di-stability}:}

\vspace*{0.1in}
% \begin{theorem}[Stability of the Critic Recursion]
% \label{di-stability_}
% Under Assumptions~\ref{assum:negative-definite}, \ref{assum:policy-lipschitz-bounded}, \ref{assum-ergodic-chain} and \ref{assum:ss}, the recursion (\ref{v1-re10}) remains uniformly bounded almost surely, i.e., $\sup_{n\rightarrow\infty} \|v_n\| <\infty$, w.p.1
% \end{theorem}
% \begin{proof}
Note that (\ref{di-c21}) denotes the  differential inclusion (\ref{di-c2}) associated with the critic update. 
\begin{equation}
\label{di-c21}
\dot{v} \in h(v),
\end{equation}
where recall that
\[
h(v) = \{ \sum_s \mu_{\theta} (s)\sum_a \pi^{\theta} (s,a) (R(s,a) - L^{\theta} + v^T \sum_{s'} p(s,a,s')\phi(s')\]
\[-v^T \phi(s))\phi(s) | \theta \in \bar{\theta^*}(v) \},
\]
and where $\bar{\theta}^*(v)$ is the closure of the set $\theta^*(v)$, that in turn is convex and bounded. Thus, $\bar{\theta}^*(v)$ is a convex and compact set. Observe that $\mu_{\theta}$, $\pi^{\theta}$ and $L^\theta$ are bounded and Lipschitz continuous. Moreover, letting ${\displaystyle h_\infty(v) = \lim_{c\rightarrow\infty} \frac{h(cv)}{c}}$, we have 
\begin{equation}
\label{di-inf}
h_\infty(v) 
= \{\sum_s \mu_{\theta}(s) \sum_a
\pi^{\theta}(s,a) (v^T \sum_{s'} p(s,a,s')\phi(s') - v^T\phi(s))\phi(s) | \theta \in \check{\theta^*}\}, 
\end{equation}
where $\check{\theta^*} = \lim_{c\rightarrow\infty}\bar{\theta^*}(cv)$.
By Assumption~\ref{assum:negative-definite}, 
    for all potential policy parameters $\theta$, the matrix $\Ab$ defined as under is negative definite:  % and has the maximum eigenvalue $- \lambda <0$.
${\displaystyle
    \Ab := \EE_{s,a,s^{'}} \big[ \phi(s) \big( \phi(s^{'}) - \phi(s)\big)^{\top} \big],
}$
where $s \sim \mu_{\theta}(\cdot)$, $a \sim \pi^{\theta}(s,\cdot), s' \sim p(s, a,\cdot)$. Consider now the DI:
\[
\dot{v} \in \{\Ab v | \theta \in \check{\theta^*}\}.
\]
%where $\check{\theta^*} = \lim_{c\rightarrow\infty} \bar{\theta^*}(cv)$
Since $\Ab$ is negative definite regardless of $\theta$, the above DI will have the origin as an asymptotically stable attractor with a unit ball around the origin as its fundamental neighborhood. 
The claim now follows from Theorem 1 of \citep{ramaswamy-bhatnagar}. \hfill $\Box$
%\end{proof}

\vspace*{0.1in}

The conclusions of Theorem~\ref{thm1}-\ref{thm2} shown above continue to hold.
We finally come to the proof of Theorem~\ref{main-thm}.

\vspace*{0.1in}
\noindent {\em Proof of Theorem~\ref{main-thm2}:}

\vspace*{0.1in}
%We finally have the following main result.

% \begin{theorem}
% \label{main-thm2}
% Suppose the ODE (\ref{ode2}) has isolated local maxima $\theta^*$. Correspondingly suppose $v^*\in \mathbb{R}^{d_1}$ is a limit point of the solution to the DI (\ref{di-c2}). Then under Assumptions~\ref{assum:negative-definite}, \ref{assum:policy-lipschitz-bounded}, \ref{assum-ergodic-chain} and \ref{assum:ss}, $\{(v_t,\theta_t)\}$ governed according to (\ref{v1})-(\ref{theta}) satisfy that $\sup_t \|v_t\| <\infty$ and $\sup_t \|\theta_t\| < \infty$ w.p.1 respectively. In addition, 
% $(v_t,\theta_t) \rightarrow (v^*,\theta^*)$ almost surely, where $\theta^*$ is a local maximum of (\ref{ode2}) and $v^*$ is the unique solution to the projected Bellman equation corresponding to the policy $\pi^{\theta^*}$, i.e., the two together satisfy
% \begin{equation}
% \label{pbe2}
% \Phi^T D^{\theta^*}\Phi v^* = 
% \Phi^T D^{\theta^*} T_{\theta^*}(\Phi v^*).
% \end{equation}
%\end{theorem}
%\begin{proof}

The proof now follows from Theorem 2 of \citep{ramaswamy-bhatnagar}. 
\subsection{ Hyper-parameters and Compute time }
\label{simdetails}
 
We describe in Tables~\ref{table_hyperparameter_ca}-\ref{table_hyperparameter_ppo_ca} below the hyper-parameters used for each of the algorithms and subsequently the training time for the various algorithms is given in Table~\ref{sample-table2}. The learning rates used for our CA algorithm are those that have been found optimal (see the finite-time analysis). Similarly, the rates used in the AC algorithm are those that have been found optimal in \citep{wu2022finite}. These rates have also been used in the PPO-CA and PPO-AC algorithms respectively.

\begin{table*}[t]
\caption{Hyper-parameters used for the Average Reward CA Algorithm
}
\label{table_hyperparameter_ca}
\vskip 0.15in
\begin{center}
\begin{small}
%\begin{sc}
\begin{tabular}{|c|c|}
\toprule
Hyper-parameter & Value/Description \\ \hline
Number of hidden layers in actor NN & 1 \\
Number of hidden layers in critic NN & 1\\
Number of nodes in hidden layer of actor NN & 64\\
Number of nodes in hidden layer of critic NN & 64\\
Activation Function used in critic NN & ReLU \\
Activation Function used in actor NN  & softmax\\ 
Learning rate for actor update (at time instant t) & $\frac{1.5}{(1 + t)^{0.5}}$\\
Learning rate for critic update (at time instant t) & $\frac{1.5}{(1 + t)^{0.51}}$\\
\bottomrule
\end{tabular}
%\end{sc}
\end{small}
\end{center}
\vskip -0.1in
\end{table*}

\begin{table*}[t]
\caption{Hyper-parameters used for the Average Reward AC  Algorithm
}
\label{table_hyperparameter_ac}
\vskip 0.15in
\begin{center}
\begin{small}
%\begin{sc}
\begin{tabular}{|c|c|}
\toprule
Hyper-parameter & Value/Description \\ \hline
Number of hidden layers in actor NN & 1 \\
Number of hidden layers in critic NN & 1\\
Number of nodes in hidden layer of actor NN & 64\\
Number of nodes in hidden layer of critic NN & 64\\
Activation Function used in critic NN & ReLU \\
Activation Function used in actor NN  & softmax\\ 
Learning rate for actor update (at time instant t) & $\frac{1.5}{(1 + t)^{0.6}}$\\
Learning rate for critic update (at time instant t) & $\frac{1.5}{(1 + t)^{0.4}}$\\
\bottomrule
\end{tabular}
%\end{sc}
\end{small}
\end{center}
\vskip -0.1in
\end{table*}

\begin{table*}[t]
\caption{Hyper-parameters used for the Average Reward Single timescale AC  Algorithm
}
\label{table_hyperparameter_single_ac}
\vskip 0.15in
\begin{center}
\begin{small}
%\begin{sc}
\begin{tabular}{|c|c|}
\toprule
Hyper-parameter & Value/Description \\ \hline
Number of hidden layers in actor NN & 1 \\
Number of hidden layers in critic NN & 1\\
Number of nodes in hidden layer of actor NN & 64\\
Number of nodes in hidden layer of critic NN & 64\\
Activation Function used in critic NN & ReLU \\
Activation Function used in actor NN  & softmax\\ 
Learning rate for actor update (at time instant t) & $\frac{1.5}{(1 + t)^{0.6}}$\\
Learning rate for critic update (at time instant t) & $\frac{1.5}{(1 + t)^{0.6}}$\\
\bottomrule
\end{tabular}
%\end{sc}
\end{small}
\end{center}
\vskip -0.1in
\end{table*}

\begin{table*}[t]
\caption{Hyper-parameters used for the Average Reward DQN Algorithm
}
\label{table_hyperparameter_dqn}
\vskip 0.15in
\begin{center}
\begin{small}
%\begin{sc}
\begin{tabular}{|c|c|}
\toprule
Hyper-parameter & Value/Description \\ \hline
Number of hidden layers in QNN & 1 \\
Number of nodes in hidden layer of QNN & 64\\
Activation Function used in QNN & ReLU \\ 
Learning rate & 0.5 \\
\bottomrule
\end{tabular}
%\end{sc}
\end{small}
\end{center}
\vskip -0.1in
\end{table*}

\begin{table*}[t]
\caption{Hyper-parameters used for the Average Reward PPO-AC Algorithm
}
\label{table_hyperparameter_ppo_ac}
\vskip 0.15in
\begin{center}
\begin{small}
%\begin{sc}
\begin{tabular}{|c|c|}
\toprule
Hyper-parameter & Value/Description \\ \hline
Number of hidden layers in actor NN & 1 \\
Number of hidden layers in critic NN & 1\\
Number of nodes in hidden layer of actor NN & 64\\
Number of nodes in hidden layer of critic NN & 64\\
Activation Function used in critic NN & ReLU \\
Activation Function used in actor NN  & softmax\\ 
Learning rate for actor update (at time instant t) & $\frac{1.5}{(1 + t)^{0.6}}$\\
Learning rate for critic update (at time instant t) & $\frac{1.5}{(1 + t)^{0.4}}$\\
Batch length & 50\\
\bottomrule
\end{tabular}
%\end{sc}
\end{small}
\end{center}
\vskip -0.1in
\end{table*}

\begin{table*}[t]
\caption{Hyper-parameters used for the Average Reward PPO-CA Algorithm
}
\label{table_hyperparameter_ppo_ca}
\vskip 0.15in
\begin{center}
\begin{small}
%\begin{sc}
\begin{tabular}{|c|c|}
\toprule
Hyper-parameter & Value/Description \\ \hline
Number of hidden layers in actor NN & 1 \\
Number of hidden layers in critic NN & 1\\
Number of nodes in hidden layer of actor NN & 64\\
Number of nodes in hidden layer of critic NN & 64\\
Activation Function used in critic NN & ReLU \\
Activation Function used in actor NN  & softmax\\ 
Learning rate for actor update (at time instant t) & $\frac{1.5}{(1 + t)^{0.5}}$\\
Learning rate for critic update (at time instant t) & $\frac{1.5}{(1 + t)^{0.51}}$\\
Batch length & 50\\
\bottomrule
\end{tabular}
%\end{sc}
\end{small}
\end{center}
\vskip -0.1in
\end{table*}

\begin{table*}[t]
\caption{Training Time (in seconds) Averaged over 10 Seeds  for 10,000 Iterations.
}
\label{sample-table2}
\vskip 0.15in
\begin{center}
\begin{small}
%\begin{sc}
\begin{tabular}{|c|c|c|c|p{2 cm}|p{1.6 cm}|c|}
\toprule
    Environment & CA & AC & DQN & PPO-AC & PPO-CA & Single Timescale AC\\     
\midrule
Frozen Lake & 21.82 & 22.66 & 14.26 & 53.01 & 51.72 &  23.77 \\ \hline
Pendulum &  13.31 & 26.87 & 9.51 & 61.84 & 61.55 & 27.91\\ \hline
Mountain Car Continous & 12.09 & 24.2 & 7.77 & 53.86 & 53.13 & 26.28\\
\bottomrule
\end{tabular}
%\end{sc}
\end{small}
\end{center}
\vskip -0.1in
\end{table*}

\begin{table*}[t]
\caption{CPU Configuration of the Server used for the Experiments}
\label{cpu_configuration}
\begin{center}
\begin{small}
%\begin{sc}
\begin{tabular}{|c|c|}
\toprule
Architecture   &                    x86\_64\\ \hline
CPU op-mode(s) &                 32-bit, 64-bit\\ \hline
Byte Order &                         Little Endian\\ \hline
Address sizes &                     46 bits physical, 48 bits virtual\\ \hline
CPU(s) &                         40\\ \hline
On-line CPU(s) list &               0-39\\ \hline
CPU family &                        6\\ \hline
Model name &                         Intel(R) Xeon(R) CPU E5-2698 v4 @ 2.20GHz\\ \hline
CPU MHz &                     2759.590 \\ \hline
OS &                          Linux\\
\bottomrule
\end{tabular}
%\end{sc}
\end{small}
\end{center}
\vskip -0.1in  
\end{table*}
Observe from Table~\ref{sample-table2} that CA shows better results than AC. In particular, on Pendulum and Frozen Lake environments, CA requires nearly half the training time of AC while showing better average reward performance (see Table~\ref{tab:experiment}. Similarly, amongst PPO-CA and PPO-AC, the former requires lower training time. Amongst all algorithms, DQN has the lowest training time even though it does not perform as well as the CA algorithm in terms of average reward performance, see Table~\ref{tab:experiment}. We have used the pytorch library in our code. 

We observed that when the number of hidden layers in both the actor NN and critic NN was more than 1, the performance of the algorithms was affected negatively. Hence the number of hidden layers in both the actor NN and critic NN was chosen to be 1. Similarly the performance was observed to be the best when 64 nodes were used in the hidden layer. This was obtained by trying out over the full range of values 1-100. For finding the best activation function, we tried Linear, Tanh, ReLU and Softmax for each network, in the various algorithms. The best in each case was selected for each network (actor/critic) and these are provided in Tables \ref{table_hyperparameter_ca}--\ref{table_hyperparameter_ppo_ca}, respectively.

Finally, Table~\ref{cpu_configuration} shows the CPU configuration of the server on which the various simulations were conducted.

\end{document}